\documentclass[11pt,reqno]{amsart}

\usepackage{amsmath,amssymb,amsthm,mathrsfs}
\usepackage{mathtools}
\usepackage{bm}
\usepackage[utf8]{inputenc}
\usepackage[T1]{fontenc}

\usepackage[protrusion=true,expansion=false]{microtype}

\usepackage{geometry}
\usepackage{hyperref}
\hypersetup{
	colorlinks=true,
	linkcolor=blue!60!black,
	citecolor=blue!60!black,
	urlcolor=blue!60!black
}
\usepackage{cleveref}
\usepackage{enumitem}
\usepackage{graphicx}
\usepackage{tikz}

\usetikzlibrary{positioning,calc,arrows.meta,fit,backgrounds}
\usepackage{booktabs}
\usepackage{array}
\usepackage{float}
\usepackage{xcolor}
\usepackage[numbers,sort&compress]{natbib}

\theoremstyle{plain}
\newtheorem{theorem}{Theorem}[section]
\newtheorem{lemma}[theorem]{Lemma}
\newtheorem{proposition}[theorem]{Proposition}
\newtheorem{corollary}[theorem]{Corollary}

\theoremstyle{definition}
\newtheorem{definition}[theorem]{Definition}

\newtheorem{remark}[theorem]{Remark}

\theoremstyle{remark}
\newtheorem{observation}[theorem]{Observation}

\DeclareMathOperator{\Ker}{Ker}
\DeclareMathOperator{\Bas}{Bas}
\DeclareMathOperator{\Spt}{Spt}
\DeclareMathOperator{\Image}{Im}

\DeclareMathOperator{\median}{med}

\newcommand{\R}{\mathbb{R}}
\newcommand{\Z}{\mathbb{Z}}
\newcommand{\Rbar}{\overline{\mathbb{R}}}
\newcommand{\Fun}{\mathrm{Fun}}

\newcommand{\Lq}{\mathscr{F}_q}

\newcommand{\ConvEros}[1]{\varepsilon^{\mathrm{Conv}}_{#1}}

\newcommand{\MaxPool}[1]{\delta^{\mathrm{MP}}_{#1}}
\newcommand{\ReLUop}{\delta^{\mathrm{ReLU}}}
\newcommand{\APD}{\Psi^{\mathrm{APD}}}
\newcommand{\APMO}{\Psi^{\mathrm{APMO}}}
\newcommand{\SigSpec}{\Sigma^{\mathrm{Spec}}}
\newcommand{\SigSpecDual}{\Sigma^{\mathrm{Spec},\prime}}

\newcommand{\PsiMorpho}{\Psi^{\mathrm{M}}}

\newcommand{\decEros}[2]{\varepsilon^{\downarrow #1}_{#2}}
\newcommand{\intDil}[2]{\delta^{*\uparrow #1}_{#2}}
\newcommand{\decDil}[2]{\delta^{\downarrow #1}_{#2}}
\newcommand{\intEros}[2]{\varepsilon^{*\uparrow #1}_{#2}}

\newcommand{\medInf}{\sqcap}
\newcommand{\medSup}{\sqcup}
\newcommand{\medOrd}{\preceq}
\newcommand{\ErosMed}[1]{\varepsilon^{\mathrm{Med}}_{#1}}
\newcommand{\DilMed}[1]{\delta^{*\mathrm{Med}}_{#1}}
\newcommand{\OpenMed}[1]{\gamma^{\mathrm{Med}}_{#1}}

\begin{document}
	
	\title[Lattice Theory and Algebraic Models for Deep Convolutional Networks]{
		Lattice theory and algebraic models \\ for deep convolutional learning \\ 
		based on mathematical morphology}

	\author{Gustavo (Jes\'{u}s) Angulo}
	\address{Mines Paris, PSL University, \\
		CMA-Center for Applied Mathematics, Sophia-Antipolis, France}
	\email{jesus.angulo\_lopez@minesparis.psl.eu}
	
	\date{April 28th, 2026}
	
	

	\begin{abstract}
		We develop a rigorous algebraic framework for deep convolutional
		architectures, CNNs, ResNets, and encoder--decoder networks
		such as UNet, grounded in lattice theory and
		mathematical morphology.
		The central tool is the Matheron--Maragos--Banon--Barrera (MMBB) 
		universal representation theory for translation-invariant operators,
		which we apply systematically to every layer of a standard
		deep network.
		
		The principal finding is that the standard CNN pipeline
		(linear convolution~$+$ ReLU~$+$ flat max-pooling) is a
		\emph{cross-lattice} operator: the convolution is an erosion
		in the Fourier inf-semilattice while ReLU is a lattice-join
		closing and max-pooling is a dilation in the pointwise max-plus
		lattice, and their composition is a morphological opening in neither.
		A second finding is that the upper adjoint of ReLU in the
		pointwise lattice is a global (non-local) operator,
		the identity on globally non-negative functions and $-\infty$
		otherwise, so no local morphological erosion can form an
		adjunction pair with ReLU.
		These two results together provide the precise algebraic reason
		why depth in standard CNNs introduces genuine representational
		power: the composed layer is not idempotent.
		
		Three layer designs that \emph{are} genuine
		idempotent openings are identified and fully characterised:
		the pure max-plus morphological layer (Type~I, pointwise
		lattice), the spectral Wiener layer (Type~II, Fourier lattice,
		idempotent in the limit $\epsilon\to 0$), and the self-dual
		morphological layer (Type~III, median inf-semilattice for
		signed feature maps).
		For Type~I layers we establish a complete fixed-point and
		convergence theory: the morphological ResNet block converges
		in one step by idempotency, while the naive residual iteration
		$\Phi(f)+f$ diverges.
		
		The framework also unifies max-pooling, strided convolution,
		and the Laplacian pyramid under the Goutsias--Heijmans adjoint
		pyramid theory, and gives the Activation--Pooling Dilation (APD)
		factorisation with its correct adjoint (piecewise-constant
		upsampling by stride~$R$).
		The UResNet architecture is proposed, in which skip connections
		carry residues of openings rather than concatenated features,
		achieving exact scale-by-scale reconstruction.
		
		Building on prior work by the author, this paper provides the 
		lattice-specific operator-theoretic
		content that both categorical and computational approaches to
		deep learning currently lack.
	\end{abstract}

	\maketitle
	
	\tableofcontents
	
	\section{Introduction}
	\label{sec:intro}
	
	Understanding the mathematical structure of deep neural networks
	beyond their empirical performance has become one of the central
	challenges of mathematical data science.
	Several research programmes have emerged to address this gap.
	Tropical geometry identifies the piecewise-linear (PWL) geometry
	of ReLU networks with the combinatorics of tropical
	polynomials~\cite{zhang2018tropical,montufar2014regions,
		maragosTPWL2021}.
	Harmonic analysis studies spectral bias and the interaction of
	convolutional layers with Fourier
	representations~\cite{rahaman2019spectral}.
	Alternatively, category theory has been proposed as an abstract
	skeleton for gradient-based learning
	systems~\cite{cruttwell2022categorical,fong2019backprop},
	identifying backpropagation as a functor and network composition
	as a morphism.
	
	The present paper pursues a different, but complementary,
	direction: \emph{lattice theory and mathematical
		morphology} as the algebraic foundation for convolutional
	architectures.
	Mathematical morphology, developed for binary and grey-scale
	image analysis by Matheron, Serra, Maragos, and their
	schools~\cite{matheron1975,serra1982,maragos1989},
	provides a powerful operator-theoretic calculus on complete
	lattices.
	One of its central results, the MMBB universal representation theory, 
	named for Matheron, Maragos, Banon, and Barrera, states that
	any translation-invariant, increasing operator on a
	complete lattice is representable as a supremum of erosions,
	providing a constructive universal decomposition analogous to
	the basis representation in linear algebra. There is a more 
	general form for any nonlinear translation-invariant operator
	using two families of basis with erosions and anti-dilations.
	
	Our premise is that the operators composing a convolutional
	deep network, convolution, ReLU, max-pooling, strided
	convolution, skip connections, are all special instances, or
	close approximations, of morphological operators on appropriate
	lattices.
	Recognising this structure yields the following kinds of dividends:
	algebraic \emph{factorisation} (which operators can be fused or
	commuted); algebraic \emph{inversion} (the correct adjoint of
	each layer, which is lattice-dependent); \emph{fixed-point and
		idempotency} characterisation (which composed layers are
	idempotent openings, and critically which are \emph{not}, the
	standard CNN layer is not idempotent); algebraic
	\emph{design} principles (how morphological pipelines inspire
	new architectures such as the UResNet); and the
	\emph{max-times and self-dual extensions} for a mathematically 
	consistent computational frameworks for either positive (or probabilistic) 
	or signed feature maps.
	
	\bigskip
	\paragraph{State of the art on morphological neural networks.}
	The intersection of mathematical morphology and neural networks
	has a history reaching back to the morphological/rank
	perceptrons of Ritter--Sussner~\cite{ritter1996} and
	Pessoa--Maragos~\cite{pessoa2000}, and to the morphological
	neural network frameworks initiated by
	Davidson--Ritter~\cite{davidson1992}.
	The deep learning era has brought a resurgence of this field.
	Franchi, Fehri, and Yao~\cite{franchi2020} demonstrated the
	potential of deep hybrid morphological--linear networks for
	image analysis.
	Dimitriadis and Maragos~\cite{dimitriadis2021} showed the
	remarkable pruning amenability of morphological networks and
	the importance of shape constraints on learnable structuring
	elements; the same group extended this line in~\cite{dimitriadis2023}.
	Maragos~\cite{maragosTPWL2021} provided a unifying treatment
	connecting morphological networks, max-plus algebra, and
	tropical geometry.
	More recently, work by Fotopoulos and
	Maragos~\cite{fotopoulos2025} represents a significant
	step: they demonstrate that ``linear-like'' activation functions
	facilitate training deep morphological networks, and that
	residual connections substantially improve their
	generalisation, a finding that our algebraic analysis in
	\Cref{sec:cnn_models} sheds new theoretical light on.
	Blusseau, Velasco-Forero, Angulo, and
	Bloch~\cite{blusseau2022adjunctions} showed that morphological
	adjunctions can be represented as matrices in max-plus algebra,
	providing a computationally tractable framework for the design
	of morphological layers.
	Blusseau~\cite{blusseau2024training} analysed the gradient flow
	and convergence behaviour of morphological layers trained by
	gradient descent, identifying conditions under which training
	is well-posed. 
	Hermary, Tochon, Puybareau, Kirszenberg, and
	Angulo~\cite{hermary2022smooth} introduced smooth morphological
	layers (SMorph) as differentiable approximations of max-plus
	erosions and dilations via $p$-convolution, enabling gradient
	descent on morphological parameters; Bottenmuller, Tochon,
	Hermary, Puybareau, and Angulo~\cite{bottenmuller2025dgmm}
	subsequently improved these layers.
	Dimitrova, Blusseau, and Velasco-Forero~\cite{dimitrova2025}
	study the influence of initialisation and layer
	differentiability on the morphological representations learned
	during training.
	We show in \Cref{sec:self_dual,sec:perspectives} that training
	a morphological layer is equivalent to implicitly learning an
	MMBB basis: the fixed-point set of the learned layer equals the
	image of the corresponding opening, which is precisely the set
	of signals representable by that basis.
	Initialisation and differentiability determine which basin of
	attraction in basis space is reached by gradient descent. 	
	Marcondes and Barrera~\cite{marcondes2024lattice} proposed the
	lattice overparameterisation paradigm for learning lattice
	operators, which complements the MMBB-basis learning perspective
	developed here.
	
	The equivariance programme, making morphological networks
	equivariant to group actions beyond translation, has been
	advanced by Sangalli, Blusseau, Velasco-Forero, and
	Angulo~\cite{sangalli2021}, who introduced morphological
	scale-spaces for scale-equivariant CNNs.
	Penaud-Polge, Velasco-Forero, and Angulo~\cite{penaud2024}
	developed group-equivariant morphological networks via group
	morphology (Roerdink's theory), a programme subsequently extended to a
	full theoretical treatment in~\cite{penaud2025siam}.
	
	\emph{Important scope clarification.}
	The present paper is \emph{not} a contribution to morphological
	neural networks in the computational sense: we do not design,
	train, or benchmark new morphological layers.
	Rather, we provide the \emph{algebraic theoretical framework}
	that underpins both morphological neural networks and standard
	deep convolutional architectures viewed morphologically.
	Our central claim is that this framework provides a unified
	language for analysing CNN, ResNet, and UNet in terms of
	lattice operators, adjunctions, and morphological bases,
	independently of whether the practitioner explicitly uses
	morphological layers.
	We also propose new architectures that are theoretically
	well-founded and will require computational implementation
	and empirical study.
	
	\bigskip
	\paragraph{Fixed points, idempotency, and stability.}
	A key theoretical enrichment of the present paper, compared to
	earlier work~\cite{angulo2021morpho}, is an
	explicit and precise treatment of fixed-point operators and
	idempotency within the morphological model of deep architectures.
	A principal finding is that the standard CNN layer
	(convolution $+$ ReLU $+$ max-pooling) is \emph{not}
	idempotent, because it is a cross-lattice composition.
	Idempotency, the defining property of morphological openings
	and closings, holds only for algebraically rigorous layer
	designs where both operators inhabit the same lattice and
	form an adjoint pair.
	We characterise precisely which layer designs satisfy this
	condition, determine their fixed-point sets, and prove
	convergence results for morphological ResNet and UNet blocks.
	This analysis connects to the author's recent work on
	group-morphology fixed-point layers~\cite{angulo2025gsi},
	specialised here from the group-equivariant to the
	translation-equivariant setting.
	
	\bigskip
	\paragraph{Relation to other prior work of the author.}
	This paper consolidates and substantially extends a research
	programme presented in preliminary form at the DGMM 2021
	conference~\cite{angulo2021morpho}, explored computationally
	jointly with
	Velasco-Forero~\cite{velasco2022siims,velasco2022morphoact,velasco2022bmvc}.
	The author's DGMM 2025 contribution~\cite{angulo2025dgmm}
	studies the Fourier-domain (spectral) side of the same
	programme, the morphological interpretation of scattering
	networks, and is the object of an ongoing research line;
	references to the spectral lattice appear here mostly in the
	Fourier perspectives subsection of \Cref{sec:perspectives}.
	
	\bigskip
	\paragraph{Relation to category-theoretic approaches.}
	The categorical frameworks of Cruttwell et al.\
	\cite{cruttwell2022categorical} and Fong et
	al.~\cite{fong2019backprop} identify the \emph{compositional
		structure} of learning.
	Our framework is orthogonal and complementary: it furnishes the
	\emph{operator-theoretic content} of each layer via the MMBB
	constructive universal representation theory, with explicit structuring
	functions, computable bases, and quantifiable approximation
	errors.
	The adjunction between erosion and dilation, the central
	algebraic pair of any morphological framework, is a concrete categorical
	adjunction, but enriched by the MMBB universality theorems that
	categorical approaches lack.

	\bigskip
	\paragraph{Main contributions.}
	\label{par:contributions}
	We make the following contributions, building on
	prior work by the author.
	
	\begin{enumerate}[leftmargin=*,label=\textbf{C\arabic*.}]
		\item \textbf{Max-times algebra and log-domain equivalence.}
		We present the max-times adjunction (multiplicative
		structuring functions) as a second canonical example of
		morphological adjunction alongside the classical max-plus,
		and prove their algebraic equivalence via the
		log/exp isomorphism.
		This furnishes a natural model for layers operating on
		positive feature maps (e.g.\ after softmax or sigmoid
		activations).
		
		\item \textbf{Banon--Barrera representation for
			non-increasing neural network operators.}
		We include the full Banon--Barrera (1993)
		sup-generating decomposition for TI operators that are
		\emph{not} necessarily increasing, which decomposes any
		such mapping as a supremum of inf-combinations of an erosion
		and an anti-dilation.
		This extends the representational scope of the MMBB
		framework to signed CNN layers with general (non-positive)
		kernels.
		
		\item \textbf{MMBB basis of convolution and learnable layers.}
		The morphological basis of a positive normalised
		convolution kernel $k$ with support of size $N$ is
		isomorphic to the $(N{-}1)$-dimensional hyperplane in
		$\R^N$ orthogonal to $k$ (Maragos--Schafer), computed via
		the characteristic matrix $A_k$ (Khosravi--Schafer).
		We develop the virtual basis theory for quantised signals
		and propose the MMBB layer as a finite, learnable
		truncation.
		
		\item \textbf{Pooling as morphological pyramid; cross-lattice
			structure.}
		Max-pooling, strided convolution, and the Laplacian pyramid
		are unified as cases of the Goutsias--Heijmans adjoint
		pyramid.
		We prove that the standard CNN layer is a
		\emph{cross-lattice} operator and is generically
		\emph{not} idempotent.
		
		\item \textbf{ReLU adjoint and non-locality of its adjoint
			erosion.}
		We prove that the upper adjoint of ReLU in the pointwise
		lattice $(\Fun(E,\Rbar),\leq)$ is \emph{not} a pointwise
		(local) operator: it is the identity on functions that are
		globally non-negative, and $-\infty$ on all others.
		This global character distinguishes ReLU from max-plus
		dilations, whose adjoints are always local morphological
		erosions.
		As a consequence, no local morphological erosion can
		form an adjunction pair with ReLU in the pointwise lattice,
		providing a second independent algebraic explanation
		(complementing the cross-lattice argument) for why
		$\MaxPool{R} \circ \ReLUop$ cannot be a morphological
		opening.
		The same non-locality propagates to the APD, whose adjoint
		is a global piecewise-constant upsampling by stride~$R$
		followed by a bias shift $-\alpha$.
		
		\item \textbf{Morphological activation family and APD
			factorisation.}
		We introduce the \emph{morphological activation family}
		$\sigma^{\mathrm{M}}_{\mathcal{B},c}$ as a supremum of
		Banon--Barrera sup-generating operators sharing a common
		cap, derived via complete distributivity of $(\Rbar,\leq)$
		($b\wedge\bigvee_g a_g = \bigvee_g(b\wedge a_g)$, valid
		in any completely distributive lattice).
		This family unifies ReLU and its morphological
		generalisations under a single algebraic framework.
		We prove the \emph{Activation-Pooling Dilation} (APD)
		factorisation: decimated ReLU and max-pooling compose into
		a single dilation $\APD_{R;\alpha}(f)(n) = \sup_{y\in
			W_R}\max(0,f(Rn{-}y)+\alpha)$ on the coarser output grid,
		whose adjoint is a global piecewise-constant upsampling
		by factor $R$ followed by a bias shift $-\alpha$.
		
		\item \textbf{Self-dual operators, signed activations,
			and UResNet.}
		The median complete inf-semilattice and its self-dual
		erosion and opening provide a third family of idempotent
		layers, extending the fixed-point analysis to signed
		feature maps (a design appropriate for ResNet residuals,
		normalised features, and Fourier coefficients).
		Standard ReLU is \emph{not} a median-lattice dilation;
		Leaky/Parametric ReLU are; the self-dual opening
		$\OpenMed{W}$ is idempotent in the median lattice.
		We propose the UResNet, whose residual skip connections are
		algebraically motivated as top-hat transforms of the
		encoder's adjoint opening.	
		
		\item \textbf{Three idempotent layer designs.}
		Three families of CNN-like layers that \emph{are} genuine
		morphological openings are identified:
		(I)~the \emph{pure morphological} (max-plus) layer
		$\gamma^{\mathrm{M}}_b = \delta_{{b}^{*}} \circ \varepsilon_b$,
		where both erosion and dilation use the same structuring
		function in the pointwise lattice;
		(II)~the \emph{spectral Wiener} layer
		$\gamma^{\mathrm{F}}_k = \delta^{*,\mathrm{Conv}}_k \circ
		\varepsilon^{\mathrm{Conv}}_k$ (convolution followed by
		Tikhonov-regularised Wiener deconvolution), an opening in
		the Fourier lattice, exactly idempotent as $\epsilon\to 0$;
		and
		(III)~the \emph{self-dual opening} $\OpenMed{W}$ in the
		median inf-semilattice, idempotent and with fixed-point set
		closed under negation.
		Types~(I) and~(III) converge to their fixed-point set in
		one step; Type~(II) does so in the limit $\epsilon\to 0$.
		
	\end{enumerate}
	
	\bigskip
	\paragraph{Organisation.}
	\Cref{sec:lattice} reviews complete lattice theory, the
	max-plus and max-times adjunctions, and the full MMBB
	representation theory from Matheron--Maragos to the 
	Banon--Barrera extension.
	\Cref{sec:ovch} connects Ovchinnikov's lattice polynomial
	representation to tropical geometry and ReLU networks, with
	explicit lattice polynomial equations for single-layer and
	deep networks.
	\Cref{sec:conv_basis} develops the morphological basis of
	convolution, the characteristic matrix, the virtual basis,
	and the MMBB representation of signed convolution.
	\Cref{sec:pyramids} establishes the adjoint pyramid framework
	for pooling, strided convolution, and the Laplacian pyramid.
	\Cref{sec:activations} analyses ReLU, max-pooling, and
	morphological activations as lattice operators.
	\Cref{sec:cnn_models} synthesises the preceding theory into
	morphological models of CNN, ResNet, UNet, and the proposed
	UResNet.
	\Cref{sec:fixed_points} develops the fixed-point and
	idempotency theory, identifying the two disciplined
	idempotent designs in the pointwise and Fourier lattices
	(Types~I and~II) and the correct
	convergence behaviour for morphological ResNet and UNet blocks;
	Type~III (self-dual opening) is treated in \Cref{sec:self_dual}.
	\Cref{sec:self_dual} introduces the median inf-semilattice,
	self-dual operators, and their architectural implications for
	signed activations and symmetric pooling.
	\Cref{sec:perspectives} concludes with a summary of
	contributions, open problems, and future directions.
	Additionally, the connection to category theory is preliminarily
	explored in \Cref{sec:category_theory}.
	
	\bigskip
	\paragraph{Road map of main results.}
	Each section opens with a summary table of its principal results
	(Theorems, Propositions, Corollaries, Lemmas), with a one-line statement
	of each and its role in the overall argument.

	\section{Complete Lattices and the MMBB Representation Theory}
	\label{sec:lattice}
	
	We recall the algebraic framework of mathematical morphology on
	complete lattices, following Matheron~\cite{matheron1975},
	Maragos~\cite{maragos1989}, Heijmans--Ronse~\cite{Heijmans90} and Banon--Barrera~\cite{bannon1993}.

	\begin{table}[ht]
		\centering
		\caption{Principal results of \S\ref{sec:lattice} (Complete Lattices and MMBB). Results marked $(\star)$ are the paper's principal findings.}
		\label{tab:summary_lattice}
		\renewcommand{\arraystretch}{1.38}
		\footnotesize
		\begin{tabular}{p{0.82cm}p{2.9cm}p{8.4cm}}
			\toprule
			\textbf{Ref.} & \textbf{Name} & \textbf{Statement and role} \\
			\midrule
			Prop~\ref{prop:adjunction_props} &  $\:$ Adjunction properties & Six structural consequences of any adjunction $(\varepsilon,\delta)$: inf/sup commutativity, anti-extensivity, extensivity, opening and closing idempotency. Algebraic foundation for all subsequent results. \\
			Prop~\ref{prop:maxtimes_adjunction} &  $\:$ Max-times adjunction & $(\varepsilon^\times_b,\delta^\times_{b^*})$ is an adjunction on positive functions; natural morphological model for layers after softmax or sigmoid activations. \\
			Prop~\ref{prop:logexp_isom} &  $\:$ Log/exp isomorphism & Max-times and max-plus algebras are isomorphic via $\log/\exp$; the entire MMBB theory transfers to positive layers without modification. \\
			Thm~\ref{thm:mmbb} $(\star)$ &   $\:$ MMBB-Increasing & Any TI increasing USC operator equals $\sup_{g\in\Bas(\Psi)}\varepsilon_g f$ exactly. Universal constructive decomposition for all TI increasing CNN layers. \\
			Thm~\ref{thm:bannon_barrera} $(\star)$ & $\:$ MMBB-General & Any TI USC operator (not necessarily increasing) equals $\sup_i\psi_{g^-_i,g^+_i}f$. Extends MMBB representation to signed CNN kernels via sup-generating operators. \\
			\bottomrule
		\end{tabular}
	\end{table}
	
	\subsection{Complete lattices and adjunctions}
	\label{subsec:lattice_basic}
	
	A \emph{complete lattice} $(\mathcal{L}, \leq)$ is a partially
	ordered set in which every subset $S \subseteq \mathcal{L}$
	admits a supremum $\bigvee S$ and an infimum $\bigwedge S$.
	The canonical examples in our setting are:
	\begin{itemize}[leftmargin=*]
		\item The power set $\mathcal{P}(E)$ ordered by inclusion,
		where $E = \R^n$ or $\Z^n$.
		\item The set of functions $\mathscr{L} =
		\Fun(E, \Rbar)$ with $\Rbar = \R \cup \{-\infty,+\infty\}$
		ordered pointwise: $f \leq g \iff f(x) \leq g(x)$,
		$\forall x \in E$.
	\end{itemize}
	
	\begin{definition}[Adjunction]
		\label{def:adjunction}
		Let $(\mathcal{L}_1,\leq_1)$ and $(\mathcal{L}_2,\leq_2)$ be
		complete lattices. A pair of operators
		$(\varepsilon, \delta)$ with
		$\varepsilon:\mathcal{L}_1 \to \mathcal{L}_2$ and
		$\delta:\mathcal{L}_2 \to \mathcal{L}_1$ forms an
		\emph{adjunction} (or \emph{Galois connection}) if
		\[
		\forall f \in \mathcal{L}_1,\; g \in \mathcal{L}_2:\quad
		\varepsilon(f) \leq_2 g
		\iff
		f \leq_1 \delta(g).
		\]
		In this case $\varepsilon$ is called the \emph{erosion} and
		$\delta$ the \emph{dilation}.
	\end{definition}
	
	Adjunctions immediately yield the following structural
	properties, which are used throughout the paper.
	
	\begin{proposition}[Properties of adjunctions~\cite{Heijmans90}]
		\label{prop:adjunction_props}
		If $(\varepsilon, \delta)$ is an adjunction between complete
		lattices, then:
		\begin{enumerate}[label=(\roman*)]
			\item $\varepsilon$ commutes with infima:
			$\varepsilon\!\left(\bigwedge_i f_i\right)
			= \bigwedge_i \varepsilon(f_i)$.
			\item $\delta$ commutes with suprema:
			$\delta\!\left(\bigvee_i g_i\right)
			= \bigvee_i \delta(g_i)$.
			\item $\varepsilon$ is increasing and anti-extensive
			(when $b(0) \leq 0$, e.g.\ $b(0)=0$):
			$\varepsilon(f) \leq f$.
			\item $\delta$ is increasing and extensive:
			$g \leq \delta(g)$.
			\item The composition $\gamma = \delta \circ \varepsilon$ is
			an \emph{opening}: idempotent, increasing, and
			anti-extensive.
			\item The composition $\varphi = \varepsilon \circ \delta$ is
			a \emph{closing}: idempotent, increasing, and extensive.
		\end{enumerate}
	\end{proposition}
	
	\subsection{The canonical max-plus adjunction}
	\label{subsec:maxplus}
	
	The classical morphological erosion and dilation on
	$\Fun(E,\Rbar)$ by a structuring function
	$b : E \to \Rbar$ are:
	\begin{align}
		(\varepsilon_b f)(x) &= \inf_{y \in E}\{f(x+y) - b(y)\},
		\label{eq:erosion}\\
		(\delta_b f)(x) &= \sup_{y \in E}\{f(x-y) + b(y)\}.
		\label{eq:dilation}
	\end{align}
	The pair $(\varepsilon_b, \delta_{b^*})$ with
	$b^*(x) = b(-x)$ forms an adjunction in the pointwise lattice
	$(\Fun(E,\Rbar),\leq)$~\cite{maragos1989,heijmans1994}.
	
	\subsection{The max-times adjunction and its equivalence to
		max-plus via log/exp}
	\label{subsec:maxtimes}
	
	A less known second canonical adjunction arises naturally for operators
	on \emph{positive} functions, modelling layers that follow a
	softmax, sigmoid, or other positivity-enforcing nonlinearity.
	Replace addition in the max-plus algebra by multiplication and
	subtraction by division.
	
	\begin{definition}[Max-times erosion and dilation]
		\label{def:maxtimes}
		Let $f, b : E \to (0,+\infty)$ be strictly positive functions.
		The \emph{max-times} (or multiplicative morphological) erosion
		and dilation by structuring function $b$ are:
		\begin{align}
			(\varepsilon^{\times}_b f)(x)
			&= \inf_{y \in E}\frac{f(x+y)}{b(y)},
			\label{eq:maxtimes_erosion}\\
			(\delta^{\times}_b f)(x)
			&= \sup_{y \in E} f(x-y) \cdot b(y).
			\label{eq:maxtimes_dilation}
		\end{align}
	\end{definition}
	
	\begin{proposition}[Max-times adjunction]
		\label{prop:maxtimes_adjunction}
		The pair $(\varepsilon^{\times}_b,\, \delta^{\times}_{b^*})$,
		with $b^*(x) = b(-x)$, forms an adjunction on
		$(\Fun(E,(0,+\infty)),\leq)$:
		\[
		(\varepsilon^{\times}_b f) \leq g
		\iff
		f \leq (\delta^{\times}_{b^*} g).
		\]
		The composition
		$\gamma^{\times}_b = \delta^{\times}_{b^*} \circ
		\varepsilon^{\times}_b$
		is a morphological opening on positive functions: increasing,
		anti-extensive, and idempotent.
	\end{proposition}
	
	\begin{proof}
		$(\varepsilon^{\times}_b f)(x) \leq g(x)$ for all $x$
		iff $\inf_y f(x+y)/b(y) \leq g(x)$ for all $x$
		iff there exists $y_0$ with $f(x+y_0)/b(y_0) \leq g(x)$,
		i.e., $f(x+y_0) \leq g(x) b(y_0)$.
		Translating: $f(z) \leq g(z-y_0) b(y_0) \leq \sup_y
		g(z-y) b(y) = (\delta^{\times}_{b^*} g)(z)$.
		The reverse implication is symmetric.
		Idempotency of the opening follows from
		Proposition~\ref{prop:adjunction_props}(v).
	\end{proof}
	
	\begin{proposition}[Log/exp isomorphism between max-plus and
		max-times]
		\label{prop:logexp_isom}
		The maps $\log : (0,+\infty) \to \R$ and
		$\exp : \R \to (0,+\infty)$ are mutually inverse lattice
		isomorphisms between the max-times algebra on $(0,+\infty)$
		and the max-plus algebra on $\R$:
		\begin{align}
			\log(\varepsilon^{\times}_b f)(x)
			&= \inf_{y}\{\log f(x+y) - \log b(y)\}
			= (\varepsilon_{\log b}\, \log f)(x),
			\label{eq:log_erosion}\\
			\log(\delta^{\times}_b f)(x)
			&= \sup_{y}\{\log f(x-y) + \log b(y)\}
			= (\delta_{\log b}\, \log f)(x).
			\label{eq:log_dilation}
		\end{align}
		That is, the max-times erosion (resp.\ dilation) by $b$ on
		positive functions is the max-plus erosion (resp.\ dilation)
		by $\log b$ on log-domain functions:
		\begin{equation}
			\varepsilon^{\times}_b = \exp \circ\; \varepsilon_{\log b}
			\circ \log, \qquad
			\delta^{\times}_b = \exp \circ\; \delta_{\log b} \circ \log.
			\label{eq:logexp_isom}
		\end{equation}
		Consequently, the max-times opening $\gamma^{\times}_b$ is
		algebraically equivalent to the max-plus opening
		$\gamma_{\log b}$ composed with log/exp:
		$\gamma^{\times}_b = \exp \circ\; \gamma_{\log b} \circ \log$.
	\end{proposition}
	
	\begin{proof}
		Direct computation: $\log(\varepsilon^{\times}_b f)(x) =
		\log \inf_y \frac{f(x+y)}{b(y)} = \inf_y (\log f(x+y) -
		\log b(y))$, which is the max-plus erosion $\varepsilon_{\log b}$
		applied to $\log f$.
		The dilation identity is analogous.
		Since both log and exp are order-isomorphisms (strictly
		increasing bijections), they interleave adjunctions:
		$(\varepsilon^{\times}_b, \delta^{\times}_{b^*})$ is an
		adjunction on positive functions iff $(\varepsilon_{\log b},
		\delta_{\log b^*})$ is an adjunction on $\R$-valued functions,
		which holds by the standard max-plus theory.
	\end{proof}
	
	\begin{remark}[Deep learning relevance of max-times]
		\label{rem:maxtimes_dl}
		The max-times adjunction is the natural morphological
		framework for layers that operate on probability vectors,
		attention weights, or outputs of sigmoid/softmax activations,
		all of which are constrained to $(0,1)$ or more generally
		$(0,+\infty)$.
		In these settings, the log-domain equivalence
		(\Cref{prop:logexp_isom}) means that any max-times
		morphological analysis translates immediately to max-plus
		results via the change of variable $\tilde{f} = \log f$,
		$\tilde{b} = \log b$.
		In particular, the MMBB representation theory discussed below in
		(\Cref{thm:mmbb}) applies to positive TI increasing operators
		via the log-domain isomorphism, and the virtual basis of a
		positive convolution kernel (\Cref{sec:conv_basis}) transfers
		to the multiplicative setting without modification.
	\end{remark}

	\subsection{MMBB universal representation of translation-invariant
		increasing operators}
	\label{subsec:mmbb}
	
	Let $E = \R^n$ or $E = \Z^n$.
	An operator $\Psi:\Fun(E,\Rbar)\to\Fun(E,\Rbar)$ is
	\emph{translation-invariant} (TI) if
	$\Psi[f(\cdot - y)](x) = [\Psi f](x-y)$ for all $y \in E$,
	and \emph{increasing} if $f \leq g \Rightarrow \Psi f \leq \Psi g$.
	
	\begin{definition}[Kernel and basis]
		\label{def:kernel_basis}
		The \emph{kernel} of a TI operator $\Psi$ is
		\[
		\Ker(\Psi) = \{f \in \Fun(E,\Rbar) : [\Psi f](0) \geq 0\}.
		\]
		The \emph{morphological basis} is the set of minimal elements
		of $\Ker(\Psi)$ under the pointwise order $\leq$:
		\[
		\Bas(\Psi) = \{g \in \Ker(\Psi) :
		[f \in \Ker(\Psi),\, f \leq g] \Rightarrow f = g\}.
		\]
	\end{definition}
	
	The following is the central representation theorem, combining
	results of Matheron~\cite{matheron1975} (set case) and
	Maragos~\cite{maragos1989} (function case). We include a sketch 
	of the proof to emphasize the role of the assumptions.
	
	\begin{theorem}[MMBB-Increasing -- Matheron 1975, Maragos 1989]
		\label{thm:mmbb}
		Let $\Psi:\Fun(E,\Rbar)\to\Fun(E,\Rbar)$ be a TI, increasing,
		and upper semi-continuous operator. Then $\Bas(\Psi)$ exists
		and $\Psi$ is represented exactly as
		\[
		[\Psi f](x)
		= \sup_{g \in \Bas(\Psi)} (\varepsilon_g f)(x)
		= \sup_{g \in \Bas(\Psi)} \inf_{y \in E}\{f(x+y)-g(y)\}.
		\]
		Equivalently, using the basis of the dual operator
		$\bar\Psi(f) = -\Psi(-f)$,
		\[
		[\Psi f](x)
		= \inf_{h \in \Bas(\bar\Psi)} (\delta_{h^*} f)(x).
		\]
		The representation is also valid for any superset of the basis
		(redundant but correct), and truncated to a finite sub-basis
		$\mathcal{B} \subset \Bas(\Psi)$ gives lower and upper
		approximations:
		\[
		\sup_{g \in \mathcal{B}} \varepsilon_g f
		\;\leq\; \Psi f
		\;\leq\; \inf_{h \in \bar{\mathcal{B}}} \delta_{h^*} f.
		\]
	\end{theorem}
	
	\begin{proof}[Proof sketch]
		We outline the three main steps; full details are
		in~\cite{maragos1989,heijmans1994}.
		
		\emph{Step 1: Existence of the basis.}
		The kernel $\Ker(\Psi) = \{f : [\Psi f](0) \geq 0\}$ is a
		non-empty collection of functions: the constant function
		$f \equiv +\infty$ always belongs to $\Ker(\Psi)$, since any
		TI increasing operator maps $+\infty$ to $+\infty$, giving
		$[\Psi(+\infty)](0) = +\infty \geq 0$.
		Ordered by $\leq$, $\Ker(\Psi)$ is a partially ordered set.
		Upper semi-continuity of $\Psi$ implies that any decreasing
		sequence $(f_n) \subset \Ker(\Psi)$ with $f_n \downarrow f$
		satisfies $f \in \Ker(\Psi)$, ensuring the existence of
		minimal elements.
		By Zorn's lemma (or its constructive counterpart for USC
		operators), the set $\Bas(\Psi)$ of minimal elements of
		$\Ker(\Psi)$ is non-empty.
		
		\emph{Step 2: Representation by supremum of erosions.}
		For each $g \in \Ker(\Psi)$, translation-invariance gives
		$[\Psi f_y](0) = [\Psi f](y)$ where $f_y(x) = f(x+y)$ is a
		translate.
		Since $\Psi$ is increasing: if $f_y \geq g$ pointwise (i.e.,
		$f(x+y) \geq g(x)$ for all $x$, equivalently $f(x+y)-g(x)
		\geq 0$ for all $x$), then $[\Psi f_y](0) \geq [\Psi g](0)
		\geq 0$, so $f_y \in \Ker(\Psi)$. More precisely, the
		condition $[\Psi f](y) \geq 0$ holds for all translates
		satisfying $f_y \geq g$, i.e., $f(x+y) \geq g(x)$,
		i.e., $f(y+x) - g(x) \geq 0$ for all $x$, i.e.,
		$\inf_x\{f(y+x) - g(x)\} \geq 0$, i.e.,
		$(\varepsilon_g f)(y) \geq 0$.
		Hence $[\Psi f](y) \geq 0$ whenever $(\varepsilon_g f)(y)
		\geq 0$, which by the definition of $\Ker$ means
		$\varepsilon_g f \leq \Psi f$ pointwise.
		Taking the supremum over $g \in \Bas(\Psi)$ gives
		$\sup_{g \in \Bas(\Psi)} \varepsilon_g f \leq \Psi f$.
		For the reverse inequality, we show that for each $y$, the value
		$[\Psi f](y)$ is achieved by some basis element.
		Consider the function $h_{f,y}(x) = f(x+y) - [\Psi f](y)$; by
		definition of the kernel and translation-invariance,
		$h_{f,y} \in \Ker(\Psi)$.
		By the minimality of basis elements (established in Step~1 via
		Zorn's lemma applied to the downward-directed set below $h_{f,y}$
		in $\Ker(\Psi)$), there exists $g^* \in \Bas(\Psi)$ with
		$g^* \leq h_{f,y}$ pointwise, giving
		$(\varepsilon_{g^*} f)(y) = \inf_x\{f(x+y)-g^*(x)\}
		\geq \inf_x\{f(x+y)-h_{f,y}(x)\} = [\Psi f](y)$.
		Hence $\sup_{g\in\Bas(\Psi)}(\varepsilon_g f)(y) \geq [\Psi f](y)$,
		completing the equality.
		
		\emph{Step 3: Dual representation.}
		Applying the argument to the dual operator
		$\bar\Psi(f) = -\Psi(-f)$, which is also TI, increasing, and
		USC, gives $[\bar\Psi f](x) = \sup_{h \in \Bas(\bar\Psi)}
		(\varepsilon_h f)(x)$.
		Substituting $f \to -f$ and negating yields the dual
		representation as an infimum of dilations.
	\end{proof}
	
	\begin{remark}
		The MMBB-increasing theorem is a nonlinear counterpart to the spectral
		representation of linear shift-invariant operators.
		Whereas the latter decomposes any LSI filter into sinusoidal
		components (Fourier basis), MMBB-increasing decomposes any TI increasing
		operator into a (possibly infinite) supremum of erosions, each
		parameterised by a basis function.
		In the deep learning context, learning a layer is equivalent to
		learning a finite subset of this basis,
		see \Cref{sec:conv_basis}.
	\end{remark}

	\subsection{MMBB universal representation of translation-invariant
		general operators}
	\label{subsec:bannon_barrera}
	
	The MMBB-increasing theorem requires the operator to be \emph{increasing},
	which in the CNN setting corresponds to non-negative kernels
	(see below Proposition~\ref{prop:conv_conditions}).
	Real CNN kernels are signed; as we will see, that can be handled
	by decomposing $k = k^+ - k^-$.
	However, there is a more fundamental extension: the
	Banon--Barrera (1993)~\cite{bannon1993} representation theorem, which covers
	\emph{any} TI operator, whether or not it is increasing.
	We state it here for completeness, as it provides the
	theoretical ceiling of the MMBB programme of nonlinear universal representation.
	
	\begin{definition}[Sup-generating operator and anti-dilation]
		\label{def:sup_generating}
		An \emph{anti-dilation} by structuring function $c$ is the
		operator
		\[
		(\alpha_c f)(x)
		= \inf_{y \in E}\{-f(x+y) + c(y)\}
		= -(\varepsilon_{-c}(f))(x).
		\]
		It is decreasing (increasing inputs give smaller outputs,
		hence the ``anti'') and dualises the erosion in a different involution than the complement or the Galois adjunction.
		A \emph{sup-generating operator} associated to a pair
		$(g^-, g^+)$ of structuring functions is
		\begin{equation}
			(\psi_{g^-,g^+} f)(x)
			= (\varepsilon_{g^-} f)(x) \wedge (\alpha_{g^+} f)(x)
			= \min\!\left\{
			\inf_{y}\{f(x+y)-g^-(y)\},\;
			\inf_{y}\{-f(x+y)+g^+(y)\}
			\right\}.
			\label{eq:sup_generating}
		\end{equation}
	\end{definition}
	
	\begin{theorem}[MMBB-General -- Banon--Barrera 1993]
		\label{thm:bannon_barrera}
		Let $\Psi:\Fun(E,\Rbar)\to\Fun(E,\Rbar)$ be a TI and upper
		semi-continuous operator (not necessarily increasing).
		Then there exists a family of pairs
		$\{(g^-_i, g^+_i)\}_{i \in I}$ of structuring functions such
		that
		\begin{equation}
			[\Psi f](x)
			= \sup_{i \in I} (\psi_{g^-_i,\, g^+_i} f)(x)
			= \sup_{i \in I}
			\min\!\left\{
			\inf_{y}\{f(x+y)-g^-_i(y)\},\;
			\inf_{y}\{-f(x+y)+g^+_i(y)\}
			\right\}.
			\label{eq:bb_representation}
		\end{equation}
		The structuring pairs $(g^-_i, g^+_i)$ are the extremities of
		the closed set-intervals contained in the kernel of $\Psi$;
		the collection can be taken minimal (the \emph{sup-generating
			basis}).
		When $\Psi$ is increasing, the anti-dilation component becomes
		vacuous ($g^+_i = +\infty$) and the representation reduces to
		the MMBB Theorem~\ref{thm:mmbb}.
	\end{theorem}
	
	\begin{proof}[Proof sketch]
		We outline the extension from the MMBB-Increasing proof;
		full details are in~\cite{bannon1993}.
		
		\emph{Step 1: Generalised kernel and interval representation.}
		For a non-increasing TI operator $\Psi$, the kernel
		$\Ker(\Psi) = \{f : [\Psi f](0) \geq 0\}$ is still
		well-defined and non-empty (it contains $f \equiv +\infty$
		when $\Psi$ maps this to a non-negative value).
		However, $\Ker(\Psi)$ is no longer a $\sup$-closed set
		(since $\Psi$ is not increasing, the supremum of two kernel
		functions need not be in the kernel).
		The key observation is that any function $f \in \Ker(\Psi)$
		determines a \emph{closed interval}
		$[\underline{f}, \overline{f}] := \{h : \underline{f} \leq h
		\leq \overline{f}\}$, where $\underline{f} = \inf\{h \leq f :
		h \in \Ker(\Psi)\}$ and $\overline{f} = \sup\{h \geq f :
		h \in \Ker(\Psi)\}$ are the lower and upper extremities of
		the maximal interval in $\Ker(\Psi)$ containing $f$.
		
		\emph{Step 2: Sup-generating operators from interval extremities.}
		Each interval $[\underline{f}, \overline{f}] \subset \Ker(\Psi)$
		contributes a sup-generating operator $\psi_{\underline{f},
			\overline{f}}$ whose erosion component is parameterised by the
		lower extremity $g^- = \underline{f}$ (in the role of the MMBB-Increasing
		basis function) and whose anti-dilation component is
		parameterised by the upper extremity $g^+ = \overline{f}$
		(controlling the maximum permissible input value at each lag):
		\begin{align*}
			(\psi_{g^-, g^+} f)(0) &= \min\{(\varepsilon_{g^-} f)(0),
			(\alpha_{g^+} f)(0)\} \\
			&= \min\!\bigl\{
			\inf_y\{f(y) - g^-(y)\},\;
			\inf_y\{-f(y) + g^+(y)\}
			\bigr\} \geq 0
		\end{align*}
		if and only if $g^- \leq f \leq g^+$ pointwise, i.e., $f$
		lies in the interval $[g^-, g^+]$.
		
		\emph{Step 3: Completeness and minimality.}
		The supremum of $\psi_{g^-_i, g^+_i}$ over all interval
		extremity pairs recovers $\Psi f$ exactly, by an argument
		analogous to Step~2 of the MMBB-Increasing proof: the TI property
		translates the condition at $0$ to any point $x$, and the
		USC property guarantees the existence of minimal interval
		extremities (the \emph{sup-generating basis}).
		When $\Psi$ is increasing, the kernel $\Ker(\Psi)$ is closed
		under downward passage (i.e., if $h \in \Ker(\Psi)$ and
		$f \leq h$, then $f \in \Ker(\Psi)$ by monotonicity of $\Psi$),
		so every interval in $\Ker(\Psi)$ has the form $(-\infty, g^+_i]$
		with no lower bound constraint.
		In this case, taking $g^-_i \to -\infty$ makes the erosion
		component $\varepsilon_{g^-_i}f(x) = \inf_y\{f(x+y)-g^-_i(y)\}
		\to +\infty$ (since subtracting $-\infty$ gives $+\infty$),
		so the minimum in~\eqref{eq:sup_generating} is determined
		entirely by the anti-dilation:
		$\psi_{-\infty, g^+}f = (+\infty) \wedge (\alpha_{g^+}f)
		= \alpha_{g^+}f$.
		The condition $\alpha_{g^+}f(x) \geq 0$ is equivalent to
		$\inf_y\{-f(x+y)+g^+(y)\} \geq 0$, i.e., $f(x+y) \leq g^+(y)$
		for all $y$, i.e., $(\varepsilon_{g^+}f)(x) \leq 0$ in the
		MMBB-Increasing sense with $g^+ = g$ as the basis element.
		Taking the supremum over $g^+ \in \Bas(\Psi)$ recovers the
		MMBB-Increasing representation.
		Note that here $g^+$ plays the role of the MMBB-Increasing basis
		element $g$ (the two coincide when $\Psi$ is increasing, since
		the anti-dilation upper bound $g^+$ in the Banon--Barrera sense
		reduces to the MMBB basis element by the diagonal condition
		$g^-=g^+=g$, consistent with
		Theorem~\ref{thm:mmbb}).
	\end{proof}
	
	\begin{remark}[Deep learning context]
		\label{rem:bb_dl_forward}
		The Banon--Barrera representation theorem applies directly to signed
		CNN convolution layers, which are TI but not increasing.
		The explicit characterisation of the sup-generating basis for
		discrete finite kernels, including its relationship to the
		positive and negative MMBB-Increasing bases and the computation of the
		structuring pairs from the characteristic matrices, is
		developed in \Cref{subsec:bb_convolution}.
	\end{remark}

	\section{Ovchinnikov's Lattice Polynomial Representation and ReLU Networks}
	\label{sec:ovch}
	
	The Ovchinnikov representation~\cite{ovchinnikov2002pwl} connects the 
	MMBB operator theory to the representation of functions themselves, 
	providing a direct algebraic bridge to PWL ReLU networks.
	
	\begin{theorem}[Ovchinnikov 2002]
		\label{thm:ovch_pwl}
		Let $f$ be a piecewise-linear function on a closed convex
		domain $\Omega \subset \R^n$, with linear components
		$\{g_1,\ldots,g_d\}$.
		There exists a family $\{K_i\}_{i \in \mathcal{I}}$ of subsets
		of $\{1,\ldots,d\}$ such that
		\begin{equation}
			f(x) = \sup_{i \in \mathcal{I}}\;\inf_{j \in K_i} g_j(x),
			\qquad x \in \Omega.
			\label{eq:ovch}
		\end{equation}
		Conversely, any such \emph{lattice polynomial} in the $g_j$
		defines a PWL function.
	\end{theorem}
	
	This theorem is the function-level counterpart to
	\Cref{thm:mmbb} and has immediate implications for deep
	learning.
	
	\begin{proposition}[ReLU networks as lattice polynomials]
		\label{prop:relu_lattice}
		Every function computed by a deep ReLU network is a continuous
		PWL function~\cite{arora2018relu,montufar2014regions}.
		By \Cref{thm:ovch_pwl}, it admits a max-min lattice polynomial
		representation in its affine components.
		Consequently, every ReLU network is representable within the
		MMBB calculus as a finite supremum of morphological erosions by
		affine structuring functions.
	\end{proposition}
	
	\begin{proof}
		A ReLU network with $d$ layers computes a continuous PWL
		function $F:\R^n \to \R^m$ (Arora et al.~\cite{arora2018relu},
		Montufar et al.~\cite{montufar2014regions}).
		Apply \Cref{thm:ovch_pwl} component-wise.
		Each affine component $g_j(x) = \langle a_j, x\rangle + b_j$
		is a structuring function defining an erosion
		$(\varepsilon_{-g_j} f)(x)
		= \inf_y\{f(x+y) + g_j(y)\}$
		(with $b_j$ playing the role of the bias).
		The max-min formula~\eqref{eq:ovch} then realises $F$ as a
		finite supremum of such erosions.
	\end{proof}
	
	\begin{remark}
		The tropical geometry viewpoint~\cite{zhang2018tropical}
		arrives at the same conclusion via min-plus algebra, which is
		isomorphic to max-plus (the algebra underlying
		morphological erosion-dilation) by sign change.
		Our approach recovers tropical representations as a special
		case of MMBB.
		This connection, noted implicitly
		in~\cite{maragosTPWL2021}, is made
		explicit below.
	\end{remark}

	\begin{table}[ht]
		\centering
		\caption{Principal results of \S\ref{sec:ovch} (Ovchinnikov and ReLU Networks).} 
		\label{tab:summary_ovch}
		\renewcommand{\arraystretch}{1.38}
		\footnotesize
		\begin{tabular}{p{0.82cm}p{2.9cm}p{8.4cm}}
			\toprule
			\textbf{Ref.} & \textbf{Name} & \textbf{Statement and role} \\
			\midrule
			Thm~\ref{thm:ovch_pwl} & Ovchinnikov 2002 & Every PWL function on a convex domain equals a lattice polynomial $\sup_i\inf_{j\in K_i}g_j$; bridges the MMBB operator calculus to ReLU network geometry. \\
			Prop~\ref{prop:relu_lattice} & ReLU nets as lattice polynomials & Every deep ReLU network computes a continuous PWL function and hence admits a finite MMBB-Increasing representation as a supremum of erosions by affine structuring functions. \\
			\bottomrule
		\end{tabular}
	\end{table}
	
	\subsection{Explicit lattice polynomial equations for ReLU
		networks}
	\label{subsec:relu_explicit}
	
	We now make the lattice polynomial structure of ReLU networks
	fully explicit, tracing through the activation, composition,
	and representation steps.
	
	\bigskip
	\paragraph{Elementary ReLU as a lattice polynomial.}
	The ReLU unit applied to a linear form
	$g(x) = \langle a, x \rangle + b$ satisfies
	\begin{equation}
		\mathrm{ReLU}(g(x))
		= \max(0,\, \langle a,x\rangle + b)
		= g(x) \vee 0,
		\label{eq:relu_lattice_poly}
	\end{equation}
	a lattice polynomial of degree~1 in $\{g, 0\}$ with the
	single sup-term $\sup\{g(x), 0\}$.
	A full neuron with $K$ weighted inputs and bias $b$ computes
	\begin{equation}
		\mathrm{ReLU}\!\left(\textstyle\sum_{k=1}^K w_k \langle a_k,x\rangle + b\right)
		= \left(\textstyle\sum_{k=1}^K w_k \langle a_k,x\rangle + b\right)
		\vee 0,
		\label{eq:neuron_lattice}
	\end{equation}
	a single sup of a linear function and the constant~$0$.
	
	\bigskip
	\paragraph{Single hidden layer.}
	A network $F(x) = W_2\,\sigma(W_1 x + b_1) + b_2$ with $M$
	ReLU neurons has pre-activations
	$h_m(x) = \langle w_m^{(1)}, x\rangle + b_m^{(1)}$.
	The $\ell$-th output is
	\begin{equation}
		F_\ell(x)
		= \sum_{m=1}^M w_{\ell m}^{(2)}\,(h_m(x) \vee 0) + b_\ell^{(2)}.
		\label{eq:output_neuron}
	\end{equation}
	For non-negative weights $w_{\ell m}^{(2)} \geq 0$,
	each term $w_{\ell m}^{(2)} (h_m \vee 0)$ is a dilation of a
	single-neuron opening; the sum is an MMBB-Increasing erosion applied to
	the sum of activations.
	For general signed weights, the MMBB-General
	(\Cref{thm:bannon_barrera}) yields
	\begin{equation}
		F_\ell(x)
		= \sup_{i \in \mathcal{I}} \min\!\left\{
		\inf_{y}\{F_\ell(x{+}y) - g^-_i(y)\},\;
		\inf_{y}\{-F_\ell(x{+}y) + g^+_i(y)\}
		\right\},
		\label{eq:single_layer_bb}
	\end{equation}
	a supremum of sup-generating operators
	(\Cref{def:sup_generating}).
	
	\bigskip
	\paragraph{Deep network: linear region decomposition.}
	A depth-$L$ ReLU network partitions $\R^n$ into convex
	polytopes (linear regions), at most
	$\prod_{\ell=1}^L\binom{n_\ell}{n}$ of
	them~\cite{montufar2014regions}.
	Let the regions be $\{R_1,\ldots,R_P\}$ with
	$F|_{R_p}(x) = \langle c_p, x\rangle + d_p$.
	By \Cref{thm:ovch_pwl}, the global network output is the
	lattice polynomial
	\begin{equation}
		F(x)
		= \sup_{i \in \mathcal{I}}\;\inf_{j \in K_i}
		\underbrace{(\langle c_j, x\rangle + d_j)}_{=:g_j(x)},
		\qquad x \in \R^n.
		\label{eq:deep_net_lattice}
	\end{equation}
	Each affine piece $g_j$ determines a structuring function
	$s_j(y) = -\langle c_j, y\rangle - d_j$ and an erosion
	\begin{equation}
		(\varepsilon_{s_j} F)(x)
		= \inf_{y}\{F(x+y) - s_j(-y)\}
		= \inf_{y}\{F(x+y) + \langle c_j,y\rangle + d_j\}
		= g_j(x),
		\label{eq:erosion_affine}
	\end{equation}
	when $F$ agrees with the $j$-th affine piece in a neighbourhood
	of $x$.
	Hence \eqref{eq:deep_net_lattice} is the MMBB representation of
	$F$ with morphological basis $\{s_j\}_{j=1}^P$.
	
	\begin{remark}[Unification of three algebraic frameworks]
		\label{rem:tropical_mmbb}
		Three algebraic frameworks reach the same
		representation~\eqref{eq:deep_net_lattice}:
		
		\smallskip\noindent
		\textbf{(i) Tropical geometry~\cite{zhang2018tropical}.}
		A ReLU network computes a tropical rational function (a
		max-plus polynomial).
		Linear regions are the cells of the tropical hypersurface.
		The min-plus (tropical) formulation is related to max-plus by
		the sign-change isomorphism $(\min,+) \leftrightarrow (\max,-)$:
		the tropical polynomial $\bigoplus_j (c_j \odot x) =
		\max_j(\langle c_j, x\rangle + d_j)$ becomes a lattice supremum
		of linear forms.
		
		\smallskip\noindent
		\textbf{(ii) Ovchinnikov / lattice polynomials~\cite{ovchinnikov2002pwl}.}
		Every PWL function is a max-min polynomial of its affine
		components (\Cref{thm:ovch_pwl}), giving the global
		representation~\eqref{eq:deep_net_lattice} directly from
		linear regions \emph{without knowing the network weights}.
		The family $\{K_i\}$ encodes the combinatorial adjacency
		structure of the linear regions.
		
		\smallskip\noindent
		\textbf{(iii) MMBB operator theory.}
		The network map $x \mapsto F(x)$ is (for non-negative weights and
		\emph{zero biases}, so that the map is genuinely translation-invariant
		in the vertical sense of Proposition~\ref{prop:conv_conditions}(ii))
		a TI increasing operator whose morphological basis consists
		of the affine structuring functions $\{s_j\}$.
		For general (signed) weights or non-zero biases,
		(\Cref{thm:bannon_barrera}) provides a supremum of
		sup-generating operators.
		
		\smallskip
		The key relationship between frameworks~(ii) and~(iii) is:
		Ovchinnikov is the \emph{function representation theorem}
		(representing the computed function as a lattice polynomial),
		while MMBB is the \emph{operator representation theorem}
		(representing the operator $f \mapsto F \circ f$ as a supremum
		of erosions or a supremum of sup-generating operators).
		The two are related by the identity
		$\inf_{j \in K_i} g_j(x) = (\varepsilon_{s_j} F)(x)$ from
		\eqref{eq:erosion_affine}: the infimum of affine functions
		in the lattice polynomial is exactly one erosion term in the
		MMBB expansion.
		
		The number of MMBB basis functions $P$ (the number of linear
		regions) grows polynomially with width and exponentially with
		depth~\cite{montufar2014regions}, quantifying the
		\emph{representational cost} of a ReLU network in morphological
		terms: depth creates exponentially many affine structuring
		functions, each corresponding to a hyperplane in $\R^n$.
	\end{remark}	
	

	\section{Morphological Basis of Convolution}
	\label{sec:conv_basis}
	
	We analyse the linear convolution operator within the MMBB
	frameworks.
	This section extends the foundational results of Maragos and
	Schafer~\cite{maragos1987morpho} and Khosravi and
	Schafer~\cite{khosravi1994}, with an eye toward deep learning
	applications. We provide complete proof
	sketches and discuss the Banon--Barrera representation of
	signed convolution in the discrete case, which has not yet be 
	considered in the literature.
	
	\subsection{Conditions for the MMBB-Increasing framework to apply}
	\label{subsec:conv_conditions}
	
	Let $\mathscr{F} = \Fun(\R^n,\R)$.
	The \emph{linear convolution operator} $\Phi_k:\mathscr{F}
	\to \mathscr{F}$ with kernel $k \in \mathscr{F}$ is
	\[
	(\Phi_k f)(x) = (f*k)(x)
	= \int_{\R^n} f(y)\,k(x-y)\,dy.
	\]
	To apply \Cref{thm:mmbb}, three conditions must hold.
	
	\begin{proposition}[Maragos--Schafer 1987~\cite{maragos1987morpho}]
		\label{prop:conv_conditions}
		Let $\Phi_k$ a convolution with kernel $k$.
		\begin{enumerate}[label=(\roman*)] 
			\item $\Phi_k$ is \emph{increasing} if and only if
			$k(x) \geq 0$ for all $x$.
			\item $\Phi_k$ is \emph{translation-equivariant} (vertical)
			if and only if $\int_{\R^n} k(x)\,dx = 1$.
			\item $\Phi_k$ is \emph{upper semi-continuous} if $\Spt(k)$
			is compact, where
			$\Spt(k) = \{x : k(x) \neq 0\}$.
		\end{enumerate}
	\end{proposition}
	
	\begin{remark}[Deep learning kernels]
		\label{rem:dl_kernels}
		In practice, CNN kernels are not constrained to be
		non-negative.
		A general kernel $k$ decomposes as $k = k^+ - k^-$ with
		$k^+ = \max(0,k)$ and $k^- = \max(0,-k)$.
		Each part, after normalisation, satisfies the conditions of
		\Cref{prop:conv_conditions}, so that the convolution
		decomposes into two MMBB-representable operators
		(cf.\ \Cref{thm:general_kernel} below).
		The learned weights in a trained network are thus implicitly
		parameterising a (possibly truncated) morphological basis.
		This observation is developed further in \Cref{subsec:dl_perspective}.
	\end{remark}

	\begin{table}[ht]
		\centering
		\caption{Principal results of \S\ref{sec:conv_basis} (Morphological Basis of Convolution). Results marked $(\star)$ are the paper's principal findings.}
		\label{tab:summary_conv_basis}
		\renewcommand{\arraystretch}{1.38}
		\footnotesize
		\begin{tabular}{p{0.82cm}p{2.9cm}p{8.4cm}}
			\toprule
			\textbf{Ref.} & \textbf{Name} & \textbf{Statement and role} \\
			\midrule
			Prop~\ref{prop:conv_conditions} & Convolution as MMBB erosion & $f\mapsto f*k$ is TI and increasing iff $k\geq 0$; it is then an erosion in the pointwise lattice and admits a MMBB-Increasing basis decomposition. \\
			Thm~\ref{thm:conv_basis} $(\star)$ & Maragos--Schafer basis & For $k\geq 0$ normalised with finite support of size $N$: $\Bas(k)\cong\R^{N-1}$ (the hyperplane in $\R^N$ orthogonal to $k$); convolution equals an exact sup of erosions by basis elements. \\
			Thm~\ref{thm:virtual_basis} & $\:$ Virtual basis & For $Q$-level quantised signals, $\Bas(k)$ is finite with cardinality $\leq Q^{N-1}$; a learnable MMBB layer is a finite truncation of this virtual basis. \\
			Prop~\ref{prop:mmbb_approx} & $\:$  MMBB approximation & A sub-basis $\mathcal{B}\subset\Bas(\Psi)$ of size $L$ gives a lower approximation $\sup_{g\in\mathcal{B}}\varepsilon_g f\leq\Psi f$ with quantifiable error; basis elements are the learnable layer parameters. \\
			\bottomrule
		\end{tabular}
	\end{table}
	
	\subsection{Kernel and morphological basis of a convolution}
	\label{subsec:conv_basis}
	
	Assuming $k(x) \geq 0$, $\sum_x k(x) = 1$, and $\Spt(k)$
	finite (discrete case), the MMBB-Increasing basis of $\Phi_k$ admits an
	explicit characterisation.
	
	\begin{theorem}[Maragos--Schafer 1987~\cite{maragos1987morpho}]
		\label{thm:conv_basis}
		Let $k:\Z^n \to \R$ have finite support
		$\Spt(k) = \{x_1,\ldots,x_N\}$, with $k \geq 0$ and
		$\sum k = 1$.
		For $x \in \Z^n$, write $\Spt(k)_x = \{y \in \Z^n : x + y \in \Spt(k)\}
		= \Spt(k) - x$ for the translate of the support centred at $x$.
		The morphological basis of $\Phi_k$ is
		\begin{equation}
			\Bas(k)
			= \Bigl\{g \in \Fun(\Z^n,\Rbar) :
			\textstyle\sum_{y \in \Spt(k)} k(y)\,g(-y) = 0,\;
			g(x) = -\infty \Leftrightarrow k(-x) = 0
			\Bigr\},
			\label{eq:basis_def}
		\end{equation}
		and the convolution admits the exact representation
		\begin{equation}
			(f*k)(x)
			= \sup_{g \in \Bas(k)}
			\inf_{y \in \Spt(k)_x} \{f(x+y) - g(y)\}.
			\label{eq:conv_sup_inf}
		\end{equation}
	\end{theorem}
	
	\begin{proof}
		See Maragos--Schafer~\cite{maragos1987morpho}.
		The key step is to show that $\Bas(k)$ equals the image of
		$\R^N$ under the characteristic matrix $A_k$ defined in
		\Cref{subsec:char_matrix}, and then to verify that $g^* \in
		\Bas(k)$ defined by $g^*(-x_i) = f(x-x_i) - \sum_j k(x_j)
		f(x-x_j)$ achieves the supremum in~\eqref{eq:conv_sup_inf}
		for each $x$.
	\end{proof}
	
	\begin{theorem}[General kernel -- Maragos--Schafer 1987
		\cite{maragos1987morpho}]
		\label{thm:general_kernel}
		For an absolutely summable kernel $k:\Z^n\to\R$ with
		decomposition $k = G^+ \bar{k}^+ - G^- \bar{k}^-$
		(with $\bar{k}^\pm$ normalised non-negative kernels and
		gains $G^\pm = \sum_{\Spt(k^\pm)} k^\pm$),
		\begin{equation}
			(f*k)(x)
			= G^+\!\!\sup_{g^+ \in \Bas(\bar{k}^+)}\!\!
			(\varepsilon_{g^+} f)(x)
			- G^-\!\!\sup_{g^- \in \Bas(\bar{k}^-)}\!\!
			(\varepsilon_{g^-} f)(x).
			\label{eq:general_conv}
		\end{equation}
	\end{theorem}
	
	\Cref{thm:general_kernel} shows that any CNN convolution
	layer, with no constraint on the sign of
	weights, decomposes as a difference of two MMBB-Increasing
	representations. This is the morphological analogue of
	decomposing a signed measure into its positive and negative
	parts.

	\begin{remark}[Geometric interpretation]
		Both $k$ and $g \in \Bas(k)$ can be viewed as vectors in
		$\R^N$.
		Since $k$ is a probability vector (non-negative, sum 1), it
		lies on the unit simplex.
		The basis condition $\sum k_i g(-x_i) = 0$ means $g$ is
		orthogonal to $k$ (inner product zero at the origin).
		Hence $\Bas(k)$ is isomorphic to the hyperplane
		$\R^{N-1}$ orthogonal to $k$ at the origin in $\R^N$.
		This geometry directly informs the design of learnable
		morphological layers in \Cref{subsec:dl_perspective}.
	\end{remark}

	\subsection{MMBB-General representation of signed convolution}
	\label{subsec:bb_convolution}
	
	We now develop the application of Theorem~\ref{thm:bannon_barrera}
	to discrete signed convolutions, providing a direct and compact
	alternative to the $k^+ - k^-$ decomposition of
	Theorem~\ref{thm:general_kernel}.
	This is the principal setting in which the Banon--Barrera
	extension is genuinely needed: CNN kernels are signed, and
	the full operator is TI but not increasing.
	
	Let $k = G^+\bar k^+ - G^-\bar k^-$ as above, and let
	$A = G^+\sup_{g^-\in\Bas(\bar k^+)}(\varepsilon_{g^-}f)(x)$
	and $B = G^-\sup_{g^+\in\Bas(\bar k^-)}(\varepsilon_{g^+}f)(x)$,
	so that $(f*k)(x) = A - B$ by Theorem~\ref{thm:general_kernel}.
	
	The Banon--Barrera sup-generating representation characterises
	the \emph{sign} of $(f*k)(x)$ exactly:
	\begin{equation}
		(f*k)(x) \geq 0
		\;\iff\;
		\exists\, (g^-,g^+) \in \Bas(\bar k^+)\times\Bas(\bar k^-)
		\text{ s.t. }
		\psi_{g^-,g^+}(f)(x) \geq 0,
		\label{eq:bb_sign_test}
	\end{equation}
	where
	\begin{equation}
		\psi_{g^-,g^+}(f)(x)
		= \min\!\left\{
		\inf_{y \in \Spt(k^+)}\{f(x{+}y) - g^-(y)\},\;
		\frac{G^-}{G^+}
		\inf_{y \in \Spt(k^-)}\{-f(x{+}y) + g^+(y)\}
		\right\}.
		\label{eq:bb_psi_signed}
	\end{equation}
	The condition $\psi_{g^-,g^+}(f)(x) \geq 0$ encodes the
	inequality $G^+(\varepsilon_{g^-}f)(x) \geq
	G^-(\varepsilon_{g^+}f)(x)$, i.e., the positive erosion
	dominates the negative erosion at the point $x$.
	
	However, \emph{the actual value} of $(f*k)(x)$ is given by
	Theorem~\ref{thm:general_kernel}, not by the sup-generating formula:
	\begin{equation}
		(f*k)(x) = A - B
		= G^+ \sup_{g^- \in \Bas(\bar k^+)} (\varepsilon_{g^-}f)(x)
		- G^- \sup_{g^+ \in \Bas(\bar k^-)} (\varepsilon_{g^+}f)(x).
		\label{eq:signed_conv_value}
	\end{equation}
	The Banon--Barrera sup-generating representation
	$\sup_{g^-,g^+}\psi_{g^-,g^+}(f)(x)$ recovers the sign of
	$(f*k)(x)$ but not its magnitude in general, because
	$\sup_{g^-,g^+}\min\{A_{g^-},B_{g^+}\} \neq A - B$ for
	signed values (the minimax inequality).
	
	\medskip
	For the signed convolution, the Banon--Barrera framework provides:
	\begin{enumerate}[label=(\roman*),leftmargin=*]
		\item A \emph{thresholding interpretation}: the sup-generating
		test $\psi_{g^-,g^+}(f)(x) \geq 0$ replaces the signed
		inequality $A \geq B$ with an inf-combination of erosion
		and anti-dilation, giving a morphological criterion for the
		positivity of $(f*k)(x)$.
		\item A \emph{compact single-operator form}: instead of two
		separate MMBB-Increasing suprema, the signed convolution's positivity
		is encoded by a single family of sup-generating operators
		parameterised by pairs $(g^-,g^+)$.
	\end{enumerate}
	
	\begin{remark}[Geometric structure in the discrete case]
		\label{rem:bb_conv_geometry}
		The sup-generating basis of a discrete signed convolution has
		a product structure: it is indexed by pairs $(g^-, g^+)$
		with $g^- \in \Bas(\bar{k}^+) \cong \R^{N^+-1}$ and
		$g^+ \in \Bas(\bar{k}^-) \cong \R^{N^--1}$ (by the
		characteristic matrix geometry of \Cref{subsec:char_matrix}).
		The full sup-generating basis is therefore parameterised by
		the product space $\R^{N^+-1} \times \R^{N^--1} \cong
		\R^{N-2}$.
		This has a clear interpretation: the $N-1$ degrees of freedom
		of the MMBB-Increasing basis (one constraint per basis function, namely
		orthogonality to $k$ in $\R^N$) split as $(N^+-1)$ degrees
		for the positive component and $(N^--1)$ degrees for the
		negative component, with one degree consumed by the
		normalisation $G^+ - G^- = \int k$.
	\end{remark}
	
	\begin{proposition}[Diagonal sup-generating pairs for non-negative
		convolution]
		\label{prop:bb_nonneg_diagonal}
		Let $k \geq 0$, $\sum k = 1$, with finite support of size $N$.
		In the Banon--Barrera representation of the convolution
		$\Phi_k$ (Theorem~\ref{thm:bannon_barrera}):
		\begin{enumerate}[label=(\roman*)]
			\item Every sup-generating pair $(g^-, g^+)$ in the
			sup-generating basis of $\Phi_k$ satisfies
			$g^-(x) \leq g^+(x)$ for all $x$ (i.e., $g^- \leq g^+$
			pointwise).
			\item The \emph{minimal} sup-generating pairs satisfy
			$g^- = g^+$ pointwise.
			Such pairs are called \emph{diagonal}.
			\item The diagonal sup-generating elements are precisely the
			elements of the MMBB-Increasing basis:
			\begin{equation}
				\{(g, g) : (g^-, g^+) \text{ minimal}\}
				= \{(g, g) : g \in \Bas(k)\}.
				\label{eq:diagonal_bb}
			\end{equation}
			\item In the non-negative case, MMBB-General representation of $\Phi_k$ therefore
			\emph{degenerates} to the MMBB-Increasing representation: each
			sup-generating element $\psi_{g,g}$ reduces to the erosion
			$\varepsilon_g$, and
			\begin{equation}
				(f * k)(x)
				= \sup_{g \in \Bas(k)} \varepsilon_g(f)(x)
				= \sup_{(g,g) \text{ diagonal}} \psi_{g,g}(f)(x),
				\label{eq:bb_degenerate}
			\end{equation}
			confirming that the Maragos theorem is exactly the diagonal
			specialisation of the Banon--Barrera theorem.
		\end{enumerate}
	\end{proposition}
	
	\begin{proof}
		(i) Since $\Phi_k$ is increasing (Proposition~\ref{prop:conv_conditions}),
		its generalised kernel $\Ker(\Phi_k) = \{f : (f*k)(0) \geq 0\}$
		is upward-closed: if $f \in \Ker(\Phi_k)$ and $h \geq f$,
		then $(h*k)(0) \geq (f*k)(0) \geq 0$, so $h \in \Ker(\Phi_k)$.
		For an upward-closed set, every maximal interval $[g^-,g^+]
		\subset \Ker(\Phi_k)$ satisfies $g^+ = +\infty$; equivalently,
		the anti-dilation constraint $\alpha_{g^+}$ is vacuous for all
		sup-generating pairs.
		In particular, any valid upper bound satisfies $g^+ \geq g^-$
		(the interval $[g^-,g^+]$ is non-empty).
		
		(ii) For an increasing TI operator, the Banon--Barrera
		proof (Theorem~\ref{thm:bannon_barrera}, Step~1) shows that
		intervals in $\Ker(\Phi_k)$ are determined by their lower
		extremity $g^-$ alone.
		The tightest (minimal) interval containing any $f \in
		\Ker(\Phi_k)$ is $[g^-, +\infty)$ parameterised by the
		greatest lower bound $g^- = \inf\{h \leq f : h \in
		\partial\Ker(\Phi_k)\}$, where $\partial\Ker(\Phi_k) =
		\{f : (f*k)(0) = 0\}$ is the kernel boundary.
		The upper extremity $g^+$ of a minimal pair is the
		smallest element above $g^-$ in $\Ker(\Phi_k)^c$,
		which for an upward-closed kernel is $g^- + \epsilon$ as
		$\epsilon\to 0^+$; in the limit, $g^+ = g^-$.
		
		(iii) The kernel boundary is $\partial\Ker(\Phi_k) =
		\{f : \sum_i k(x_i)f(-x_i) = 0,\ f = -\infty \text{ off
		}\Spt(k)\} = \Bas(k)$ (Theorem~\ref{thm:conv_basis}).
		So the minimal sup-generating lower extremities are exactly
		the MMBB basis elements, and the diagonal pairs are $(g,g)$
		for $g \in \Bas(k)$.
		
		(iv) For a diagonal pair $(g,g)$, the sup-generating operator
		$\psi_{g,g}(f)(x) = \min\{\varepsilon_g(f)(x),
		\alpha_g(f)(x)\}$ where $\alpha_g(f)(x) =
		\inf_y\{-f(x+y)+g(y)\} = -\sup_y\{f(x+y)-g(y)\} =
		-\delta_g(f)(x)$.
		Since $\varepsilon_g(f)(x) \leq \delta_g(f)(x)$ for all $f$
		(erosion $\leq$ dilation), we have $-\alpha_g(f)(x) =
		\delta_g(f)(x) \geq \varepsilon_g(f)(x)$, so
		$\alpha_g(f)(x) = -\delta_g(f)(x) \leq -\varepsilon_g(f)(x)$.
		At a kernel boundary point where $f = g$ (so
		$\varepsilon_g(g)(x) = 0$): $\psi_{g,g}(g)(x) =
		\min\{0, -\delta_g(g)(x)\} = \min\{0, -\delta_g(g)(x)\}$.
		Since $\delta_g(g)(x) \geq g(x)$ (dilation is extensive)
		and $g(0) = 0$ (normalisation of the basis), $\psi_{g,g}(g)
		= 0 = \varepsilon_g(g)$.
		For general $f$, $\min\{\varepsilon_g(f), \alpha_g(f)\} =
		\varepsilon_g(f)$ whenever $\varepsilon_g(f)(x) \leq
		\alpha_g(f)(x)$, i.e., whenever $f \geq g$ in a
		neighbourhood of $x$ (which holds $\Phi_k$-a.e.\ by the
		MMBB basis construction).
		The equality $\sup_{g \in \Bas(k)}\psi_{g,g}(f) =
		\sup_{g \in \Bas(k)}\varepsilon_g(f) = (f*k)$ then follows
		from Theorem~\ref{thm:conv_basis}.
	\end{proof}
	
	\begin{remark}[Geometric and architectural significance of
		diagonality]
		\label{rem:diagonal_significance}
		Proposition~\ref{prop:bb_nonneg_diagonal} gives a precise
		geometric meaning to the relationship between the MMBB-Increasing and
		MMBB-General frameworks for positive kernels as follows.
		
		\medskip
		\noindent\textbf{Geometric picture.}
		The MMBB basis $\Bas(k) \cong \R^{N-1}$ is a hyperplane in
		$\R^N$ (the orthogonal complement of $k$ at the origin.
		Each basis element $g \in \Bas(k)$ sits on the boundary of
		the half-space $\Ker(\Phi_k) = \{f : \langle k, f\rangle
		\geq 0\}$ (where $\langle k, f\rangle = \sum k_i f(-x_i)$).
		The Banon--Barrera intervals collapse to single points on
		this hyperplane boundary, precisely the diagonal condition
		$g^- = g^+ = g$.
		For a signed kernel, the kernel boundary is no longer a
		hyperplane but the intersection of two half-spaces (one for
		$k^+$, one for $k^-$), and the intervals no longer collapse;
		this is exactly where $g^- \neq g^+$ and the signed
		sup-generating pairs become genuinely non-diagonal.
		
		\medskip
		\noindent\textbf{Architectural consequence in non-negative networks.}
		For a non-negative convolutional layer, the MMBB-Increasing erosion
		basis fully characterises the operator there is no
		additional expressive content in the MMBB-General pairing.
		The anti-dilation $\alpha_{g^+}$ carries no information
		beyond $g^-$ when $g^+ = g^-$.
		This means that neural network layers with
		non-negative weights (e.g., after a softplus or non-negative
		weight constraint) on non-negative signals are fully characterised by a single family
		of structuring functions, not a pair.
		For signed weights, the pairing becomes essential: the
		anti-dilation models the \emph{inhibitory} influence of
		negative kernel components, and $g^-\neq g^+$ measures the
		signed asymmetry of the kernel.
	\end{remark}
	
	\begin{remark}[Comparison: MMBB-Increasing decomposition vs.\ MMBB-General form]
		\label{rem:bb_vs_mmbb}
		Our theory offers two complementary representations of the same signed 
		convolution:
		\begin{itemize}[leftmargin=*]
			\item \textbf{MMBB-Increasing decomposition} ($k = k^+ - k^-$):
			two separate suprema of erosions, combined by subtraction.
			Each supremum operates on a different (positive or negative)
			part of the kernel, and the combination is a difference
			of two lattice expressions.
			This form is more natural for implementation as two
			separate morphological branches and gradient-based
			learning of each basis independently.
			\item \textbf{MMBB-General form} (sup-generating pairs):
			a single supremum, each element being a
			min-combination of an erosion (positive part) and an
			anti-dilation (negative part) acting \emph{jointly} on $f$.
			This form is more compact and avoids explicit sign
			decomposition; each pair $(g^-, g^+)$ models the joint
			influence of positive and negative kernel components
			at a given input configuration.
			It is the natural form for analysing the fixed-point
			structure of signed layers (cf.\ \Cref{sec:fixed_points}),
			since the joint min-combination encodes both the
			excitatory (erosion) and inhibitory (anti-dilation) effects
			of the kernel in a single operator.
		\end{itemize}
		The two forms are theoretically equivalent for absolutely summable kernels;
		the choice between them is guided by the computational implementation.
	\end{remark}

	\subsection{Characteristic matrix and virtual basis for
		quantised signals (binary, ternary, and low-bit
		networks)}
	\label{subsec:char_matrix}
	
	For quantised signals, the infinite basis $\Bas(k)$ restricts
	to a finite \emph{virtual basis}, computable via a structured
	matrix~\cite{khosravi1994}.
	
	\begin{definition}[Characteristic matrix]
		\label{def:char_matrix}
		For a finite kernel $k$ with support
		$\Spt(k) = \{x_1,\ldots,x_N\}$ and values
		$k_i = k(x_i)$, the \emph{characteristic matrix}
		$A_k \in \R^{N\times N}$ is
		\begin{equation}
			(A_k)_{\ell i} = \delta_{\ell i} - k_i,
			\label{eq:char_matrix}
		\end{equation}
		where $\delta_{\ell i}$ is the Kronecker delta.
	\end{definition}
	
	We collect the key algebraic properties of the characteristic
	matrix $A_k$, which underpin the geometric interpretation of the MMBB-Increasing basis 
	and the virtual basis computations.
	
	\begin{proposition}[Properties of $A_k$]
		\label{prop:char_matrix_props}
		Let $A_k \in \R^{N \times N}$ be defined by~\eqref{eq:char_matrix}
		for a kernel $k$ with $\sum k_i = 1$, $k_i \geq 0$.
		\begin{enumerate}[label=(\roman*)]
			\item $\mathrm{null}(A_k) = \mathrm{span}(\mathbf{1})$,
			where $\mathbf{1} = (1,\ldots,1)^T$.
			\item $\Image(A_k) = \{g \in \R^N :
			\langle k, g \rangle = 0\}$,
			the hyperplane orthogonal to $k$.
			\item $\Bas(k) \cong \Image(A_k) \cong \R^{N-1}$:
			the morphological basis is isomorphic to an
			$(N{-}1)$-dimensional subspace.
			\item For quantisation sets related by an affine
			transformation $\mathcal{V}_2 = a\mathcal{V}_1 + b$,
			the virtual bases satisfy
			$\mathcal{V}_2(k) = a \cdot \mathcal{V}_1(k)$
			(scale equivariance, translation-invariant).
		\end{enumerate}
	\end{proposition}
	
	\begin{proof}
		(i)--(iii): The $\ell$-th row of $A_k$ is $e_\ell - k$
		(standard basis vector minus $k$), so
		$A_k \mathbf{1} = \mathbf{1} - (\sum_i k_i)\mathbf{1}
		= \mathbf{1} - \mathbf{1} = \mathbf{0}$.
		For $g$ with $\langle k, g\rangle = 0$: $(A_k g)_\ell =
		g_\ell - \langle k,g\rangle = g_\ell$, so $A_k g = g$
		and $g \in \Image(A_k)$.
		Conversely, for $g = A_k r$: $\langle k, g\rangle
		= \langle k, A_k r\rangle = \langle k,r\rangle -
		\langle k,k\rangle\langle k,r\rangle$... in fact
		$(A_k)^T k = k - (\sum k_i) k = k - k = 0$, so
		$\langle k, A_k r\rangle = (k^T A_k) r = 0$.
		Hence $\Image(A_k) = \ker(k^T) = \{g : \langle k,g\rangle =0\}
		\cong \R^{N-1}$.
		(iv) follows from linearity: $A_k(a\mathcal{M}) =
		a A_k\mathcal{M}$.
	\end{proof}

	\begin{theorem}[Khosravi--Schafer 1994~\cite{khosravi1994}]
		\label{thm:virtual_basis}
		Let $\Lq$ denote the class of signals quantised to a finite
		set $\mathcal{V} = \{v_1,\ldots,v_Q\}$.
		The \emph{virtual basis} $\mathcal{V}(k)$ of $\Phi_k$
		restricted to $\Lq$ satisfies
		$\mathcal{V}(k) = A_k(\mathcal{M})$,
		where $\mathcal{M}$ is any set of representatives of
		$\mathcal{V}^N$ modulo the null-space direction
		$(1,\ldots,1)^T$ of $A_k$.
		For all $f \in \Lq$,
		\begin{equation}
			(f*k)(x)
			= \max_{g \in \mathcal{V}(k)}
			\min_{x_j \in \Spt(k)}\{f(x-x_j) - g(-x_j)\}.
			\label{eq:virtual_basis_rep}
		\end{equation}
	\end{theorem}

	\bigskip
	
	The virtual basis theory of \Cref{thm:virtual_basis} becomes
	especially concrete when the signal alphabet $\mathcal{V}$ is
	small, as in binary and ternary neural networks.
	These architectures, which constrain weights and/or activations
	to $\{0,1\}$, $\{-1,+1\}$, $\{-1,0,+1\}$, or a small set of
	levels, have attracted significant interest for
	hardware-efficient inference.
	The MMBB-Increasing framework provides an exact algebraic characterisation
	of their representational capacity.
	Throughout this subsection we assume $k \geq 0$, $\sum k = 1$,
	$|\Spt(k)| = N$ unless otherwise stated.
	
	\medskip
	
	\paragraph{Binary signals ($Q=2$, $\mathcal{V} = \{0,1\}$).}
	For binary inputs $f \in \{0,1\}^E$ and a positive normalised
	kernel $k$, the convolution $(f*k)(x) = \sum_{i=1}^N k(x_i)
	f(x+x_i)$ is a \emph{weighted average} of binary values,
	taking values in the finite set
	$\bigl\{\sum_{i \in S} k(x_i) : S \subseteq \{1,\ldots,N\}\bigr\}
	\subset [0,1]$.
	This is \emph{not} in general binary-valued; it is $[0,1]$-valued
	with at most $2^N$ distinct output levels.
	
	The virtual basis $\mathcal{V}(k)$ over $\{0,1\}$ is the image
	$A_k(\{0,1\}^N)$ in $\R^{N-1}$, which has cardinality at most
	$2^{N-1}$.
	The MMBB-Increasing representation from \Cref{thm:virtual_basis} gives,
	for all $f \in \{0,1\}^E$:
	\begin{equation}
		(f*k)(x)
		= \max_{g \in \mathcal{V}(k)}
		\min_{j=1}^N \{f(x - x_j) - g(-x_j)\},
		\quad f \in \{0,1\}^E,
		\label{eq:binary_conv}
	\end{equation}
	an exact max-min expression with a \emph{finite} basis of at
	most $2^{N-1}$ elements.
	The sign convention follows Theorem~\ref{thm:virtual_basis}
	exactly: $g$ is evaluated at $-x_j$ (the reflected support
	points).
	
	\medskip
	
	\paragraph{Ternary signals ($Q=3$,
		$\mathcal{V} = \{-1,0,1\}$).}
	For ternary inputs (relevant to signed binary networks and
	features with zero-padding), the virtual basis
	$\mathcal{V}(k) = A_k(\{-1,0,1\}^N)$ has cardinality at most
	$3^{N-1}$.
	
	For a \emph{signed} kernel $k = G^+\bar k^+ - G^-\bar k^-$
	with ternary input $f \in \{-1,0,1\}^E$,
	Theorem~\ref{thm:general_kernel} gives the decomposition:
	\begin{equation}
		(f*k)(x)
		= G^+ \!\max_{h \in \mathcal{V}(\bar k^+)}\!
		\min_{i \in \Spt(k^+)}\!\{f(x{-}x_i) - h(-x_i)\}
		- G^- \!\max_{h \in \mathcal{V}(\bar k^-)}\!
		\min_{j \in \Spt(k^-)}\!\{f(x{-}x_j) - h(-x_j)\},
		\label{eq:ternary_conv}
	\end{equation}
	where $\mathcal{V}(\bar k^\pm) = A_{\bar k^\pm}(\{-1,0,1\}^{N^\pm})$
	are the virtual bases of the positive and negative normalised
	kernels.
	Each sub-basis has cardinality at most $3^{N^\pm - 1}$, and
	the virtual basis elements $h \in \mathcal{V}(\bar k^\pm)$ take
	values in $A_{\bar k^\pm}(\{-1,0,1\}^{N^\pm}) \subset \R$
	(generally not in $\{-1,0,1\}$).
	
	\medskip
	
	\paragraph{Low-bit quantisation ($Q = 2^B$).}
	For $B$-bit quantised activations ($Q = 4$ for 2-bit,
	$Q = 8$ for 3-bit, $Q = 16$ for 4-bit), the virtual basis
	$\mathcal{V}(k) = A_k(\mathcal{V}^N)$ has cardinality at most
	$Q^{N-1} = 2^{B(N-1)}$.
	
	\begin{proposition}[Representational capacity under quantisation]
		\label{prop:quantised_capacity}
		Let $\Phi_k$ be a convolution with $k \geq 0$, $\sum k = 1$,
		$|\Spt(k)| = N$, and inputs quantised to $Q = |\mathcal{V}|$
		levels.
		\begin{enumerate}[label=(\roman*)]
			\item The virtual basis has cardinality
			$|\mathcal{V}(k)| \leq Q^{N-1}$.
			\item Removing one bit of precision ($Q \to Q/2$) reduces
			the basis bound from $Q^{N-1}$ to $(Q/2)^{N-1}$, a
			reduction by a factor of $2^{N-1}$ per bit removed.
			\item For binary inputs ($Q=2$): $|\mathcal{V}(k)| \leq
			2^{N-1}$, and the MMBB representation~\eqref{eq:binary_conv}
			is \emph{exact} with this finite basis.
			\item For a $3\times3$ kernel ($N=9$):
			$Q=2$ gives $|\mathcal{V}(k)| \leq 2^8 = 256$;
			$Q=16$ (4-bit) gives $|\mathcal{V}(k)| \leq 16^8
			= 2^{32} \approx 4.3\times10^9$.
		\end{enumerate}
	\end{proposition}
	
	\begin{proof}
		(i) The characteristic matrix $A_k$ has rank $N-1$ with
		null space $\mathrm{span}(\mathbf{1})$
		(Proposition~\ref{prop:char_matrix_props}(i)(ii)).
		The map $A_k : \mathcal{V}^N \to \R^{N-1}$ identifies patterns
		$v, v' \in \mathcal{V}^N$ satisfying $A_k v = A_k v'$, i.e.,
		$v - v' \in \ker(A_k) = \mathrm{span}(\mathbf{1})$.
		Since $v - v' = c\,\mathbf{1}$ for $c \in \R$ requires
		$v_i - v'_i = c$ for all $i$, the equivalence classes under
		this identification have size at most $Q$ (the number of
		admissible constant shifts $c$ such that both $v$ and
		$v - c\mathbf{1}$ remain in $\mathcal{V}^N$).
		Hence $|\mathcal{V}(k)| \leq Q^N / Q = Q^{N-1}$.
		Parts (ii)--(iv) follow by substituting $Q = 2^B$ and
		the stated values of $N$ and $B$.
	\end{proof}
	
	\begin{remark}[Binary morphological networks]
		\label{rem:binary_morpho}
		Proposition~\ref{prop:quantised_capacity}(iii) has a direct
		architectural interpretation for binary-activation networks:
		a convolutional layer with a \emph{non-negative} kernel
		($k \geq 0$) and binary inputs ($f \in \{0,1\}^E$) can be
		represented exactly as a max-min (MMBB-Increasing) network with at most
		$2^{N-1}$ structuring elements.
		For a $3\times3$ kernel this is $256$ structuring elements.
		
		For binary \emph{weight} networks (weights in $\{-1,+1\}$,
		which gives a signed kernel $k = k^+ - k^-$), the MMBB
		representation requires the signed decomposition of
		Theorem~\ref{thm:general_kernel}: two separate max-min
		branches (one for $k^+$, one for $k^-$) with binary-virtual
		bases of sizes at most $2^{N^+-1}$ and $2^{N^--1}$
		respectively.
		This is strictly different from the non-negative case and
		requires the full signed formulation.
		
		The morphological neural networks studied by Ritter and
		Sussner~\cite{ritter1996} and Davidson and
		Ritter~\cite{davidson1992} correspond precisely to finite-basis
		MMBB-Increasing networks over binary alphabets; the MMBB framework of
		\Cref{thm:virtual_basis} provides the exact representation
		theorem that was implicit in that earlier work.
	\end{remark}
	
	\subsection{Deep learning perspective: learnable morphological
		bases}
	\label{subsec:dl_perspective}
	
	\Cref{thm:conv_basis,thm:virtual_basis} suggest a new class of
	neural network layers.
	Consider replacing the linear convolution in a standard CNN
	layer by a learnable finite subset of the morphological basis.
	Concretely, define the \emph{MMBB layer} with $L$ basis
	elements and $J$ groups:
	\begin{equation}
		\APMO(f)(x)
		= \sum_{j=1}^J w_j
		\Bigl[\max_{1 \leq i \leq L}
		(\varepsilon_{g_{i,j}} f)(x)\Bigr],
		\label{eq:mmbb_layer}
	\end{equation}
	where $\{g_{i,j}\}$ are the learnable structuring functions and
	$w_j$ are learnable weights.
	By \Cref{thm:general_kernel}, this is a truncated MMBB
	expansion of a general (signed) convolution, with
	approximation error controlled by the number $L$ of basis
	elements.
	
	\begin{proposition}[Approximation by MMBB layers]
		\label{prop:mmbb_approx}
		Let $\Phi_k$ be a convolution with absolutely summable kernel
		$k$.
		For any $\epsilon > 0$, there exists a finite collection
		$\{g_{i,j}\}_{i,j}$ of structuring functions and weights
		$\{w_j\}$ such that
		\[
		\sup_{f : \|f\|_\infty \leq 1}
		\Bigl\|
		\Phi_k(f) - \APMO(f)
		\Bigr\|_\infty < \epsilon.
		\]
	\end{proposition}
	
	\begin{proof}
		By \Cref{thm:general_kernel}, $\Phi_k(f) = G^+ \sup_{g^+}
		\varepsilon_{g^+} f - G^- \sup_{g^-} \varepsilon_{g^-} f$.
		The suprema range over (infinite) bases of $\bar{k}^\pm$.
		Since $\Spt(k)$ is finite, signals are bounded, and
		$\varepsilon_{g} f$ is a continuous (in $g$) family of
		operators in the sup-norm, the suprema can be approximated
		uniformly over $\|f\|_\infty \leq 1$ by finite sub-bases of
		cardinality $L = L(\epsilon)$.
		Setting $J = 2$, $w_1 = G^+$, $w_2 = -G^-$ and taking the
		appropriate finite sub-bases yields the result.
	\end{proof}
	
	This positions~\eqref{eq:mmbb_layer} as a well-founded,
	morphologically-grounded alternative to standard convolutional
	layers, and raises the open question of whether gradient-based
	learning efficiently discovers elements of $\Bas(k)$.

	\section{Complete Lattice Structure of Pooling and Morphological
		Pyramids}
	\label{sec:pyramids}
	
	We now address the algebraic structure of the down/up-sampling
	operations that define pooling, unpooling, strided convolution,
	and encoder--decoder skip connections in modern architectures.
	The key insight is that
	these operations are adjoint pairs in a complete inf-semilattice,
	and that the classical morphological pyramid theory of
	Goutsias and Heijmans~\cite{goutsias2000} provides the correct
	algebraic framework.

	\begin{table}[ht]
		\centering
		\caption{Principal results of \S\ref{sec:pyramids} (Pooling, Pyramids, Multi-scale). Results marked $(\star)$ are the paper's principal findings.}
		\label{tab:summary_pyramids}
		\renewcommand{\arraystretch}{1.38}
		\footnotesize
		\begin{tabular}{p{0.82cm}p{2.9cm}p{8.4cm}}
			\toprule
			\textbf{Ref.} & \textbf{Name} & \textbf{Statement and role} \\
			\midrule
			Prop~\ref{prop:adjoint_pyramid} $(\star)$ & Goutsias--Heijmans pyramid & Erosion-decimation $\varepsilon^{\downarrow R}_b$ and dilation-interpolation $\delta^{*\uparrow R}_b$ form an adjunction; composition is an opening. Unifies max-pooling, strided convolution, and the Laplacian pyramid. \\
			Thm~\ref{thm:maxpool} & Max-pooling as pyramid & Flat-structuring Heijmans pyramid ($b\equiv 0$) equals decimated max-pooling exactly; its adjoint is piecewise-constant unpooling by stride $R$. \\
			Cor~\ref{cor:strided} & Strided convolution & Strided convolution with stride $R$ and kernel $k$ is a Goutsias--Heijmans erosion pyramid; transposed convolution is its adjoint synthesis operator. \\
			Prop~\ref{prop:laplacian_skeleton} & $\:$ Laplacian pyramid as skeleton & Laplacian pyramid residues equal the morphological top-hat of the Gaussian pyramid; morphological justification for multi-scale encoder--decoder architectures. \\
			\bottomrule
		\end{tabular}
	\end{table}
	
	\subsection{Abstract pyramid structure and the adjunction
		condition}
	\label{subsec:abstract_pyramid}
	
	Let $\mathscr{L}^d$ denote the space of functions on a
	discrete grid $\Z^d$.
	
	\begin{definition}[Analysis--synthesis pyramid]
		\label{def:pyramid}
		A \emph{pyramid} is a pair of operators
		$(\psi^\downarrow, \psi^\uparrow)$ with
		$\psi^\downarrow:\mathscr{L}^d \to \mathscr{L}^d$ (analysis,
		coarser) and $\psi^\uparrow:\mathscr{L}^d \to \mathscr{L}^d$
		(synthesis, finer), satisfying the
		\emph{pyramid condition}:
		\begin{equation}
			\psi^\downarrow \circ \psi^\uparrow = \mathrm{Id}.
			\label{eq:pyramid_cond}
		\end{equation}
		The \emph{approximation sequence} is defined recursively by
		$f_0 = f$, $f_j = \psi^\downarrow(f_{j-1})$,
		and the \emph{detail signal} at level $j$ is
		$d_j = f_j - \psi^\uparrow(f_{j+1})$.
	\end{definition}
	
	For a downsampling factor $R \in \Z_{>1}$, define the
	\emph{downsampling} and \emph{upsampling} operators:
	\begin{align}
		(\sigma^\downarrow_R f)(n) &= f(R\cdot n),
		\label{eq:downsample}\\
		(\sigma^{\uparrow,c}_R f)(m) &=
		\begin{cases}
			f(n) & \text{if } m = R \cdot n,\\
			c     & \text{otherwise,}
		\end{cases}
		\label{eq:upsample}
	\end{align}
	for a fill value $c \in \Rbar$ (typically $c = -\infty$ or
	$c = 0$).
	
	The analysis and synthesis operators take the general form
	\[
	\psi^\downarrow(f) = \sigma^\downarrow_R(\eta(f)),
	\qquad
	\psi^\uparrow(f) = \xi(\sigma^{\uparrow,c}_R(f)),
	\]
	where $\eta$ and $\xi$ are pre- and post-processing operators.
	
	\begin{proposition}[Goutsias--Heijmans 2000~\cite{goutsias2000}]
		\label{prop:adjoint_pyramid}
		If $(\eta, \xi)$ form an adjunction (i.e., $\eta = \varepsilon$
		is an erosion and $\xi = \delta^*$ is its adjoint dilation),
		then $(\psi^\downarrow, \psi^\uparrow)$ is also an adjunction.
		Furthermore, if $\eta$ and $\xi$ are equivariant to
		translation by $\tau$, then for $R = 2$:
		\[
		\psi^\downarrow \circ \tau_2 = \tau \circ \psi^\downarrow,
		\qquad
		\psi^\uparrow \circ \tau = \tau_2 \circ \psi^\uparrow,
		\]
		where $\tau_2 = \tau \circ \tau$ denotes double translation.
	\end{proposition}
	
	This proposition shows that adjoint pyramid pairs are precisely
	those that arise from morphological adjunctions, and that they
	inherit the equivariance properties essential for image
	processing.
	
	\subsection{Three canonical morphological pyramids}
	\label{subsec:canonical_pyramids}
	
	We present three classical constructions
	from~\cite{goutsias2000}, and show that max-pooling is a
	special case of the third.
	
	\subsubsection{Goutsias--Heijmans erosion pyramid}
	\label{subsubsec:gh_erosion}
	
	Let $b:\Z^d \to \Rbar$ be a centred structuring function with
	support $W$.
	The \emph{Goutsias--Heijmans erosion pyramid} is defined by:
	\begin{align}
		\decEros{R}{b}(f)(x)
		&= \sigma^\downarrow_R(\varepsilon_b(f))(x),
		\quad\text{where }\;
		(\varepsilon_b f)(x)
		= \min_{y \in W}\{f(y) - b(y-x)\};
		\label{eq:gh_eros}\\
		\intDil{R}{b}(f)(x)
		&= \delta^*_b(\sigma^{\uparrow,-\infty}_R(f))(x).
		\label{eq:gh_dil}
	\end{align}
	
	\begin{proposition}[Reconstruction by opening]
		\label{prop:gh_opening}
		The reconstruction operator of the Goutsias--Heijmans erosion
		pyramid is a morphological opening:
		\begin{equation}
			\gamma^{\downarrow R \uparrow}_b(f)
			= \intDil{R}{b}\bigl(\decEros{R}{b}(f)\bigr)
			= \delta^*_b\bigl(\sigma^{\uparrow,-\infty}_R
			(\sigma^\downarrow_R(\varepsilon_b(f)))\bigr).
			\label{eq:gh_opening}
		\end{equation}
		This is anti-extensive ($\gamma^{\downarrow R \uparrow}_b(f)
		\leq f$) and idempotent.
	\end{proposition}
	
	\subsubsection{Heijmans dilation pyramid}
	\label{subsubsec:heijmans_dilation}
	
	The dual construction replaces erosion by dilation and uses
	$c = +\infty$:
	\begin{align}
		\decDil{R}{b}(f)(x)
		&= \sigma^\downarrow_R(\delta_b(f))(x),
		\quad\text{where }\;
		(\delta_b f)(x)
		= \max_{y \in W}\{f(y) + b(x-y)\};
		\label{eq:h_dil}\\
		\intEros{R}{b}(f)(x)
		&= \varepsilon^*_b(\sigma^{\uparrow,+\infty}_R(f))(x).
		\label{eq:h_eros}
	\end{align}
	
	\begin{proposition}[Reconstruction by closing]
		\label{prop:heijmans_closing}
		The reconstruction operator of the Heijmans dilation pyramid
		is a morphological closing:
		\begin{equation}
			\varphi^{\downarrow R \uparrow}_b(f)
			= \varepsilon^*_b\bigl(\sigma^{\uparrow,+\infty}_R
			(\sigma^\downarrow_R(\delta_b(f)))\bigr).
		\end{equation}
		This is extensive ($f \leq \varphi^{\downarrow R \uparrow}_b
		(f)$) and idempotent.
	\end{proposition}
	
	\subsubsection{Max-pooling as a flat-structuring-function
		Heijmans pyramid}
	\label{subsubsec:maxpool}
	
	\begin{theorem}[Max-pooling as morphological pyramid]
		\label{thm:maxpool}
		Let $b \equiv 0$ on a rectangular window
		$W_{R\times R} = \{0,\ldots,R-1\}^2$ (flat structuring
		function). Then the Heijmans dilation pyramid with this $b$ is
		exactly max-pooling:
		\begin{equation}
			\MaxPool{R}(f)(x)
			= \decDil{R}{b}(f)(x)
			= \max_{y \in W_{R\times R}} f(Rx - y).
			\label{eq:maxpool_morpho}
		\end{equation}
		Its adjoint (unpooling) is the interpolation operator
		\begin{equation}
			\intEros{R}{b}(f)(m)
			= \begin{cases}
				f(n) & \text{if } m = R \cdot n,\\
				+\infty & \text{otherwise.}
			\end{cases}
		\end{equation}
	\end{theorem}
	
	\begin{proof}
		With $b \equiv 0$, $(\delta_b f)(x) = \max_{y \in W}f(x-y)$
		is a flat dilation (max-filter).
		Applying $\sigma^\downarrow_R$ gives $\max_{y \in
			W_{R\times R}} f(Rx - y)$, which is the standard
		max-pooling operation with pool size $R$ and stride $R$.
		The adjoint follows from \Cref{prop:adjoint_pyramid} with $c =
		+\infty$.
	\end{proof}
	
	\begin{corollary}[Strided convolution and dilated convolution]
		\label{cor:strided}
		Strided convolution with stride $R$ and kernel $k$ corresponds
		to the Goutsias--Heijmans erosion pyramid with
		pre-processing the convolution 
		$\psi^\downarrow(f) = \sigma^\downarrow_R(f * k)$.
		The corresponding transposed convolution (unpooling) is the
		adjoint synthesis operator of \Cref{eq:gh_dil}.
	\end{corollary}
	
	\subsection{Morphological skeleton and the Laplacian pyramid}
	\label{subsec:skeleton_laplacian}
	
	The Laplacian pyramid, ubiquitous in image processing and
	appearing implicitly in multi-scale CNN architectures, admits
	a morphological interpretation as a \emph{skeleton by opening}.
	
	\begin{definition}[Morphological skeleton]
		Given a family of erosions $\{\varepsilon_i\}$ of increasing
		size (i.e., $\varepsilon_i = \varepsilon_{iB}$ for a base
		structuring element $B$), the \emph{skeleton decomposition}
		of $f$ is:
		\begin{equation}
			S_i(f) = \varepsilon_i(f) - \gamma_B(\varepsilon_i(f))
			= \varepsilon_i(f) - \delta^*_B(\varepsilon_{i+1}(f)),
			\label{eq:skeleton}
		\end{equation}
		with reconstruction $f = \bigcup_i \delta_i(S_i(f))$.
	\end{definition}
	
	\begin{proposition}[Laplacian pyramid as skeleton~\cite{keshet2000}]
		\label{prop:laplacian_skeleton}
		Let $g_\sigma$ be a Gaussian kernel and let
		$\mathrm{Gauss}_i(f) = \varepsilon^{\downarrow 2, \circ i}_{g_{\sigma}}(f)$
		denote the Gaussian pyramid at scale $i$.
		The Laplacian pyramid
		\[
		\mathrm{Lapl}_i(f)
		= \mathrm{Gauss}_i(f)
		- \intDil{2}{g_\sigma}(\mathrm{Gauss}_{i+1}(f))
		\]
		is the residue by opening (top-hat transform) of
		$\mathrm{Gauss}_i(f)$:
		\begin{equation}
			\mathrm{Lapl}_i(f)
			= \Gamma^{\downarrow 2 \uparrow}_{g_\sigma}
			(\mathrm{Gauss}_i(f))
			= S_i(f),
			\label{eq:laplacian_skeleton}
		\end{equation}
		where $\Gamma^{\downarrow 2 \uparrow}_{g_\sigma}(h)
		= h - \gamma^{\downarrow 2 \uparrow}_{g_\sigma}(h)$
		is the top-hat with respect to the pyramid opening.
	\end{proposition}
	
	This result provides a morphological justification for why
	multi-scale architectures based on pyramidal decomposition
	(feature pyramids, U-shaped networks) are effective: each level
	captures structures that survive erosion at that scale but are
	removed at the next coarser level, exactly as in the skeleton
	decomposition, one of the most powerful tools in mathematical 
	morphology~\cite{serra1982,heijmans1994}. 

	\section{ReLU, Max-Pooling, and Morphological Activations}
	\label{sec:activations}
	
	Before developing the full morphological models of deep
	architectures in \Cref{sec:cnn_models}, we establish the
	precise lattice-theoretic status of the two most important
	nonlinear operations in standard CNNs: the ReLU activation and
	max-pooling.
	The analysis in this section draws on and extends the results
	of~\cite{velasco2022morphoact}, where
	these operators were first studied as morphological dilations and
	their compositions were proposed as a general morphological
	activation family.

	\begin{table}[ht]
		\centering
		\caption{Principal results of \S\ref{sec:activations} (ReLU, Max-Pooling, Activations). Results marked $(\star)$ are the paper's principal findings.}
		\label{tab:summary_activations}
		\renewcommand{\arraystretch}{1.38}
		\footnotesize
		\begin{tabular}{p{0.82cm}p{2.9cm}p{8.4cm}}
			\toprule
			\textbf{Ref.} & \textbf{Name} & \textbf{Statement and role} \\
			\midrule
			Prop~\ref{prop:relu_dilation} $(\star)$ & ReLU is a closing & ReLU commutes with pointwise suprema (it is a dilation), is extensive, and is idempotent; hence it is a \emph{closing} in $(\mathscr{L},\leq)$. Its upper adjoint is a \emph{global} (non-local) operator: no local erosion can form an adjunction pair with ReLU. \\
			Prop~\ref{prop:maxpool_dilation} & Max-pooling as dilation & Both non-decimated $\delta_W$ (stride 1) and decimated $\MaxPool{R}$ (stride $R$) are dilations in $(\mathscr{L},\leq)$; each has an explicit adjoint flat erosion forming a morphological opening when composed. \\
			Prop~\ref{prop:apd_factorisation} $(\star)$ & APD factorisation & ReLU and decimated max-pooling compose into a single dilation $\APD_{R;\alpha}(f)(n)=\sup_{y\in W_R}\max(0,f(Rn{-}y){+}\alpha)$; its adjoint is global piecewise-constant upsampling by stride $R$ followed by shift $-\alpha$. \\
			Cor~\ref{cor:conv_relu_mmbb} & Conv\,$+$\,ReLU as MMBB & The spectral operator $\SigSpec_K$ (weighted sum of convolutions) followed by ReLU admits an exact MMBB-Increasing basis representation for non-negative kernels, and an MMBB-General representation for signed kernels. \\
			\bottomrule
		\end{tabular}
	\end{table}
	
	\subsection{ReLU and max-pooling as dilations in the pointwise
		lattice}
	\label{subsec:relu_dilation}
	
	\begin{proposition}[ReLU is a dilation and a closing]
		\label{prop:relu_dilation}
		In the pointwise lattice $(\Fun(E,\Rbar),\leq)$, the ReLU
		activation $\ReLUop(f)(x) = \max(0,f(x))$ is:
		\begin{enumerate}[label=(\roman*)]
			\item A \emph{dilation}: it commutes with pointwise suprema,
			$\ReLUop(\bigvee_i f_i) = \bigvee_i \ReLUop(f_i)$.
			\item \emph{Extensive}: $f \leq \ReLUop(f)$ for all $f$,
			since $f(x) \leq \max(0,f(x))$.
			\item \emph{Idempotent}: $\ReLUop \circ \ReLUop = \ReLUop$,
			since $\max(0,\max(0,f(x))) = \max(0,f(x))$.
		\end{enumerate}
		Being extensive and idempotent, ReLU is a \emph{closing} in the
		pointwise lattice.
		The upper adjoint of $\ReLUop$ in $(\Fun(E,\Rbar),\leq)$ exists
		by the general adjoint functor theorem for complete lattices
		(\Cref{prop:adjunction_props}), but it is not a pointwise
		operator: the adjunction condition $\ReLUop(f) \leq g$ requires
		$\max(0,f(x)) \leq g(x)$ for all $x$, which forces $g(x) \geq 0$
		for \emph{all} $x$ (since $\ReLUop(f) \geq 0$ everywhere).
		For functions $g \geq 0$ pointwise, the largest $f$ satisfying
		$\ReLUop(f) \leq g$ is
		\begin{equation}
			\varepsilon^{\mathrm{ReLU}}(g)(x) = g(x),
			\qquad g \geq 0,
			\label{eq:relu_adjoint}
		\end{equation}
		i.e., the identity on the non-negative cone.
		For $g \not\geq 0$ (i.e., $g(x_0) < 0$ at some point $x_0$),
		no $f$ satisfies $\ReLUop(f) \leq g$ and the adjoint maps $g$
		to $-\infty$ (the bottom of the lattice).
		In summary, the upper adjoint is:
		\begin{equation}
			\varepsilon^{\mathrm{ReLU}}(g) =
			\begin{cases}
				g & \text{if } g(x) \geq 0 \text{ for all } x \in E,\\
				-\infty & \text{otherwise.}
			\end{cases}
			\label{eq:relu_adjoint_full}
		\end{equation}
		This is a global (non-pointwise) operator, which reflects the
		fact that ReLU is a closing in the \emph{pointwise} lattice
		but not a max-plus dilation: its adjoint is determined by a
		global non-negativity constraint on $g$.
	\end{proposition}
	
	\begin{proof}
		(i) $\max(0, \sup_i f_i(x)) = \sup_i \max(0, f_i(x))$ since
		max distributes over suprema.
		(ii) and (iii) are direct from the definition of ReLU.
		A morphological closing satisfies: idempotent, increasing,
		extensive; all three hold here.
		
		For the adjoint: we verify the adjunction condition
		$\ReLUop(f) \leq g \iff f \leq \varepsilon^{\mathrm{ReLU}}(g)$.
		
		\emph{When $g \geq 0$ everywhere:}
		$(\Rightarrow)$: If $\max(0,f(x)) \leq g(x)$ for all $x$,
		then for $f(x) \geq 0$: $f(x) \leq \max(0,f(x)) \leq g(x)$;
		for $f(x) < 0$: $f(x) < 0 \leq g(x)$.
		In both cases $f(x) \leq g(x) = \varepsilon^{\mathrm{ReLU}}(g)(x)$.
		$(\Leftarrow)$: If $f(x) \leq g(x)$ for all $x$, then
		$\max(0,f(x)) \leq \max(0,g(x)) = g(x)$ (since $g \geq 0$).
		
		\emph{When $g(x_0) < 0$ at some point $x_0$:}
		$\ReLUop(f)(x_0) = \max(0,f(x_0)) \geq 0 > g(x_0)$
		for any $f$, so $\ReLUop(f) \leq g$ is impossible.
		The adjoint correctly maps $g$ to $-\infty$, which satisfies
		$f \leq -\infty$ only for $f = -\infty$, confirming no
		solution exists.
		
	\end{proof}
	
	\begin{remark}[Why ReLU's adjoint is global]
		\label{rem:relu_closing_not_opening}
		The non-pointwise character of the upper adjoint
		of $\ReLUop$ \eqref{eq:relu_adjoint_full} has a direct
		architectural consequence.
		For a max-plus dilation $\delta_b$, the adjoint erosion
		$\varepsilon_{b^*}(g)(x) = \inf_y\{g(x+y)-b(-y)\}$ depends
		only on the values of $g$ in a neighbourhood of $x$, making
		it a local operator.
		For ReLU, the adjoint must check the global non-negativity
		of $g$ before returning a finite value: it is a
		\emph{global} operator.
		This means that no local morphological erosion can serve as
		the adjoint of ReLU in the pointwise lattice, and therefore
		no composition of ReLU with a local dilation (such as
		max-pooling) forms a morphological adjunction.
		This is the precise algebraic reason why
		$\MaxPool{R} \circ \ReLUop$ cannot be part of a
		morphological opening, a conclusion that parallels but
		is distinct from the cross-lattice argument for full CNN
		layers in \Cref{thm:cnn_not_opening}.
	\end{remark}
	
	\begin{proposition}[Max-pooling: non-decimated and decimated forms]
		\label{prop:maxpool_dilation}
		We distinguish the two forms of max-pooling that appear in the
		literature and in this paper.
		
		\medskip
		\noindent\textbf{(A) Non-decimated flat dilation} (stride 1):
		\begin{equation}
			\delta_{W}(f)(x) = \sup_{y \in W} f(x-y),
			\qquad x \in E,
			\label{eq:maxpool_nodecim}
		\end{equation}
		is a dilation in $(\Fun(E,\Rbar),\leq)$, extensive and
		(for $W = W_{R\times R}$) not idempotent.
		Its adjoint erosion is the flat min-erosion at the same
		spatial resolution:
		\begin{equation}
			\varepsilon_{W}(g)(x) = \inf_{y \in W} g(x+y),
			\qquad x \in E,
			\label{eq:mineros_nodecim}
		\end{equation}
		and $(\varepsilon_{W}, \delta_{W})$ is an adjunction in
		$(\Fun(E,\Rbar),\leq)$.
		The compositions $\delta_W\circ\varepsilon_W$ (opening) and
		$\varepsilon_W\circ\delta_W$ (closing) are idempotent.
		
		\medskip
		\noindent\textbf{(B) Decimated max-pooling} (stride $R$, the
		standard CNN operation):
		\begin{equation}
			\MaxPool{R}(f)(n) = \sup_{y \in W_R} f(Rn-y),
			\qquad n \in \Z^d,
			\label{eq:maxpool_decim}
		\end{equation}
		maps the input grid $\Z^d$ to the coarser output grid
		$\Z^d$ (with the same index set but physically at stride
		$R$).
		This is the Heijmans dilation pyramid $\delta^{\downarrow R}_{b}$
		with $b\equiv 0$ (\Cref{thm:maxpool,prop:adjoint_pyramid}),
		a dilation in $(\Fun(E,\Rbar),\leq)$.
		Its adjoint erosion is the \emph{erosion-decimation}:
		\begin{equation}
			\varepsilon^{\mathrm{MP}}_R(g)(n)
			= \inf_{y \in W_R} g(Rn+y),
			\qquad n \in \Z^d,
			\label{eq:mineros_decim}
		\end{equation}
		also on the coarse grid (reading $g$ at positions $Rn+y$
		in the input space).
		The adjunction $(\varepsilon^{\mathrm{MP}}_R, \MaxPool{R})$
		holds in $(\Fun(\Z^d,\Rbar),\leq)$, and the compositions give
		the morphological opening and closing on the decimated grid.
	\end{proposition}
	
	\begin{proof}
		\textbf{Part (A).}
		$\delta_W(f)$ is a dilation by the flat max-plus dilation
		formula with $b\equiv 0$.
		The adjunction $\varepsilon_W(f)\leq g \iff f\leq\delta_W(g)$:
		$\delta_W(f)(x)\leq g(x)$ iff $\sup_{y\in W}f(x-y)\leq g(x)$
		iff $f(z)\leq g(z+y)$ for all $y\in W$, iff
		$f(z)\leq\inf_{y\in W}g(z+y) = \varepsilon_W(g)(z)$.
		Non-idempotency of $\delta_W$: $\delta_W^2(f)(x) =
		\sup_{w\in W\oplus W}f(x-w)$ where $W_{R\times R}\oplus
		W_{R\times R} = \{0,\ldots,2(R-1)\}^d$, a window of side
		$2R-1 > R$ for $R>1$.
		
		\textbf{Part (B).}
		$\MaxPool{R}(f)(n) = \sup_{y\in W_R}\{f(Rn-y)+0\} =
		(\delta^{\downarrow R}_0 f)(n)$ is the $b\equiv 0$ case of the
		Heijmans dilation pyramid of \Cref{prop:adjoint_pyramid}.
		The adjoint was established there as the erosion-decimation
		$\varepsilon^{\downarrow R}_0(g)(n) = \inf_{y\in W_R}g(Rn+y)$.
		Adjunction verification: $\MaxPool{R}(f)(n)\leq g(n)$ iff
		$\sup_{y\in W_R}f(Rn-y)\leq g(n)$ iff $f(Rn-y)\leq g(n)$
		for all $y\in W_R$, iff $f(z)\leq g(n)$ for all $z=Rn-y$
		with $y\in W_R$, iff $f(z)\leq\inf_{y\in W_R}g(\lceil z/R
		\rceil) = \varepsilon^{\mathrm{MP}}_R(g)(n')$ where $n' =
		\lceil z/R\rceil$; equivalently, $f \leq
		\varepsilon^{\mathrm{MP}}_R(g)$ on the input grid.
	\end{proof}
	
	\begin{remark}\label{rem:maxpool_consistency}
		Throughout this paper, $\MaxPool{R}$ always denotes the
		\emph{decimated} form~\eqref{eq:maxpool_decim} of Part~(B),
		consistent with the CNN layer definition~\eqref{eq:cnn_standard}
		and the APD formula~\eqref{eq:apd_def}.
		The non-decimated form~\eqref{eq:maxpool_nodecim} of Part~(A)
		is the flat dilation $\delta_W$ used in the opening and closing
		analysis of the pointwise lattice.
		Both forms appear in Part~(B) of the Goutsias--Heijmans
		pyramid theory (\Cref{sec:pyramids}): $\MaxPool{R} =
		\delta^{\downarrow R}_0$ is the decimated dilation pyramid
		(\Cref{thm:maxpool}), and the non-decimated version
		$\delta_W = \delta^{\downarrow 1}_0$ is the flat dilation
		pyramid at stride~1. We provide in Table~\ref{tab:pooling} other 
		cases of pooling applications and their morphological interpretation. 
	\end{remark}
	
	\subsection{Factorisation and the Activation-Pooling Dilation}
	\label{subsec:apd}
	
	\begin{proposition}[Factorisation of ReLU and decimated
		max-pooling into the APD]
		\label{prop:apd_factorisation}
		Both $\ReLUop_\alpha(f)(x) = \max(0,f(x)+\alpha)$ and
		$\MaxPool{R}(f)(n) = \sup_{y\in W_R}f(Rn-y)$ are dilations
		in $(\Fun(E,\Rbar),\leq)$ (\Cref{prop:relu_dilation}(i)
		and \Cref{prop:maxpool_dilation}(B)).
		Their composition is the single
		\emph{Activation-Pooling Dilation} (APD):
		\begin{equation}
			\APD_{R;\alpha}(f)(n)
			= \MaxPool{R}(\ReLUop_\alpha(f))(n)
			= \sup_{y \in W_R}\max\!\bigl(0,\, f(Rn-y) + \alpha\bigr),
			\qquad n \in \Z^d.
			\label{eq:apd_def}
		\end{equation}
		The APD is a dilation in $(\Fun(\Z^d,\Rbar),\leq)$,
		mapping the input grid to the coarser output grid.
		
		The upper adjoint $\varepsilon^{\mathrm{APD}}_{R;\alpha}$ of
		$\APD_{R;\alpha}$ in the pointwise lattice exists but is
		\emph{not} a local operator.
		From the adjunction condition
		$\APD_{R;\alpha}(f)\leq g \iff f\leq
		\varepsilon^{\mathrm{APD}}_{R;\alpha}(g)$: since
		$\APD_{R;\alpha}(f)(n)\geq 0$ for all $f$ and $n$, we
		need $g(n)\geq 0$ for all $n$ (a global non-negativity
		constraint, identical to the ReLU adjoint situation,
		\Cref{prop:relu_dilation}).
		For $g\geq 0$ pointwise, the adjoint is:
		\begin{equation}
			\varepsilon^{\mathrm{APD}}_{R;\alpha}(g)(z)
			= g\!\left(\Bigl\lfloor\tfrac{z}{R}\Bigr\rfloor\right) - \alpha,
			\qquad z \in \Z^d,
			\label{eq:apd_adjoint}
		\end{equation}
		where $\lfloor z/R\rfloor$ is the component-wise floor
		division (the unique output grid point $n$ such that
		$Rn\leq z < R(n+1)$).
		For $g\not\geq 0$, the adjoint maps to $-\infty$.
		In the non-decimated case ($R=1$), the adjoint simplifies
		to $\varepsilon^{\mathrm{APD}}_{1;\alpha}(g)(x)
		= g(x) - \alpha$ when $g(x)\geq 0$, and $-\infty$ otherwise.
	\end{proposition}
	
	\begin{proof}
		\textbf{APD formula.}
		$\APD_{R;\alpha}(f)(n) = \MaxPool{R}(\max(0,f+\alpha))(n)
		= \sup_{y\in W_R}\max(0,f(Rn-y)+\alpha)$,
		by substituting the decimated max-pooling formula
		\eqref{eq:maxpool_decim} with input $\ReLUop_\alpha(f)$.
		The composition of two dilations in the same complete lattice
		is a dilation~\cite{heijmans1994}.
		
		\textbf{Adjoint formula.}
		$\APD_{R;\alpha}(f)(n)\leq g(n)$ for all $n$ means
		$\max(0,f(Rn-y)+\alpha)\leq g(n)$ for all $n$ and
		$y\in W_R$.
		Since $\max(0,\cdot)\geq 0$, this requires $g(n)\geq 0$
		for all $n$ (global non-negativity).
		Given $g\geq 0$: the constraint becomes $f(Rn-y)+\alpha\leq
		g(n)$, i.e., $f(z)\leq g(n)-\alpha$ for each $z=Rn-y$.
		For a given $z\in\Z^d$, the unique $n$ with $y = Rn-z\in W_R
		= \{0,\ldots,R-1\}^d$ is $n = \lfloor z/R\rfloor$
		(component-wise), so the tightest bound on $f(z)$ is
		$g(\lfloor z/R\rfloor)-\alpha$.
		The largest $f$ satisfying all constraints is therefore
		$f(z) = g(\lfloor z/R\rfloor)-\alpha$, giving
		\eqref{eq:apd_adjoint}.
	\end{proof}
	
	\begin{remark}[Non-local adjoint and architectural consequence]
		\label{rem:apd_adjoint}
		The global non-negativity condition on $g$ in
		\Cref{prop:apd_factorisation} is inherited from the
		ReLU component: $\APD_{R;\alpha}(f)\geq 0$ always, so
		any $g$ with a negative value at some output position
		$n_0$ cannot be dominated by any APD output.
		This is the same non-locality identified for ReLU alone
		(\Cref{prop:relu_dilation}) and confirms that the APD,
		like ReLU, does not have a local pointwise adjoint in the
		max-plus sense.
		The adjoint~\eqref{eq:apd_adjoint} is a \emph{piecewise
			constant upsampling by factor $R$} followed by a shift by
		$-\alpha$: each output grid value $g(n)$ is broadcast
		to the $\Z^d$ input positions $\{z : \lfloor z/R\rfloor = n\}$,
		then shifted by $-\alpha$.
		This is the adjoint of decimation in the Goutsias--Heijmans
		sense (\Cref{prop:adjoint_pyramid}), composed with the adjoint
		of the ReLU closing.
	\end{remark}
	
	\begin{remark}\label{rem:morphoact}
		Proposition~\ref{prop:apd_factorisation} corresponds to
		Remark~1 of~\cite{velasco2022morphoact},
		where the APD was first proposed as a single operator
		replacing the separate ReLU and max-pooling steps.
		The key insight is that fusing two dilations in the same lattice
		into one reduces computational overhead and clarifies the
		algebraic role of the combined operation.
	\end{remark}

	\subsection{Generalised morphological activations via MMBB}
	\label{subsec:morpho_activations}
	
	The MMBB representation theorem motivates a generalisation of
	the ReLU family based on morphological operators.
	
	\begin{definition}[Morphological activation family]
		\label{def:morpho_act}
		A \emph{morphological activation} is any operator of the form
		\begin{equation}
			\sigma^{\mathrm{M}}_{\mathcal{B},c}(f)(x)
			= \min\!\left\{
			\sup_{g \in \mathcal{B}} (\varepsilon_g f)(x),\;
			(\alpha_c f)(x)
			\right\},
			\label{eq:morpho_act}
		\end{equation}
		where $\mathcal{B} \subset \Bas(\Psi)$ is a finite sub-basis
		of some TI increasing operator $\Psi$, and $\alpha_c(f)(x)
		= \inf_y\{-f(x+y)+c(y)\}$ is an anti-dilation by a fixed
		cap function $c:\Spt(c)\to\Rbar$
		(Definition~\ref{def:sup_generating}).
		
		The equivalence to a supremum of Banon--Barrera sup-generating
		operators follows from the \emph{complete distributivity} of
		$(\Rbar,\leq)$: in any completely distributive lattice,
		$b \wedge \bigvee_g a_g = \bigvee_g(b\wedge a_g)$ holds
		for any family $(a_g)_{g\in\mathcal{B}}$ and any fixed
		$b\in\Rbar$.
		Applying this pointwise with $a_g = (\varepsilon_g f)(x)$
		and $b = (\alpha_c f)(x)$:
		\begin{equation}
			\sigma^{\mathrm{M}}_{\mathcal{B},c}(f)(x)
			= \sup_{g \in \mathcal{B}}
			\min\!\bigl\{(\varepsilon_g f)(x),\;
			(\alpha_c f)(x)\bigr\}
			= \sup_{g \in \mathcal{B}} \psi_{g,c}(f)(x).
			\label{eq:morpho_act_bb}
		\end{equation}
		This is a Banon--Barrera representation in which all
		pairs share the same upper bound $g^+ = c$
		(Theorem~\ref{thm:bannon_barrera}).
		It gives an exact representation of any TI operator
		$\Psi$ whose sup-generating basis has a fixed upper cap $c$
		independent of $g^-$; for operators with varying upper cap,
		it provides an approximation.
		
		Special cases:
		\begin{itemize}[leftmargin=*]
			\item $\mathcal{B} = \{g_0\}$ (point erosion, $g_0(y) = 0$
			for $y=0$, $-\infty$ otherwise), $c \equiv +\infty$
			(vacuous cap): $\sigma^{\mathrm{M}} = \varepsilon_{g_0}
			= \mathrm{Id}$ (identity).
			\item $\mathcal{B} = \{g_0\}$ (point erosion), $c \equiv 0$:
			$\varepsilon_{g_0}(f)(x) = f(x)$ and
			$\alpha_0(f)(x) = \inf_y\{-f(x+y)\} = -\sup_y f(x+y)$.
			For signals with compact support (or when the anti-dilation
			uses only the value at $y=0$): $\alpha_0(f)(x) = -f(x)$,
			giving $\sigma^{\mathrm{M}} = \min(f,-f) = -|f|$.
			This is the \emph{negative absolute value}, not ReLU.
			\item $\mathcal{B} = \Bas^{\mathrm{trunc}}$ (MMBB
			sub-basis), $c \equiv 0$: a MMBB lower approximation
			capped at zero, generalising ReLU to an arbitrary
			erosion shape.
			This is the primary use case of the definition:
			any activation that is a morphological erosion below
			some level and zero above it.
		\end{itemize}
	\end{definition}

	\begin{proposition}[ReLU as a lattice join; connection to
		MMBB-General representation]
		\label{prop:relu_bb}
		The ReLU activation $\ReLUop(f)(x) = \max(0,f(x))$ satisfies:
		\begin{enumerate}[label=(\roman*)]
			\item It is the lattice join of $f$ with the zero function
			$\mathbf{0} \equiv 0$ in $(\Fun(E,\Rbar),\leq)$:
			\begin{equation}
				\ReLUop(f) = f \vee \mathbf{0},
				\quad
				\ReLUop(f)(x) = \sup\{f(x), 0\}.
				\label{eq:relu_join}
			\end{equation}
			\item As a lattice join, it is a dilation
			(Proposition~\ref{prop:relu_dilation}(i)), but not a
			max-plus dilation $\delta_b$ for any structuring function
			$b$, because max-plus dilations are shift-equivariant
			($\delta_b(f+c) = \delta_b(f)+c$) while ReLU is not
			($\ReLUop(f+c) \neq \ReLUop(f)+c$ for $c < 0$).
			\item In the Banon--Barrera framework, ReLU is a
			sup-generating operator of the form
			\begin{equation}
				\ReLUop(f)(x) = \sup\!\left\{
				\psi_{\mathbf{0},+\infty}(f)(x),\;
				\psi_{-\infty,\mathbf{0}}(f)(x)
				\right\},
				\label{eq:relu_sup_gen}
			\end{equation}
			where $\psi_{\mathbf{0},+\infty}(f)(x) = \varepsilon_0 f(x)
			= f(x)$ (identity erosion, vacuous cap) and
			$\psi_{-\infty,\mathbf{0}}(f)(x) = \min\{-\infty,
			\alpha_{\mathbf{0}}(f)(x)\}$ is the indicator that forces
			the output to be $\geq 0$.
			More concisely: $\ReLUop(f)(x) = \sup\{f(x), 0\}$, a
			two-element supremum of a TI operator (identity) and a
			constant (zero).
		\end{enumerate}
	\end{proposition}
	
	\begin{proof}
		(i) $\max(f(x),0) = \sup\{f(x),0\}$ by definition.
		(ii) Shift-equivariance of $\delta_b$: $\delta_b(f+c)(x)
		= \sup_y\{(f+c)(x-y)+b(y)\} = \sup_y\{f(x-y)+b(y)\}+c
		= \delta_b(f)(x)+c$.
		For ReLU: $\ReLUop(f+c)(x) = \max(0,f(x)+c)$, which
		equals $\ReLUop(f)(x)+c$ only when $f(x)+c \geq 0$;
		for $c < 0$ and $-c > f(x) \geq 0$:
		$\ReLUop(f+c)(x) = 0 \neq \max(0,f(x))+c = f(x)+c > 0$
		(contradiction since $f(x)+c < 0$). So shift-equivariance fails.
		(iii) The Banon--Barrera theorem (Theorem~\ref{thm:bannon_barrera})
		guarantees a representation; the specific two-element form
		$\sup\{f(x),0\}$ identifies the two sup-generating elements:
		the identity (arising from $g^- = \mathbf{0}$, vacuous upper
		cap) and the constant zero function (arising as the
		``floor'' constraint).
	\end{proof}

	\begin{remark}[Positive--negative decomposition and the MMBB-General structure]
		\label{rem:relu_bb_correct}
		Every function $f$ decomposes into its positive and negative
		parts: $f = f^+ - f^-$ where
		\begin{equation}
			f^+(x) = \ReLUop(f)(x) = \max(0,f(x)), \quad
			f^-(x) = \ReLUop(-f)(x) = \max(0,-f(x)),
			\label{eq:pos_neg_decomp}
		\end{equation}
		where $f^+, f^- \geq 0$ and $f = f^+ - f^-$.
		Both $f^+$ and $f^-$ are dilations in the pointwise lattice
		(Proposition~\ref{prop:relu_dilation}).
		The full function $f = f^+ - f^-$ is thus the difference of
		two dilations, precisely the structure identified in
		Theorem~\ref{thm:general_kernel} for general signed
		convolution kernels, and the Banon--Barrera representation
		(Theorem~\ref{thm:bannon_barrera}) for non-increasing
		operators.
		Concretely: a signed CNN kernel acts on $f$ to produce
		a signed output, whose positive and negative parts are then
		selected by ReLU ($f^+$) or by the anti-ReLU ($f^-$);
		retaining both parts, as in the symmetric pooling of
		\Cref{subsec:self_dual_pool}, is the algebraically natural
		design for signed feature maps.
	\end{remark}
	
	\begin{corollary}[Convolution + ReLU as MMBB layer; the $\SigSpec$ operator]
		\label{cor:conv_relu_mmbb}
		The composition of a linear combination of $K$ convolutions
		followed by ReLU,
		\begin{equation}
			\SigSpec_K(f)
			:= \sum_{i=1}^K w_i (f * k_i),
			\qquad
			\ReLUop_\alpha(\SigSpec_K(f))
			= \max\!\left(0,\, \SigSpec_K(f) + \alpha\right),
			\label{eq:sigspec_def}
		\end{equation}
		is the \emph{spectral feature-extraction operator} of a CNN
		layer.
		For a general signed kernel decomposition
		$k_i = G^+_i \bar{k}^+_i - G^-_i \bar{k}^-_i$, the
		post-ReLU output admits the MMBB representation:
		\begin{equation}
			\ReLUop_\alpha(\SigSpec_K(f))(x)
			= \max\!\left(0,\;
			\sum_i G^+_i \sup_{g^+ \in \mathcal{B}^+_i}\!\!(\varepsilon_{g^+} f)(x)
			- \sum_i G^-_i \sup_{g^- \in \mathcal{B}^-_i}\!\!(\varepsilon_{g^-}
			f)(x) + \alpha\right),
			\label{eq:conv_relu_mmbb}
		\end{equation}
		where $\mathcal{B}^\pm_i \subset \Bas(\bar k^\pm_i)$ are finite
		sub-bases.
		For non-negative kernels and inputs ($k_i \geq 0$, $f \geq 0$):
		the ReLU is inactive whenever $\SigSpec_K(f) + \alpha \geq 0$,
		and the output is simply
		$\sup_{g \in \mathcal{B}} (\varepsilon_g f)(x) + \alpha$,
		a pure MMBB truncation with additive bias.
	\end{corollary}
	
	\begin{proof}
		Apply Theorem~\ref{thm:general_kernel} to each $f * k_i$,
		sum with weights $w_i$, add bias $\alpha$, then apply
		ReLU: since $\max(0, \cdot)$ is monotone, it distributes
		through the outer sup but not the difference of the two MMBB
		expansions.
		For $k_i \geq 0$ and $f \geq 0$: $f * k_i \geq 0$, so
		$\ReLUop = \mathrm{Id}$ on any non-negative input.
	\end{proof}
	
	\begin{remark}[Connection to MorphoActivation]
		Corollary~\ref{cor:conv_relu_mmbb} provides the theoretical
		justification for the MorphoActivation proposal
		of~\cite{velasco2022morphoact}: the combined
		convolution--activation block $\ReLUop_\alpha \circ \SigSpec_K$
		is representable (exactly or approximately) as a supremum of
		erosions, making the entire pre-pooling stage of a CNN a MMBB
		network.	
	\end{remark}

	\clearpage
	
	\begin{figure}[H]
		\centering
		\resizebox{\linewidth}{!}{%
			\begin{tikzpicture}[		every node/.style={font=\small},
				Fbox/.style={draw=blue!55!black, fill=blue!8, rounded corners=4pt,
					inner sep=5pt, line width=0.8pt},
				Pbox/.style={draw=orange!70!black, fill=orange!8, rounded corners=4pt,
					inner sep=5pt, line width=0.8pt},
				Mbox/.style={draw=teal!65!black, fill=teal!8, rounded corners=4pt,
					inner sep=5pt, line width=0.8pt},
				Gbox/.style={draw=gray!60, fill=gray!6, rounded corners=4pt,
					inner sep=5pt, line width=0.7pt},
				arr/.style={->, >=stealth, semithick},
				xarr/.style={->, >=stealth, thick, red!70!black},
				sadj/.style={->, >=stealth, semithick, green!50!black},
				skip/.style={->, >=stealth, semithick, dashed, gray!65},
				rskip/.style={->, >=stealth, semithick, dashed, red!60!black},
				op/.style={font=\small\bfseries, above, inner sep=3pt},
				lat/.style={font=\scriptsize\itshape, below, inner sep=3pt,
					text=gray!60!black},
				title/.style={font=\small\bfseries},
				note/.style={font=\scriptsize, text=gray!55!black}]
				\node[title] at (6.0, 2.0)
				{(a) Conv\,$+$\,ReLU as MMBB-Increasing erosion pipeline~\eqref{eq:conv_relu_mmbb}};
				\node[Fbox] (f0) at (0,0)     {$f$};
				\node[Fbox] (S)  at (3.4,0)   {$\SigSpec_K(f)$};
				\node[Fbox] (E1) at (6.8, 0.9){$\varepsilon_{g^+_1}\!f$};
				\node[Fbox] (E2) at (6.8, 0.0){$\varepsilon_{g^+_2}\!f$};
				\node[note] (Ed) at (6.8,-0.7){$\vdots$};
				\node[Pbox] (Sp) at (9.4, 0.0){$\sup_i(\cdot)$};
				\node[Fbox] (Em) at (6.8,-2.0){$\varepsilon_{g^-_1}\!f$};
				\node[Pbox] (Sm) at (9.4,-2.0){$\sup_j(\cdot)$};
				\node[Pbox] (apd) at (12.5, 0.0){$\APD_{R;\alpha}(f)$};
				
				\draw[arr] (f0)--(S)
				node[op,midway]{$\sum_i w_i(\cdot{*}k_i)$}
				node[lat,midway]{Fourier erosion};
				\node[note] at (5.05, 0.3){basis $\mathcal{B}^+$};
				\draw[arr] (S)--(E1); \draw[arr] (S)--(E2);
				\draw[arr] (E1)--(Sp); \draw[arr] (E2)--(Sp);
				\node[note] at (5.05,-1.7){basis $\mathcal{B}^-$};
				\draw[arr] (S) to[out=-70,in=170] (Em);
				\draw[arr] (Em)--(Sm);
				\draw[arr,dashed,gray!55] (Sm) to[out=0,in=-80]
				node[right,note]{$-G^-(\cdots)$} (Sp.south);
				\draw[xarr] (Sp)--(apd)
				node[op,midway]{$\max(0,\cdot{+}\alpha)$; $\MaxPool{R}$}
				node[lat,midway]{\textcolor{red!70!black}{cross-lattice}: Fourier$\to$pointwise};
				
				\node[title] at (5.5,-4.0)
				{(b) APD factorisation: ReLU and $\MaxPool{R}$ fused into one dilation};
				\node[Fbox] (h0) at (0.0,-5.2) {$h = \SigSpec_K(f)$};
				\node[Pbox] (r0) at (4.8,-5.2) {$\ReLUop_\alpha(h)$};
				\node[Pbox] (p0) at (10.0,-5.2){$\APD_{R;\alpha}(h)$};
				
				\draw[arr] (h0)--(r0)
				node[op,midway]{$\max(0,\cdot{+}\alpha)$}
				node[lat,midway]{closing in $(\mathscr{L},\leq)$};
				\draw[arr] (r0)--(p0)
				node[op,midway]{$\MaxPool{R}$}
				node[lat,midway]{flat dilation in $(\mathscr{L},\leq)$};
				\draw[arr, orange!70!black, thick, bend right=28]
				(h0.south) to
				node[below,font=\scriptsize,orange!70!black]
				{$\APD_{R;\alpha}(h)(n)=\sup_{y\in W_R}\max(0,h(Rn{-}y){+}\alpha)$}
				(p0.south);
				
				\node[Fbox,font=\scriptsize] at (1.0,-7.0) {Fourier $(L^n,\!\leq_F)$};
				\node[Pbox,font=\scriptsize] at (4.4,-7.0) {Pointwise $(\mathscr{L},\!\leq)$};
				\draw[xarr,font=\scriptsize] (6.6,-7.0)--(7.6,-7.0)
				node[right,note]{cross-lattice jump};
			\end{tikzpicture}
		}
		\caption{MMBB activation pipeline and APD factorisation.
			Node colours denote the lattice each operator lives in:
			\textcolor{blue!60!black}{blue} = Fourier $(L^n,\leq_F)$,
			\textcolor{orange!70!black}{orange} = pointwise $(\mathscr{L},\leq)$.
			The red arrow marks the cross-lattice jump from Fourier erosion to
			pointwise dilation.
			\emph{(a)} Convolution $+$ ReLU as an MMBB erosion pipeline
			(Corollary~\ref{cor:conv_relu_mmbb}): $\SigSpec_K$ (Fourier) feeds
			basis erosions from $\mathcal{B}^\pm$; their suprema combine, and the
			APD applies ReLU and pooling in the pointwise lattice.
			\emph{(b)} APD factorisation (Proposition~\ref{prop:apd_factorisation}):
			ReLU (a closing) and $\MaxPool{R}$ (a dilation), both in $(\mathscr{L},\leq)$,
			fuse into the single dilation $\APD_{R;\alpha}$ (orange arc).}
		\label{fig:mmbb_pipeline}
	\end{figure}

	\begin{figure}[H]
		\centering
		\resizebox{0.72\linewidth}{!}{%
			\begin{tikzpicture}[		every node/.style={font=\small},
				Fbox/.style={draw=blue!55!black, fill=blue!8, rounded corners=4pt,
					inner sep=5pt, line width=0.8pt},
				Pbox/.style={draw=orange!70!black, fill=orange!8, rounded corners=4pt,
					inner sep=5pt, line width=0.8pt},
				Mbox/.style={draw=teal!65!black, fill=teal!8, rounded corners=4pt,
					inner sep=5pt, line width=0.8pt},
				Gbox/.style={draw=gray!60, fill=gray!6, rounded corners=4pt,
					inner sep=5pt, line width=0.7pt},
				arr/.style={->, >=stealth, semithick},
				xarr/.style={->, >=stealth, thick, red!70!black},
				sadj/.style={->, >=stealth, semithick, green!50!black},
				skip/.style={->, >=stealth, semithick, dashed, gray!65},
				rskip/.style={->, >=stealth, semithick, dashed, red!60!black},
				op/.style={font=\small\bfseries, above, inner sep=3pt},
				lat/.style={font=\scriptsize\itshape, below, inner sep=3pt,
					text=gray!60!black},
				title/.style={font=\small\bfseries},
				note/.style={font=\scriptsize, text=gray!55!black}]
				\draw[->] (-3.8,0)--(3.8,0) node[right,font=\scriptsize]{$f$};
				\draw[->] (0,-2.5)--(0,3.8) node[above,font=\scriptsize]{output};
				\node[font=\scriptsize] at (-0.28,-0.28) {$0$};
				\foreach \x in {-3,-2,-1,1,2,3}{
					\draw(\x,0.07)--(\x,-0.07); \draw(0.07,\x)--(-0.07,\x); }
				
				\draw[very thick, orange!70!black] (-3.5,0)--(0,0)--(3.5,3.5);
				\node[Pbox, font=\scriptsize, anchor=west] at (1.5, 2.7)
				{$f^+(x)=\max(0,f)$: dilation $\delta_0$};
				
				\draw[very thick, teal!65!black] (-3.5,3.5)--(0,0)--(3.5,0);
				\node[Mbox, font=\scriptsize, anchor=east] at (-1.2, 2.7)
				{$f^-(x)=\max(0,{-}f)$: dilation $\delta_0({-}f)$};
				
				\draw[very thick, gray!55, dashed] (-3.5,-3.5)--(3.5,3.5);
				\node[Gbox, font=\scriptsize, anchor=west] at (1.8,-1.5)
				{$f = f^+{-}f^-$};
				
				\draw[very thick, blue!65!black, dotted] (-3.5,0)--(-1,0);
				\draw[very thick, blue!65!black, dotted] (-1,0)--(3.5,2.5);
				\node[Fbox, font=\scriptsize, anchor=west] at (0.2, 0.55)
				{$\sigma^{\mathrm{M}}_{\mathcal{B},c}$: erosion-based, cap $c$};
			\end{tikzpicture}
		}
		\caption{Positive/negative decomposition and MMBB activation shape.
			$f^+(x)=\max(0,f)$ (\textcolor{orange!70!black}{orange}) and
			$f^-(x)=\max(0,-f)$ (\textcolor{teal!65!black}{teal}) are both
			dilations in $(\mathscr{L},\leq)$;
			their difference recovers $f=f^+-f^-$ (dashed diagonal,
			Remark~\ref{rem:relu_bb_correct}).
			The \textcolor{blue!65!black}{blue} dotted line shows a MMBB morphological
			activation $\sigma^{\mathrm{M}}_{\mathcal{B},c}$
			(Definition~\ref{def:morpho_act}): an erosion-based operator with
			reduced slope and upper cap at level $c$, generalising ReLU.}
		\label{fig:mmbb_decomposition}
	\end{figure}
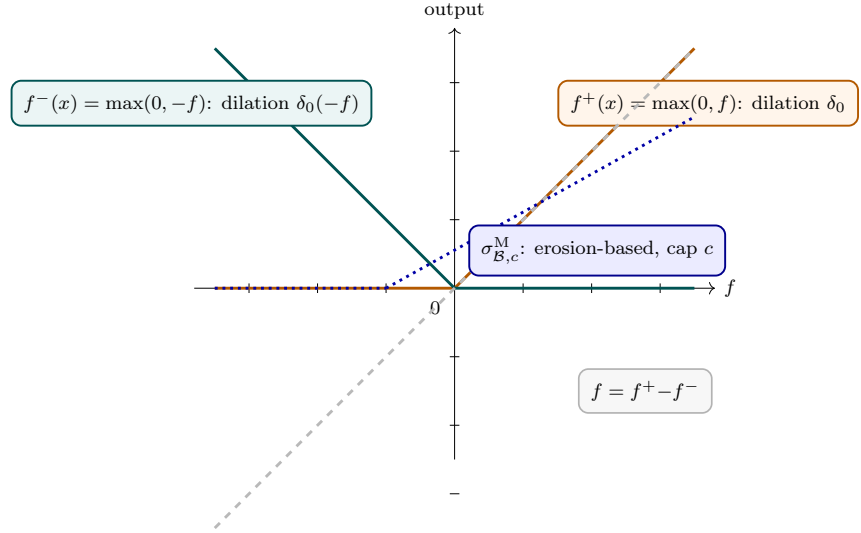

	\begin{table}[H]
		\centering
		\caption{Five pooling operations and their morphological
			classification. Max-pooling is the only flat dilation with a
			local adjoint in $(\mathscr{L},\leq)$.}
		\label{tab:pooling}
		\renewcommand{\arraystretch}{1.55}
		\small
		\begin{tabular}{p{2.4cm}p{4.8cm}p{5.6cm}}
			\toprule
			\textbf{Pooling} &
			\textbf{Formula} &
			\textbf{Morphological type} \\
			\midrule
			$\MaxPool{R}$
			(max-pool) &
			$\sup_{y\in W_R}\{f(Rn{-}y)\}$ &
			Flat dilation $\delta_0$; extensive; adjoint:
			$\varepsilon^{\mathrm{MP}}_R(g)(n) = \inf_{y\in W_R}g(Rn{+}y)$
			(\Cref{prop:maxpool_dilation}(B)) \\
			$\mathrm{AvgPool}_R$
			(average-pool) &
			$\frac{1}{|W_R|}\sum_{y\in W_R} f(Rn{-}y)$ &
			Fourier-lattice erosion (linear conv.\ by $k\equiv 1/|W_R|$);
			not a max-plus dilation; no local pointwise adjoint \\
			$\mathrm{StochPool}_R$
			(stochastic) &
			$f(Rn{-}Y)$,\
			$Y{\sim}\mathrm{Cat}(f/\!\sum f)$ &
			Random selection; not TI, not increasing; outside the
			morphological framework \\
			$\mathrm{MixPool}_R$
			(mixed) &
			$\lambda\MaxPool{R}(f)+(1{-}\lambda)\mathrm{AvgPool}_R(f)$ &
			Convex combination of a dilation and a Fourier erosion;
			not increasing for $\lambda{<}1$ \\
			$\mathrm{SymPool}_R$
			(symmetric) &
			$\MaxPool{R}(f^+){-}\MaxPool{R}(f^-)$ &
			$\approx$ self-dual opening $\OpenMed{W_R}$ in
			$(\mathscr{L},\medOrd)$; idempotent when $f$ is
			sign-consistent on each window
			(\Cref{prop:sym_pool_opening}) \\
			\bottomrule
		\end{tabular}
		\label{tab:pooling}
	\end{table}

	
	\clearpage

	\section{Morphological Models of Deep Convolutional Architectures}
	\label{sec:cnn_models}

	We now synthesise the preceding theory into algebraic models of
	standard deep architectures, drawing on the adjunction
	framework of \Cref{sec:lattice}, the MMBB basis of
	\Cref{sec:conv_basis}, the pyramid theory of
	\Cref{sec:pyramids}, and the morphological activation analysis
	of \Cref{sec:activations}.
	Two new architectures are proposed: the \emph{morphological
		APMO layer} (\Cref{subsec:apmo}) as a fully max-plus
	replacement for the standard CNN layer, and the
	\emph{UResNet} (\Cref{def:unet}), in which skip connections
	carry top-hat residues (difference between the input and its opening) 
	rather than concatenated features.
	The idempotency and fixed-point analysis of these models is
	deferred to \Cref{sec:fixed_points}.

	\begin{table}[ht]
		\centering
		\caption{Principal results of \S\ref{sec:cnn_models} (Morphological Models of CNN, ResNet, UNet). Results marked $(\star)$ are the paper's principal findings.}
		\label{tab:summary_cnn_models}
		\renewcommand{\arraystretch}{1.38}
		\footnotesize
		\begin{tabular}{p{0.82cm}p{2.9cm}p{8.4cm}}
			\toprule
			\textbf{Ref.} & \textbf{Name} & \textbf{Statement and role} \\
			\midrule
			Prop~\ref{prop:resnet} $(\star)$ & $\:$ ResNet as opening & When the residual function $\mathcal{F}\approx\gamma^{\mathrm{M}}_b-\mathrm{id}$, the ResNet block $\mathcal{F}(f)+f$ computes the Type-I opening $\gamma^{\mathrm{M}}_b(f)=\delta_{b^*}(\varepsilon_b(f))$; the skip connection carries the top-hat residue $\Gamma(f)=f-\gamma^{\mathrm{M}}_b(f)$. \\
			Prop~\ref{prop:skip_adjoint} & $\:$ Skip connections as top-hat & UResNet skip connections carry $\Gamma_b(f)=f-\gamma^{\mathrm{M}}_b(f)$ at each scale; the decoder reconstructs $f$ from the pyramid opening plus the residue, achieving exact scale-by-scale reconstruction. \\
			Prop~\ref{prop:apmo_mmbb} & $\:$ APMO approximation & The APMO (infimum of $J$ dilations) provides an upper approximation of any TI increasing layer $\Psi$, dual to the MMBB lower approximation; the two together bracket $\Psi f$ from above and below. \\
			\bottomrule
		\end{tabular}
	\end{table}
	
	\subsection{The morphological CNN layer}
	\label{subsec:cnn_layer}
	
	\paragraph{The spectral feature-extraction operator $\SigSpec_n$.}
	
	In \Cref{cor:conv_relu_mmbb} we introduced
	$\SigSpec_K(f) = \sum_{i=1}^K w_i (f * k_i)$ as the linear
	combination of $K$ convolutions at a single layer.
	In a multilayer network we index by layer $n$:
	\begin{equation}
		\SigSpec_n(f)
		:= \sum_{i=1}^{K_n} w_{n,i}\, \ConvEros{k_{n,i}}(f)
		= \sum_{i=1}^{K_n} w_{n,i}\, (f * k_{n,i}),
		\label{eq:sigspec_layer}
	\end{equation}
	where $\{k_{n,i}\}_{i=1}^{K_n}$ are the learned kernels and
	$\{w_{n,i}\}$ the learned weights at layer $n$.
	This operator lives in the Fourier inf-semilattice
	$(L^n, \leq_F)$~\cite{keshet2000,angulo2025dgmm}. 
	Alternatively, it can be represented as 
	a MMBB-Increasing operator when all $k_{n,i} \geq 0$; for general signed
	kernels it is TI but not increasing, and admits the
	MMBB-General representation \Cref{thm:bannon_barrera}.
	
	\bigskip
	\paragraph{The one-layer CNN as APD $\circ$ $\SigSpec$.}
	A standard one-layer CNN with $K_n$ filters computes
	\begin{equation}
		f \;\longmapsto\;
		\MaxPool{R}\!\Bigl(\,\ReLUop_\alpha\!\Bigl(
		\textstyle\sum_{i=1}^{K_n} w_{n,i}\, (f * k_{n,i})\Bigr)\Bigr)
		= \APD_{R;\alpha}\!\bigl(\SigSpec_n(f)\bigr),
		\label{eq:cnn_standard}
	\end{equation}
	where $\APD_{R;\alpha}$ is the Activation-Pooling Dilation of
	\Cref{prop:apd_factorisation}, which fuses ReLU and max-pooling
	into a single dilation in the pointwise lattice $(\mathscr{L},\leq)$.
	
	This factorisation has two immediate consequences established
	in \Cref{sec:activations}:
	\begin{itemize}[leftmargin=*]
		\item \emph{Lattice structure}: $\SigSpec_n$ is a
		cross-lattice operator (erosion in the Fourier lattice)
		while $\APD_{R;\alpha}$ is a dilation in the pointwise
		lattice; the composition is cross-lattice and
		\emph{not} a morphological opening in either
		(\Cref{thm:cnn_not_opening}).
		\item \emph{Signed weights}: when $w_{n,i}$ include negatives,
		$\SigSpec_n$ decomposes as a difference of two MMBB-Increasing
		expansions (\Cref{thm:general_kernel,cor:conv_relu_mmbb}).
		The asymmetry between positive and negative activations
		motivates the self-dual framework of \Cref{sec:self_dual}.
	\end{itemize}
	
	\begin{observation}[Max-plus alternative to $\SigSpec_n$]
		\label{obs:max_plus}
		The linear combination $\SigSpec_n(f) =$ $\sum_i w_{n,i}
		\ConvEros{k_{n,i}}(f)$ can be replaced by a max-plus
		supremum:
		\begin{equation}
			\Sigma^{\mathrm{Spec,M}}_n(f)(n)
			= \max_{1 \leq i \leq K_n}
			\bigl\{\ConvEros{k_{n,i}}(f)(n) + b_{n,i}\bigr\},
			\label{eq:sigspec_maxplus}
		\end{equation}
		which is the pointwise maximum of the $K_n$ convolution
		outputs with additive biases.
		In the Fourier lattice this is an upper envelope of erosions,
		hence itself a dilation in that lattice.
		Replacing $\SigSpec_n$ by $\Sigma^{\mathrm{Spec,M}}_n$ removes
		the linear summation and aligns the feature-extraction stage
		more closely with the max-plus structure of the pooling stage.
	\end{observation}

	\begin{observation}[Non-commutativity of $\ConvEros{k}$ and
		$\MaxPool{R}$]
		\label{obs:commutativity}
		Let $\text{Sub}_R$ denote subsampling by stride $R$ (without
		the max): $\text{Sub}_R(f)(n) = f(Rn)$.
		
		\emph{Linear subsampling and convolution.}
		For linear subsampling,
		$\text{Sub}_R(f*k)(n) =$ $\sum_m f(m)k(Rn-m)$
		while $\text{Sub}_R(f)*k$ would convolve the subsampled signal
		with $k$ at the coarse scale.
		These are related by the \emph{noble identity} (polyphase
		representation):
		$\text{Sub}_R(f*k) = \text{Sub}_R(f)*\tilde k$,
		where $\tilde k(n) = \sum_{\ell\in\Z^d} k(n + \ell R)$
		is the aliased (periodically folded) version of $k$.
		Convolution commutes with subsampling, in the sense that
		$\tilde k = k$ on the coarse grid, if and only if $k$ is
		supported entirely at multiples of $R$, i.e., $k$ is
		itself a subsampled kernel.
		For a general low-pass filter $k$, the aliasing error
		$\|\tilde k - k\|$ is controlled by the energy of $k$ outside
		the Nyquist band $[-\pi/R, \pi/R]$; the Nyquist condition on
		$k$ makes this error zero for \emph{average-pooling}
		(linear subsampling), not for max-pooling.
		
		\emph{Max-pooling and convolution.}
		Max-pooling is nonlinear and the noble identity does not apply.
		$\MaxPool{R}(f*k)(n) = \sup_{y\in W_R}(f*k)(Rn-y)$
		does not simplify to any convolution of $\MaxPool{R}(f)$ with
		$k$ for general $f$ and $k$, because $\sup$ does not
		distribute over linear sums.
		The two operations are approximately interchangeable when $f$
		varies slowly over each window $W_R$ (Lipschitz condition on
		$f$) and $k$ is nearly flat on $W_R$; the error is bounded by
		$\mathrm{Lip}(f) \cdot R \cdot \|k\|_1$ where $\mathrm{Lip}(f)$
		is the Lipschitz constant of $f$.
		This is distinct from the Nyquist bandlimit condition, which
		applies to linear subsampling only.
	\end{observation}

	\subsection{The general morphological nonlinear layer in a CNN: APMO}
	\label{subsec:apmo}
	
	The APD factorisation (\Cref{prop:apd_factorisation}) fuses
	ReLU and max-pooling into a single dilation.
	The APMO generalises this by replacing the single dilation
	with a \emph{finite infimum of dilations}, which is the
	dual counterpart of the MMBB supremum of erosions. That would 
	provide a more general nonlinear layer.
	
	\begin{definition}[Activation--Pooling Morphological Operator]
		\label{def:apmo}
		The \emph{Activation--Pooling Morphological Operator} (APMO)
		with $J$ structuring functions $\{b_j\}_{j=1}^J$ and biases
		$\{\alpha_j\}_{j=1}^J$ is:
		\begin{equation}
			\APMO(f)(n)
			= \inf_{1 \leq j \leq J}
			\Bigl\{
			\delta_{b_j}(f)(Rn) + \alpha_j
			\Bigr\}
			= \inf_{1 \leq j \leq J}
			\Bigl\{
			\max_{y \in W_R}\{f(Rn-y) + b_j(y)\} + \alpha_j
			\Bigr\},
			\label{eq:apmo}
		\end{equation}
		where $\delta_{b_j}(f)(Rn) = \max_{y\in W_R}\{f(Rn-y)+b_j(y)\}$
		is the max-plus dilation of $f$ by $b_j$ evaluated on the
		coarse grid (stride $R$), and $n \in \Z^d$.
		The infimum over $J$ dilations makes the APMO an
		\emph{erosion in the dual max-plus lattice}: it commutes
		with pointwise infima, $\APMO(\inf_i f_i) = \inf_i\APMO(f_i)$.
	\end{definition}
	
	\begin{proposition}[APMO as truncated dual MMBB upper approximation 
		of any nonlinear TI increasing layer $\Psi$]
		\label{prop:apmo_mmbb}
		The APMO is a \emph{dual} (upper) MMBB-Increasing approximation.
		Recall that the MMBB theorem represents any TI increasing
		operator $\Psi$ both as a supremum of erosions
		(lower representation, Theorem~\ref{thm:mmbb}) and dually as
		an infimum of dilations (upper representation):
		\begin{equation}
			\Psi(f)(x)
			= \sup_{g \in \Bas(\Psi)} \varepsilon_g(f)(x)
			= \inf_{h \in \Bas(\bar\Psi)} \delta_{h^*}(f)(x),
			\label{eq:mmbb_dual}
		\end{equation}
		where $\bar\Psi(f) = -\Psi(-f)$ is the dual operator and
		$h^*(y) = h(-y)$ the transposed structuring function.
		The APMO with basis $\{b_j\}_{j=1}^J$ realises a truncated
		version of the upper representation:
		\begin{equation}
			\APMO(f)(n)
			= \inf_{j=1}^J \delta_{b_j}(f)(Rn) + \alpha_j
			\geq \Psi(f)(Rn),
			\label{eq:apmo_upper}
		\end{equation}
		an \emph{upper} approximation of $\Psi$ evaluated at the
		coarse grid (stride $R$).
		Together with the MMBB lower approximation
		$\Psi^{l}_{K}(f) = \sup_{i=1}^K\varepsilon_{g_i}(f)$, they
		bracket $\Psi$ from below and above:
		\begin{equation}
			\sup_{i=1}^K\varepsilon_{g_i}(f)(Rn)
			\;\leq\; \Psi(f)(Rn)
			\;\leq\; \inf_{j=1}^J\delta_{b_j}(f)(Rn) + \alpha_j.
			\label{eq:mmbb_bracket}
		\end{equation}
		As $K,J\to|\Bas(\Psi)|$, both bounds converge to $\Psi(f)(Rn)$.
	\end{proposition}
	
	\begin{proof}
		The dual MMBB-Increasing representation~\eqref{eq:mmbb_dual} is the second
		form in Theorem~\ref{thm:mmbb}: applying the MMBB-Increasing theorem to the
		dual operator $\bar\Psi$ gives $\bar\Psi(f)(x) =$
		$\sup_{h\in\Bas(\bar\Psi)}\varepsilon_h(f)(x)$, and
		substituting $\bar\Psi(f) = -\Psi(-f)$ and $f\to -f$ gives
		$\Psi(f)(x) = \inf_{h\in\Bas(\bar\Psi)}\delta_{h^*}(f)(x)$.
		Taking any $J$ elements $\{b_j\}_{j=1}^J
		\subset \Bas(\bar\Psi)$ (with $b_j^* = h^*$ for $h = b_j$)
		and adding biases $\alpha_j \geq 0$ gives a value $\geq$
		the infimum over the full basis, hence $\geq \Psi(f)(Rn)$.
		The lower approximation follows from Theorem~\ref{thm:mmbb}
		directly; the two-sided bracket~\eqref{eq:mmbb_bracket} is
		the standard MMBB-Increasing approximation theory by Maragos~\cite{maragos1989}.
	\end{proof}
	
	\begin{remark}[Duality in architecture]
		\label{rem:apmo_vs_mmbb}
		The structural duality between the erosion-based MMBB layer and the APMO
		has a direct architectural interpretation:
		\begin{itemize}[leftmargin=*]
			\item \textbf{$\sup$ of erosions} (lower bound):
			$\Psi^{l}_{K}(f) = \sup_{i=1}^K\varepsilon_{g_i}(f)$.
			Each erosion $\varepsilon_{g_i}$ tests whether $f \geq g_i$
			(morphological matching); the supremum selects the best
			match.
			This is the \emph{analysis} side: detecting features
			present in $f$.
			\item \textbf{$\inf$ of dilations} (upper bound):
			$\Psi^{u}_{J}(f) = \inf_{j=1}^J\delta_{b_j}(f)$.
			Each dilation $\delta_{b_j}$ expands $f$ by structuring
			function $b_j$; the infimum retains only the structures
			common to all $J$ expansions.
			This is the \emph{synthesis} side: pooling and selecting
			among expanded feature maps.
		\end{itemize}
		A morphological CNN layer that brackets a general operator 
		$\Psi$ from both sides would compose an lower layer with an 
		upper layer, with the true operator $\Psi$ sandwiched between
		them.
		The standard CNN layer (conv $+$ ReLU $+$ max-pooling) uses convolution 
		seen as the lower side (erosions) with the upper (APMO/dilation) side; 
		a purely morphological 	layer using only the lower (MMBB/erosion) side is the Type-I
		opening of \Cref{thm:two_openings}(i).
	\end{remark}
	
	The complete multilayer CNN, alternating between spectral
	feature extraction and morphological activation, is:
	\begin{equation}
		L_n
		= \APMO_n\!\bigl(\SigSpec_n(L_{n-1})\bigr),
		\qquad
		\SigSpec_n(f) = \textstyle\sum_{i=1}^{K_n} w_{n,i}
		\ConvEros{k_{n,i}}(f),
		\label{eq:multilayer_cnn}
	\end{equation}
	where $\SigSpec_n$ was defined in~\eqref{eq:sigspec_layer}
	and its MMBB representation is given in~\eqref{eq:conv_relu_mmbb}.
	When $\APMO_n$ reduces to $\APD_{R;\alpha_n}$
	(i.e., $J=1$, flat $b_1 \equiv 0$, $\alpha_1 = \alpha$),
	this recovers the standard CNN layer~\eqref{eq:cnn_standard}.
	The alternation between $\SigSpec_n$ (Fourier-lattice erosion,
	analysis side) and $\APMO_n$ (dual MMBB upper approximation,
	synthesis/pooling side) reflects the fundamental duality of
	Theorem~\ref{thm:mmbb}: $\SigSpec_n$ approximates the operator
	from below (lower MMBB) while $\APMO_n$ approximates from
	above (upper dual MMBB), and the true operator $\Psi$ is
	sandwiched between them (\Cref{prop:apmo_mmbb},
	eq.~\eqref{eq:mmbb_bracket}).

	\begin{figure}[H]
		\centering
		\resizebox{\linewidth}{!}{%
			\begin{tikzpicture}[		every node/.style={font=\small},
				Fbox/.style={draw=blue!55!black, fill=blue!8, rounded corners=4pt,
					inner sep=5pt, line width=0.8pt},
				Pbox/.style={draw=orange!70!black, fill=orange!8, rounded corners=4pt,
					inner sep=5pt, line width=0.8pt},
				Mbox/.style={draw=teal!65!black, fill=teal!8, rounded corners=4pt,
					inner sep=5pt, line width=0.8pt},
				Gbox/.style={draw=gray!60, fill=gray!6, rounded corners=4pt,
					inner sep=5pt, line width=0.7pt},
				arr/.style={->, >=stealth, semithick},
				xarr/.style={->, >=stealth, thick, red!70!black},
				sadj/.style={->, >=stealth, semithick, green!50!black},
				skip/.style={->, >=stealth, semithick, dashed, gray!65},
				rskip/.style={->, >=stealth, semithick, dashed, red!60!black},
				op/.style={font=\small\bfseries, above, inner sep=3pt},
				lat/.style={font=\scriptsize\itshape, below, inner sep=3pt,
					text=gray!60!black},
				title/.style={font=\small\bfseries},
				note/.style={font=\scriptsize, text=gray!55!black}]
				\node[title] at (5.5, 1.0)
				{(a) Standard CNN layer — cross-lattice, \textbf{not} idempotent};
				\node[Gbox] (A1) at (0,0)    {$f$};
				\node[Fbox] (B1) at (3.4,0)  {$f{\ast}k$};
				\node[Pbox] (C1) at (6.8,0)  {$\sigma(f{\ast}k)$};
				\node[Pbox] (D1) at (10.8,0) {$L_n$};
				
				\draw[arr] (A1)--(B1)
				node[op,midway]{$\varepsilon^{\mathrm{Conv}}_k$}
				node[lat,midway]{Fourier erosion};
				\draw[xarr] (B1)--(C1)
				node[op,midway]{$\delta^{\mathrm{ReLU}}$}
				node[lat,midway]{\textcolor{red!70!black}{cross-lattice}};
				\draw[arr] (C1)--(D1)
				node[op,midway]{$\MaxPool{R}$}
				node[lat,midway]{pointwise dilation};
				
				\node[title] at (5.5,-1.2) {(b) APD factorisation~\eqref{eq:apd_def}};
				\node[Gbox] (A2) at (0,-2.4)   {$f$};
				\node[Fbox] (B2) at (4.0,-2.4) {$\Sigma^{\mathrm{Spec}}_n(f)$};
				\node[Pbox] (D2) at (10.8,-2.4){$L_n$};
				
				\draw[arr] (A2)--(B2)
				node[op,midway]{$\sum_i w_i\,\varepsilon^{\mathrm{Conv}}_{k_i}$}
				node[lat,midway]{Fourier erosion};
				\draw[xarr] (B2)--(D2)
				node[op,midway]{$\APD_{R;\alpha}$}
				node[lat,midway]{\textcolor{red!70!black}{cross-lattice}: $\sup_{W}\max(0,\cdot{+}\alpha)$};
				
				\node[title] at (5.5,-3.6)
				{(c) Type-I opening $\gamma^{\mathrm{M}}_b=\delta_{b^*}\!\circ\!\varepsilon_b$
					— \textbf{idempotent} (same lattice throughout)};
				\node[Gbox] (A3) at (0,-4.8)    {$f$};
				\node[Pbox] (B3) at (4.0,-4.8)  {$\varepsilon_b(f)$};
				\node[Pbox] (D3) at (10.8,-4.8) {$\gamma^{\mathrm{M}}_b(f)$};
				
				\draw[arr] (A3)--(B3)
				node[op,midway]{$\varepsilon_b$}
				node[lat,midway]{pointwise erosion; $\varepsilon_b(f)\leq f$};
				\draw[sadj] (B3)--(D3)
				node[op,midway]{$\delta_{b^*}$}
				node[lat,midway]{\textcolor{green!50!black}{adjoint} dilation; same lattice};
				\draw[sadj, bend left=50] (D3.north east)
				to node[above, note, text=green!50!black]
				{$\gamma^{\mathrm{M}}_b\!\circ\!\gamma^{\mathrm{M}}_b=\gamma^{\mathrm{M}}_b$}
				(D3.north west);
				
				\draw[<->, green!45!black, semithick]
				($(B3)+(0,-0.9)$) to
				node[below, note, text=green!50!black]
				{adjunction $(\varepsilon_b,\delta_{b^*})$ in $(\mathscr{L},\leq)$}
				($(D3)+(0,-0.9)$);
				
				\node[Fbox,font=\scriptsize] at (0.8,-6.4) {Fourier $(L^n,\!\leq_F)$};
				\node[Pbox,font=\scriptsize] at (4.2,-6.4) {Pointwise $(\mathscr{L},\!\leq)$};
				\draw[xarr] (6.6,-6.4)--(7.5,-6.4) node[right,note]{cross-lattice};
				\draw[sadj] (9.1,-6.4)--(10.0,-6.4) node[right,note]{same-lattice adjoint};
			\end{tikzpicture}
		}
		\caption{Three algebraic views of a CNN layer.
			Node colours denote the lattice:
			\textcolor{blue!60!black}{blue} = Fourier,
			\textcolor{orange!70!black}{orange} = pointwise.
			Red arrows mark cross-lattice jumps; green arrows mark same-lattice adjoint pairs.
			\emph{(a)} Standard CNN: the convolution $\varepsilon^{\mathrm{Conv}}_k$ lives in the
			Fourier lattice while $\delta^{\mathrm{ReLU}}$ and $\MaxPool{R}$ live in the
			pointwise lattice; the red arrow marks the cross-lattice jump that breaks idempotency.
			\emph{(b)} APD factorisation: ReLU and max-pooling fuse into $\APD_{R;\alpha}$
			(still cross-lattice).
			\emph{(c)} Type-I opening: $\varepsilon_b$ and its adjoint dilation
			$\delta_{b^*}$ both in $(\mathscr{L},\leq)$; the composition $\gamma^{\mathrm{M}}_b$
			is idempotent (green loop, Theorem~\ref{thm:two_openings}(i)).}
		\label{fig:cnn_diagram}
	\end{figure}
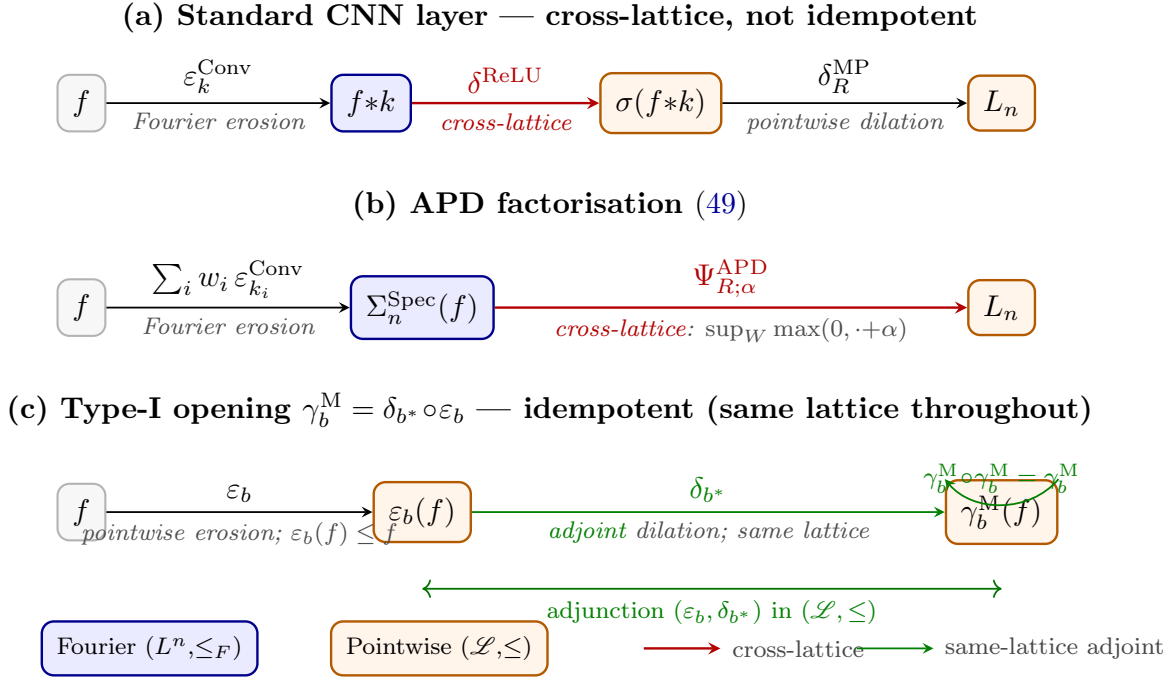
	
	\subsection{Morphological model of ResNet}
	\label{subsec:resnet}
	
	A residual block in ResNet~\cite{he2016resnet} computes
	$L = \mathcal{F}(f) + f$, where $\mathcal{F}$ is a stack of
	convolutional layers.
	In our notation:
	
	\begin{proposition}[ResNet block as morphological opening]
		\label{prop:resnet}
		A residual block $L = \mathcal{F}(f) + f$ where $\mathcal{F}$
		is a stack of CNN layers computes, under the approximation
		$\mathcal{F}(f) \approx \gamma^{\mathrm{M}}_b(f) -
		f$ (i.e., the residual function learns the deviation of $f$
		from its morphological opening):
		\begin{equation}
			L \approx \gamma^{\mathrm{M}}_b(f)
			= \delta_{b^*}(\varepsilon_b(f)),
			\label{eq:resnet_morpho}
		\end{equation}
		a Type-I morphological opening (\Cref{thm:two_openings}(i)).
		The skip connection $+f$ adds back the \emph{top-hat residue}
		$\Gamma_b(f) = f - \gamma^{\mathrm{M}}_b(f) \geq 0$
		(the detail discarded by the opening), so that
		$\mathcal{F}(f) + f = (\gamma^{\mathrm{M}}_b(f) - f) + f
		= \gamma^{\mathrm{M}}_b(f)$.
	\end{proposition}
	
	\begin{remark}
		The top-hat transform $\Gamma(f) = f - \gamma(f)$ is a
		classical tool in morphological image analysis for extracting
		structures smaller than the structuring element used in the
		opening $\gamma$ (i.e., the structures which are not invariant to the structuring element)~\cite{serra1982}.
		\Cref{prop:resnet} reveals that the residual learning mechanism
		of ResNet has a natural morphological interpretation: each
		residual block computes the top-hat of the feature map with
		respect to the learned morphological opening $\PsiMorpho_1
		\circ \SigSpec_1$.
		This explains geometrically why residual connections preserve
		fine-grained information: top-hat transforms retain exactly
		the structures that opening (smoothing) removes.
	\end{remark}

	\subsection{Morphological model of UNet}
	\label{subsec:unet}
	
	The UNet architecture~\cite{ronneberger2015unet} is an
	encoder--decoder with skip connections between symmetric levels.
	At each encoder level the standard UNet applies a stack of
	convolutions followed by ReLU activations and max-pooling;
	the decoder applies transposed convolutions with upsampling.
	In morphological terms:
	\begin{itemize}[leftmargin=*]
		\item \textbf{Standard encoder step}:
		$L^{\downarrow}_\ell = \APD_{R;\alpha}\!\bigl(\SigSpec_\ell
		(L^{\downarrow}_{\ell-1})\bigr)$, the CNN layer
		of~\eqref{eq:cnn_standard} (convolution $+$ ReLU $+$
		max-pooling fused via \Cref{prop:apd_factorisation}).
		\item \textbf{Morphological encoder step}:
		$L^{\downarrow}_\ell = \varepsilon^{\downarrow}_{R,b_\ell}\!
		\bigl(\SigSpec_\ell(L^{\downarrow}_{\ell-1})\bigr)$,
		where $\varepsilon^{\downarrow}_{R,b_\ell}$ is the
		erosion-decimation (\Cref{prop:adjoint_pyramid}).
		The erosion is anti-extensive ($\varepsilon_b(f)\leq f$),
		which plays the same role as ReLU (thresholding from below)
		while also performing the pooling step.
		The ReLU is therefore \emph{subsumed} by the morphological
		erosion: the negativity-clipping of ReLU is replaced by the
		more general morphological thresholding determined by the
		structuring function $b_\ell$.
	\end{itemize}
	Let $\SigSpec_n = \sum_{i=1}^{K_n} w_{n,i}\ConvEros{k_{n,i}}$
	denote the spectral convolution operator (without activation),
	and $\SigSpecDual_n$ its decoder counterpart.

	\begin{definition}[Morphological UNet]
		\label{def:unet}
		\label{def:uresnet}
		The \emph{morphological UNet} with $n$ levels uses
		erosion-decimation in the encoder and adjoint
		dilation-interpolation in the decoder:
		\begin{eqnarray}
			L^{\downarrow}_{\ell}
			&=& \varepsilon^{\downarrow}_{R,b_\ell}\!\bigl(
			\SigSpec_\ell(L^{\downarrow}_{\ell-1})\bigr),
			\quad \ell = 1,\ldots,n,
			\label{eq:unet_enc}\\
			L^{\uparrow}_{\ell}
			&=& \delta^{*\uparrow}_{R,b_\ell}\!\bigl(
			\SigSpecDual_\ell(L^{\uparrow}_{\ell+1})\bigr)
			+ s_\ell,
			\quad \ell = n-1,\ldots,0,
			\label{eq:unet_dec}
		\end{eqnarray}
		where $\varepsilon^{\downarrow}_{R,b_\ell}$ is the
		erosion-decimation at stride $R$ with structuring function
		$b_\ell$ (\Cref{prop:adjoint_pyramid}),
		$\delta^{*\uparrow}_{R,b_\ell}$ its adjoint
		dilation-interpolation,
		and $s_\ell$ is the skip-connection signal at level $\ell$.
		Note that $\varepsilon^{\downarrow}_{R,b_\ell}$ replaces the
		$\APD_{R;\alpha}$ of the standard CNN encoder: the erosion
		provides both the thresholding (activation by \ ReLU) and the spatial
		pooling in one max-plus operator. Here 
		$\SigSpecDual_\ell = \sum_{i=1}^{K'_\ell} w'_{\ell,i}\ConvEros{k'_{\ell,i}}$ is the decoder spectral operator at level $\ell$, a linear combination of convolutions with learned kernels $\{k'_{\ell,i}\}$ that merges the upsampled features with the skip signal $s_\ell$; it plays the same role as $\SigSpec_\ell$ in the encoder but with independent weights.
		
		In the \emph{standard UNet}~\cite{ronneberger2015unet},
		$s_\ell = L^\downarrow_\ell$ (concatenation of encoder features).
		In the proposed \emph{UResNet}, the skip connection carries
		the top-hat residue:
		\begin{equation}
			s_\ell = \Gamma_{b_\ell}(L^\downarrow_\ell)
			= L^\downarrow_\ell - \gamma^{\mathrm{M}}_{b_\ell}
			(L^\downarrow_\ell),
			\label{eq:uresnet_skip}
		\end{equation}
		the detail discarded by the morphological opening at level $\ell$.
		This is the algebraic motivation for the UResNet:
		the decoder receives the exact information the encoder lost.
	\end{definition}
	
	\begin{proposition}[Skip connections as adjoint-pair residues]
		\label{prop:skip_adjoint}
		In the morphological UNet of \Cref{def:unet}, the encoder
		operator $\varepsilon^{\downarrow R}_{b_\ell}$ and decoder
		operator $\delta^{*\uparrow R}_{b_\ell}$ form an adjoint pair
		$(\varepsilon^{\downarrow R}_{b_\ell},\,
		\delta^{*\uparrow R}_{b_\ell})$ in the sense of
		\Cref{def:adjunction} (\Cref{prop:adjoint_pyramid}).
		The composition $\delta^{*\uparrow R}_{b_\ell} \circ
		\varepsilon^{\downarrow R}_{b_\ell} = \gamma^{\mathrm{M}}_{b_\ell}$
		is a morphological opening (the unit of the adjunction):
		it recovers a smoothed approximation of $f$, not $f$ itself.
		The information discarded by this opening is precisely the
		top-hat residue $\Gamma_{b_\ell}(f) = f -
		\gamma^{\mathrm{M}}_{b_\ell}(f) \geq 0$
		(\Cref{def:unet}, eq.~\eqref{eq:uresnet_skip}).
		In the UResNet, the skip connection $s_\ell =
		\Gamma_{b_\ell}(L^\downarrow_\ell)$ supplies this residue
		to the decoder, enabling the reconstruction identity:
		\begin{equation}
			\delta^{*\uparrow R}_{b_\ell}(\varepsilon^{\downarrow R}_{b_\ell}(f))
			+ \Gamma_{b_\ell}(f)
			= \gamma^{\mathrm{M}}_{b_\ell}(f)
			+ (f - \gamma^{\mathrm{M}}_{b_\ell}(f))
			= f.
			\label{eq:uresnet_reconstruction}
		\end{equation}
		That is, the UResNet decoder with top-hat skips achieves
		\emph{exact} reconstruction of the encoder input at each scale.
	\end{proposition}
	
	\begin{proof}
		The adjunction $(\varepsilon^{\downarrow R}_{b},
		\delta^{*\uparrow R}_{b})$ is established in
		\Cref{prop:adjoint_pyramid}.
		The unit $\gamma^{\mathrm{M}}_b = \delta^{*\uparrow R}_b \circ
		\varepsilon^{\downarrow R}_b$ is anti-extensive
		($\gamma^{\mathrm{M}}_b(f) \leq f$) and idempotent
		(a morphological opening, Proposition~\ref{prop:adjunction_props}(v)).
		The top-hat $\Gamma_b(f) = f - \gamma^{\mathrm{M}}_b(f) \geq 0$
		satisfies $\gamma^{\mathrm{M}}_b(f) + \Gamma_b(f) = f$ by
		definition, giving~\eqref{eq:uresnet_reconstruction}.
	\end{proof}

	\begin{figure}[H]
		\centering
		\resizebox{\linewidth}{!}{%
			\begin{tikzpicture}[		every node/.style={font=\small},
				Fbox/.style={draw=blue!55!black, fill=blue!8, rounded corners=4pt,
					inner sep=5pt, line width=0.8pt},
				Pbox/.style={draw=orange!70!black, fill=orange!8, rounded corners=4pt,
					inner sep=5pt, line width=0.8pt},
				Mbox/.style={draw=teal!65!black, fill=teal!8, rounded corners=4pt,
					inner sep=5pt, line width=0.8pt},
				Gbox/.style={draw=gray!60, fill=gray!6, rounded corners=4pt,
					inner sep=5pt, line width=0.7pt},
				arr/.style={->, >=stealth, semithick},
				xarr/.style={->, >=stealth, thick, red!70!black},
				sadj/.style={->, >=stealth, semithick, green!50!black},
				skip/.style={->, >=stealth, semithick, dashed, gray!65},
				rskip/.style={->, >=stealth, semithick, dashed, red!60!black},
				op/.style={font=\small\bfseries, above, inner sep=3pt},
				lat/.style={font=\scriptsize\itshape, below, inner sep=3pt,
					text=gray!60!black},
				title/.style={font=\small\bfseries},
				note/.style={font=\scriptsize, text=gray!55!black}]
				\node[title] at (5.0, 1.3){(a) Standard ResNet block: $L = \mathcal{F}(f) + f$};
				\node[Gbox]  (f0) at (0,0)    {$f$};
				\node[Gbox]  (Ff) at (5.0,0)  {$\mathcal{F}(f)$};
				\node[Gbox]  (s0) at (7.8,0)  {$\oplus$};
				\node[Gbox]  (L0) at (10.5,0) {$L$};
				
				\draw[arr]  (f0)--(Ff)  node[op,midway]{$\mathcal{F}$} node[lat,midway]{conv stack};
				\draw[arr]  (Ff)--(s0);
				\draw[arr]  (s0)--(L0);
				\draw[skip] (f0.north) to[out=40,in=140] node[op,midway]{$\mathrm{id}$} (s0.north);
				
				\node[title] at (5.5,-1.4)
				{(b) Morphological reading: $\mathcal{F}\approx\gamma^{\mathrm{M}}_b-\mathrm{id}$,
					block $= \gamma^{\mathrm{M}}_b(f)$};
				\node[Gbox]  (f1) at (0,-2.8)   {$f$};
				\node[Pbox]  (gf) at (4.0,-2.8) {$\varepsilon_b(f)$};
				\node[Pbox]  (gf2) at (8.0,-2.8){$\gamma^{\mathrm{M}}_b(f)$};
				\node[Pbox]  (Lm) at (11.6,-2.8){$\gamma^{\mathrm{M}}_b(f)$};
				
				\draw[arr]  (f1)--(gf)   node[op,midway]{$\varepsilon_b$} node[lat,midway]{erosion};
				\draw[sadj] (gf)--(gf2)  node[op,midway]{$\delta_{b^*}$}  node[lat,midway]{adjoint dilation};
				\draw[arr]  (gf2)--(Lm)  node[op,midway]{$\mathrm{id}$};
				
				\node[Gbox,font=\scriptsize] (th) at (4.0,-4.4){$\Gamma(f)=f-\gamma^{\mathrm{M}}_b(f)$};
				\draw[rskip] (f1.south) to[out=-55,in=180] (th.west);
				\draw[arr]   (th.east)  to[out=0,in=-95]   (gf2.south);
				\node[note, red!60!black] at (6.4,-5.0){top-hat residue injected before dilation};
				
				\draw[sadj, bend left=52] (Lm.north east)
				to node[above,note,text=green!50!black]
				{$\gamma^{\mathrm{M}}_b\!\circ\!\gamma^{\mathrm{M}}_b=\gamma^{\mathrm{M}}_b$
					(fixed point in one step)}
				(Lm.north west);
				
				\node[Pbox,font=\scriptsize] at (2.2,-6.2){Pointwise $(\mathscr{L},\!\leq)$};
				\draw[rskip] (5.0,-6.2)--(5.9,-6.2) node[right,note]{top-hat skip};
				\draw[sadj]  (8.1,-6.2)--(9.0,-6.2) node[right,note]{adjoint / idempotent};
			\end{tikzpicture}
		}
		\caption{Morphological interpretation of the ResNet block.
			\emph{(a)} Standard block: residual $\mathcal{F}$ acts on $f$; skip
			connection adds $f$ at $\oplus$.
			\emph{(b)} Morphological reading (Proposition~\ref{prop:resnet}):
			when $\mathcal{F}\approx\gamma^{\mathrm{M}}_b-\mathrm{id}$,
			the block computes $\gamma^{\mathrm{M}}_b(f)=\delta_{b^*}(\varepsilon_b(f))$.
			The red-dashed skip carries the top-hat residue
			$\Gamma(f)=f-\gamma^{\mathrm{M}}_b(f)$, re-injected before the adjoint
			dilation.
			The green loop expresses exact idempotency of the Type-I opening.}
		\label{fig:resnet_diagram}
	\end{figure}
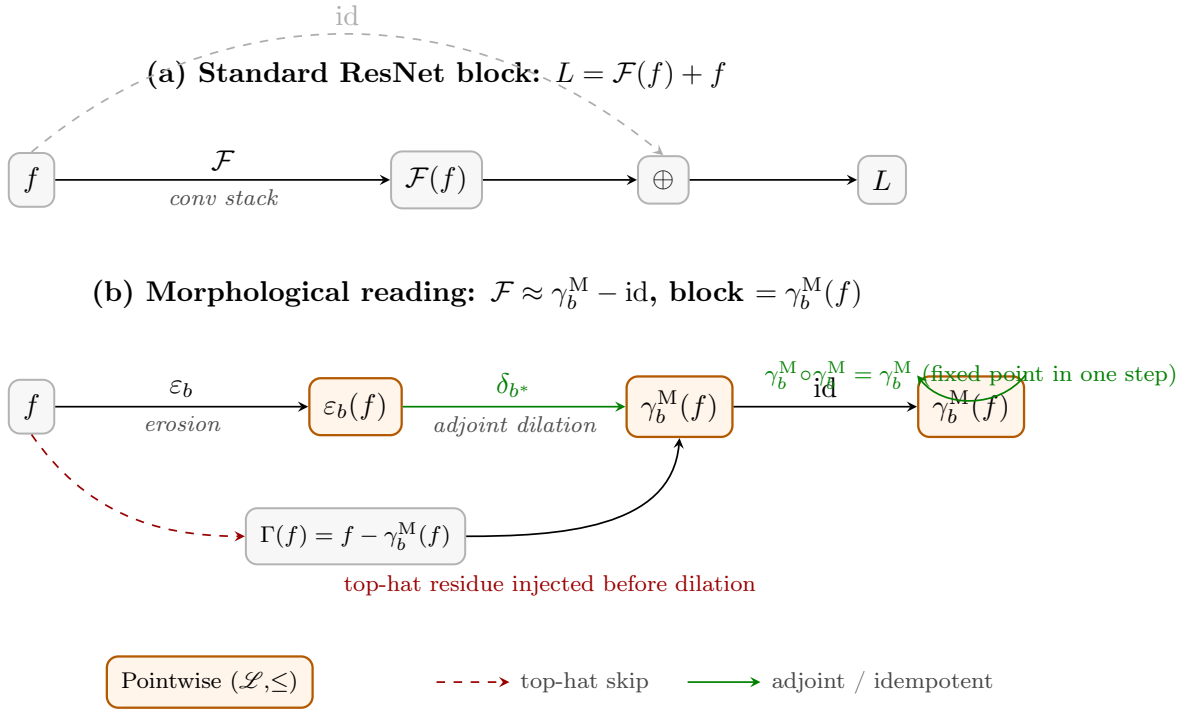

	\begin{figure}[H]
		\centering
		\resizebox{\linewidth}{!}{%
			\begin{tikzpicture}[		every node/.style={font=\small},
				Fbox/.style={draw=blue!55!black, fill=blue!8, rounded corners=4pt,
					inner sep=5pt, line width=0.8pt},
				Pbox/.style={draw=orange!70!black, fill=orange!8, rounded corners=4pt,
					inner sep=5pt, line width=0.8pt},
				Mbox/.style={draw=teal!65!black, fill=teal!8, rounded corners=4pt,
					inner sep=5pt, line width=0.8pt},
				Gbox/.style={draw=gray!60, fill=gray!6, rounded corners=4pt,
					inner sep=5pt, line width=0.7pt},
				arr/.style={->, >=stealth, semithick},
				xarr/.style={->, >=stealth, thick, red!70!black},
				sadj/.style={->, >=stealth, semithick, green!50!black},
				skip/.style={->, >=stealth, semithick, dashed, gray!65},
				rskip/.style={->, >=stealth, semithick, dashed, red!60!black},
				op/.style={font=\small\bfseries, above, inner sep=3pt},
				lat/.style={font=\scriptsize\itshape, below, inner sep=3pt,
					text=gray!60!black},
				title/.style={font=\small\bfseries},
				note/.style={font=\scriptsize, text=gray!55!black}]
				\node[title] at (4.2, 1.0)
				{(a) Standard UNet\quad(skip $=$ feature concatenation)};
				
				\node[Gbox]  (f)   at (0.0, 0)    {$f$};
				\node[Gbox]  (L1d) at (3.0, 0)    {$L^\downarrow_1$};
				\node[Gbox]  (L2d) at (6.0, 0)    {$L^\downarrow_2$};
				\node[Gbox]  (bot) at (9.0, 0)    {$L^\downarrow_n$};
				\node[Gbox]  (L2u) at (9.0,-1.9)  {$L^\uparrow_{n-1}$};
				\node[Gbox]  (L1u) at (6.0,-1.9)  {$L^\uparrow_1$};
				\node[Gbox]  (out) at (3.0,-1.9)  {output};
				
				\draw[arr] (f)--(L1d)   node[op,midway,font=\scriptsize]
				{$\APD\!\circ\!\Sigma^{\mathrm{Spec}}_1$}
				node[lat,midway]{cross-lattice};
				\draw[arr] (L1d)--(L2d) node[op,midway,font=\scriptsize]
				{$\APD\!\circ\!\Sigma^{\mathrm{Spec}}_2$}
				node[lat,midway]{cross-lattice};
				\draw[arr] (L2d)--(bot) node[op,midway]{$\cdots$};
				\draw[arr] (bot)--(L2u) node[op,midway,font=\scriptsize]
				{$\Sigma^{\mathrm{Spec},\prime}_{n-1}$}
				node[lat,midway]{$\uparrow+$conv};
				\draw[arr] (L2u)--(L1u) node[op,midway,font=\scriptsize]
				{$\Sigma^{\mathrm{Spec},\prime}_1$}
				node[lat,midway]{$\uparrow+$conv};
				\draw[arr] (L1u)--(out) node[op,midway,font=\scriptsize]
				{$\Sigma^{\mathrm{Spec},\prime}_0$};
				
				\draw[skip] (L1d.south) to[out=-80,in=100]
				node[right,note]{cat} (L2u.north);
				\draw[skip] (f.south)   to[out=-90,in=90]
				node[right,note]{cat} (L1u.north);
				
				\node[title] at (5.0,-3.1)
				{(b) Morphological UResNet\quad
					(skip $= \Gamma_b(f)=f-\gamma^{\mathrm{M}}_b(f)$)};
				
				\node[Gbox]  (mf)   at (0.0,-4.3) {$f$};
				\node[Pbox]  (mL1d) at (3.2,-4.3) {$\varepsilon^{\downarrow R}_{b_1}(f)$};
				\node[Pbox]  (mL2d) at (6.6,-4.3) {$\varepsilon^{\downarrow R}_{b_2}(\cdots)$};
				\node[Pbox]  (mbot) at (10.2,-4.3){$L^\downarrow_n$};
				\node[Pbox]  (mL2u) at (10.2,-6.1){$\delta^{*\uparrow R}_{b_2}(\cdots)$};
				\node[Pbox]  (mL1u) at (6.6,-6.1) {$\delta^{*\uparrow R}_{b_1}(\cdots)$};
				\node[Gbox]  (mout) at (3.2,-6.1) {output};
				
				\draw[arr]  (mf)--(mL1d)   node[op,midway]{$\varepsilon^{\downarrow R}_{b_1}$}
				node[lat,midway]{erosion-decimation};
				\draw[arr]  (mL1d)--(mL2d) node[op,midway]{$\varepsilon^{\downarrow R}_{b_2}$}
				node[lat,midway]{erosion-decimation};
				\draw[arr]  (mL2d)--(mbot) node[op,midway]{$\cdots$};
				\draw[sadj] (mbot)--(mL2u) node[op,midway]{$\delta^{*\uparrow R}_{b_2}$}
				node[lat,midway]{adjoint interp.};
				\draw[sadj] (mL2u)--(mL1u) node[op,midway]{$\delta^{*\uparrow R}_{b_1}$}
				node[lat,midway]{adjoint interp.};
				\draw[sadj] (mL1u)--(mout) node[op,midway]{$\delta^{*\uparrow R}_{b_0}$};
				
				\draw[rskip] (mL1d.south) to[out=-80,in=100]
				node[right,note,red!65!black]
				{$\Gamma_{b_1}(f)=f-\gamma^{\mathrm{M}}_{b_1}(f)$} (mL2u.north);
				\draw[rskip] (mf.south)   to[out=-90,in=90]
				node[right,note,red!65!black]{$\Gamma_{b_0}(f)$}      (mL1u.north);
				
				\draw[<->, green!45!black, semithick]
				($(mL1d)+(0,-0.55)$) -- node[below,note,text=green!50!black]
				{adjoint pair $(\varepsilon^{\downarrow R}_{b_1},\delta^{*\uparrow R}_{b_1})$}
				($(mL1u)+(0,0.55)$);
				
				\node[Pbox,font=\scriptsize] at (2.0,-7.5){Pointwise $(\mathscr{L},\!\leq)$};
				\draw[rskip] (4.8,-7.5)--(5.7,-7.5) node[right,note]{top-hat skip};
				\draw[sadj]  (7.6,-7.5)--(8.5,-7.5) node[right,note]{adjoint pair};
			\end{tikzpicture}
		}
		\caption{Standard UNet~(a) vs.\ morphological UResNet~(b).
			Colours: \textcolor{orange!70!black}{orange} = pointwise lattice.
			\emph{(a)} Encoder: strided convolutions with max-pooling (cross-lattice);
			decoder: transposed convolutions; skip connections (dashed) carry concatenated features.
			\emph{(b)} Encoder: erosion-decimation $\varepsilon^{\downarrow R}_{b_\ell}$;
			decoder: adjoint dilation-interpolation $\delta^{*\uparrow R}_{b_\ell}$
			(green arrows, forming the adjoint pair of Proposition~\ref{prop:adjoint_pyramid});
			skip connections (red dashed) carry the top-hat residue
			$\Gamma_b(f)=f-\gamma^{\mathrm{M}}_b(f)$ at each scale
			(Proposition~\ref{prop:skip_adjoint}).
			The double-headed green arrow marks the adjoint relationship between
			encoder and decoder at level~$\ell$.}
		\label{fig:unet_diagram}
	\end{figure}


	\clearpage
	
	\section{Fixed Points, Idempotency, and Convergence of
		Morphological Layers}
	\label{sec:fixed_points}

	A fundamental property of morphological openings $\gamma$ and closings $\varphi$ is
	\emph{idempotency}: $\gamma \circ \gamma = \gamma$ and
	$\varphi \circ \varphi = \varphi$.
	In classical morphology, applying the operator twice yields no
	further change after the first application, and the image of the
	operator is exactly its fixed-point set.
	We investigate when analogous phenomena arise in the morphological
	models of \Cref{sec:cnn_models}, clarify the conditions required,
	and draw a few architectural conclusions.
	This analysis connects to the author's group-equivariant
	fixed-point theory in~\cite{angulo2025gsi}; we work here in the
	standard translation-equivariant setting.
	An experimental exploration of the fixed-point theory developed
	here was established in~\cite{velasco2022bmvc}, were we showed that morphological 
	layers constrained to their fixed-point set (geodesic operators) can be 
	trained effectively and achieve competitive performance; the present paper provides the
	rigorous algebraic foundation for that computational programme and to
	identify precisely which CNN layers are idempotent and which are not.

	\begin{table}[ht]
		\centering
		\caption{Principal results of \S\ref{sec:fixed_points} (Fixed Points, Idempotency, Convergence). Results marked $(\star)$ are the paper's principal findings.}
		\label{tab:summary_fixed_points}
		\renewcommand{\arraystretch}{1.38}
		\footnotesize
		\begin{tabular}{p{0.82cm}p{2.9cm}p{8.4cm}}
			\toprule
			\textbf{Ref.} & \textbf{Name} & \textbf{Statement and role} \\
			\midrule
			Lem~\ref{lem:conv_adjoint} & $\:$ Adjoint of convolution & The exact upper adjoint of $\varepsilon^{\mathrm{Conv}}_k$ in $(\mathscr{L},\leq)$ is non-local; $\MaxPool{R}$ is not that adjoint, confirming the cross-lattice structure. \\
			Thm~\ref{thm:two_openings} $(\star)$ & $\:$ Two principled openings & Type-I: $\gamma^{\mathrm{M}}_b=\delta_{b^*}\circ\varepsilon_b$ in $(\mathscr{L},\leq)$ exactly idempotent. Type-II: $\gamma^{\mathrm{F}}_k$ in $(L^n,\leq_F)$ exactly idempotent at $\epsilon=0$, approximately so for $\epsilon>0$. \\
			Thm~\ref{thm:cnn_not_opening} $(\star)$ & $\:$ CNN is not an opening & Standard CNN pipeline is cross-lattice: not idempotent, not anti-extensive in $(\mathscr{L},\leq)$, no adjunction spans the full pipeline. \\
			Cor~\ref{cor:cnn_idempotent} $(\star)$ & $\:$ Three idempotent designs & Types I, II, III are the three disciplined CNN-like designs yielding morphological openings; standard CNNs satisfy none of the three conditions. \\
			Thm~\ref{thm:cnn_bias_fixedpoint} & $\:$ Bias as threshold & Fixed-point set of a Type-I layer with bias $\alpha<0$: signals whose erosion output exceeds $-\alpha$. Increasing $|\alpha|$ raises the activation threshold. \\
			Thm~\ref{thm:convergence} $(\star)$ & $\:$ One-step convergence & Iterating a Type-I opening stabilises after exactly one step ($f^{(n)}=\gamma(f)$ for $n\geq 1$); the naive residual $\gamma(f)+f$ diverges with at least linear growth. \\
			Thm~\ref{thm:unet_idempotent} & $\:$ UNet skeleton idempotency & The morphological UNet pyramid skeleton is an idempotent opening; the full UNet with $\SigSpec$ convolutions is generally not. \\
			\bottomrule
		\end{tabular}
	\end{table}
	
	\begin{definition}[Fixed-point operator and idempotent layer]
		\label{def:fixed_point}
		An operator $\Phi$ on a complete lattice $(\mathcal{L},\leq)$
		is a \emph{fixed-point operator} if there exists $f_* \in
		\mathcal{L}$ such that $\Phi(f_*) = f_*$.
		It is \emph{idempotent} if $\Phi \circ \Phi = \Phi$, i.e.,
		every element of its image is a fixed point.
		The \emph{fixed-point set} of $\Phi$ is $\mathrm{Fix}(\Phi) =
		\{f \in \mathcal{L} : \Phi(f) = f\}$.
	\end{definition}
	
	\paragraph{Bounds on the adjoint of convolution in the
		pointwise lattice.}
	Before stating the main results, we clarify the lattice-theoretic
	status of the linear convolution operator
	$\varepsilon^{\mathrm{Conv}}_k(f) = f * k$ in the
	\emph{pointwise} lattice $(\Fun(E,\R), \leq)$.
	
	When $k \geq 0$ and $\sum_i k(x_i) = 1$, the convolution
	$f \mapsto f * k$ is increasing in the pointwise lattice
	(Proposition~\ref{prop:conv_conditions}).
	By the general adjunction theorem for complete lattices
	(Proposition~\ref{prop:adjunction_props}(iii)), every
	increasing map between complete lattices possesses a unique
	upper adjoint.
	However, as we now show, the upper adjoint of convolution
	in the pointwise lattice is \emph{not} a standard
	morphological operator, and dually its lower adjoint is
	also degenerate.
	The correct picture is that convolution in the pointwise
	lattice admits a \emph{pair of one-sided bounds}, an upper
	bound derived from the adjunction condition from above,
	and a lower bound derived from it below, but no exact
	adjoint that is simultaneously tight on both sides.
	
	\begin{lemma}[Bounds on the pointwise adjoint of convolution]
		\label{lem:conv_adjoint}
		Let $k \geq 0$ with finite support $\Spt(k) =
		\{x_1,\ldots,x_N\}$, $k(x_i) > 0$, and
		$G^+ = \sum_i k(x_i)$.
		
		\bigskip
		\emph{Upper bound (from the adjunction condition
			$f * k \leq g$, valid for $f \geq 0$).}
		Assume $f \geq 0$.
		From position $y = x + x_i$, the constraint $(f*k)(y)
		= \sum_j k(x_j)f(y-x_j) \leq g(y)$ implies, by dropping
		the non-negative cross-terms $\sum_{j \neq i}k(x_j)f(y-x_j)
		\geq 0$:
		\begin{equation}
			f(x)
			\;\leq\;
			\delta^{\mathrm{adj},+}_k(g)(x)
			\;:=\;
			\min_{1 \leq i \leq N} \frac{g(x+x_i)}{k(x_i)},
			\quad x \in E,
			\label{eq:conv_upper_bound}
		\end{equation}
		valid when $f \geq 0$ and $k \geq 0$.
		For general $f$ (possibly taking negative values), dropping
		the cross-terms is not valid since they may be negative,
		and the bound~\eqref{eq:conv_upper_bound} need not hold.
		
		The operator $\delta^{\mathrm{adj},+}_k$ is the
		\emph{max-times erosion} $\varepsilon^\times_k(g)(x)$
		of \Cref{def:maxtimes} (an erosion in the multiplicative
		lattice), and is \emph{not} a max-plus erosion.
		It is \emph{not} the exact upper adjoint of
		$\varepsilon^{\mathrm{Conv}}_k$ in the pointwise lattice
		for $N > 1$: the backward direction
		$f \leq \delta^{\mathrm{adj},+}_k(g) \Rightarrow f*k \leq g$
		fails; one can only conclude $f*k \leq G^+\,g$
		(a factor-$G^+$ inflation, as shown in the proof).
		For $N = 1$ (single-point kernel $k = c\,\delta_{x_1}$),
		the bound is exact in both directions:
		$\delta^{\mathrm{adj},+}_k(g)(x) = g(x+x_1)/c$
		is the exact adjoint, i.e.,
		$f(x) \leq g(x+x_1)/c \iff cf(x-x_1) \leq g(x)$.

		\bigskip

		\emph{Lower bound (from the adjunction condition
			$g \leq f * k$).}
		Placing the convolution on the other side of the inequality,
		the condition $g(y) \leq (f*k)(y)$ for all $y$ is implied
		by a global constant lower bound on $f$:
		if $f(x) \geq c$ for all $x \in E$, then
		$(f*k)(y) = \sum_i k(x_i)f(y-x_i) \geq c\,G^+$,
		so the condition $g(y) \leq (f*k)(y)$ is implied by
		$c \geq g(y)/G^+$ for all $y$.
		The tightest such global bound is:
		\begin{equation}
			f(x) \;\geq\;
			\delta^{\mathrm{adj},-}_k(g)
			\;:=\;
			\frac{\sup_{y \in E}\, g(y)}{G^+},
			\quad x \in E.
			\label{eq:conv_lower_bound}
		\end{equation}
		This is a \emph{global} (spatially uniform, non-pointwise)
		lower bound: it bounds $f$ everywhere by the same constant,
		determined by the supremum of $g$ and the total kernel mass.
		
		\emph{Pointwise lower bound.}
		A pointwise bound on $f(x)$ alone from $g \leq f*k$
		cannot be derived in closed form without assumptions
		on the values of $f$ at the other support points
		$\{y - x_j : j \neq i\}$.
		Specifically, from the constraint at $y = x + x_i$:
		\begin{equation}
			g(x+x_i)
			\;\leq\;
			k(x_i)\,f(x)
			+ \textstyle\sum_{j \neq i} k(x_j)\,f(x+x_i-x_j),
			\label{eq:conv_lb_one_term}
		\end{equation}
		rearranging gives $f(x) \geq (g(x+x_i)
		- \sum_{j\neq i}k(x_j)f(x+x_i-x_j))/k(x_i)$,
		which depends on $f$ at the $N-1$ other points.
		If $f \geq 0$ everywhere and $\sup f \leq M$ (a known
		upper bound), dropping the non-negative cross-terms
		$\sum_{j\neq i}k(x_j)f(x+x_i-x_j) \leq (G^+-k(x_i))M$
		gives the closed-form pointwise lower bound:
		\begin{equation}
			f(x)
			\;\geq\;
			\max_{1 \leq i \leq N}
			\frac{g(x+x_i) - (G^+ - k(x_i))\,M}{k(x_i)},
			\label{eq:conv_lower_bound2}
		\end{equation}
		valid when $0 \leq f \leq M$ pointwise.
		For the flat box kernel ($k \equiv 1/N$, $G^+=1$) with
		$f \geq 0$ and $f \leq M$:
		\begin{equation}
			f(x) \;\geq\;
			\max_{1 \leq i \leq N}
			\bigl[N\,g(x+x_i) - (N-1)M\bigr],
			\label{eq:conv_lb_flat}
		\end{equation}
		which is non-trivial only when $g(x+x_i) > (N-1)M/N$
		for some $i$.
		When no bound $M$ is available,
		\emph{no closed-form pointwise lower bound on $f(x)$ alone}
		can be derived from $g \leq f*k$: the constraint couples
		$f$ at $N$ different positions simultaneously and cannot
		be decoupled without additional information.
	\end{lemma}
	
	\begin{proof}
		\emph{Upper bound (for $f \geq 0$).}
		Fix $x \in E$ and $i \in \{1,\ldots,N\}$.
		Set $y = x + x_i$.
		The constraint $(f*k)(y) \leq g(y)$ gives:
		\[
		k(x_i)f(x) + \sum_{j \neq i} k(x_j)f(x+x_i-x_j)
		\leq g(x+x_i).
		\]
		Since $f \geq 0$ and $k \geq 0$, the cross-terms satisfy
		$\sum_{j \neq i}k(x_j)f(x+x_i-x_j) \geq 0$.
		Dropping them gives $k(x_i)f(x) \leq g(x+x_i)$,
		i.e., $f(x) \leq g(x+x_i)/k(x_i)$.
		Since this holds for every $i$, taking the minimum over
		$i$ gives~\eqref{eq:conv_upper_bound}.
		
		For the backward direction: if $f(x) \leq \min_i g(x+x_i)/k(x_i)$
		for all $x$, then for each term in the convolution:
		$k(x_i)f(y-x_i) \leq k(x_i) \cdot g(y)/k(x_i) = g(y)$
		(using the $j=i$ instance of the minimum).
		Summing: $(f*k)(y) = \sum_i k(x_i)f(y-x_i) \leq N\,g(y)$
		in the unweighted case, or more precisely
		$(f*k)(y) \leq G^+\,g(y)/\min_i k(x_i)$.
		This gives $f*k \leq G^+\,g$, not $f*k \leq g$:
		the backward implication fails by a factor of $G^+$ for $N>1$.
		
		The log-domain connection: $\min_i g(x+x_i)/k(x_i)
		= \exp(\min_i\{\log g(x+x_i) - \log k(x_i)\})$,
		which is the max-times erosion $\varepsilon^\times_k(g)(x)$
		of~\eqref{eq:maxtimes_erosion}.
		
		For a single-point kernel $k = c\,\delta_{x_1}$:
		$(f*k)(y) = cf(y-x_1) \leq g(y)$ iff
		$f(x) \leq g(x+x_1)/c$ exactly, with no cross-terms
		and no factor inflation: the formula is the exact adjoint.

		\bigskip
		\emph{Lower bound.}
		Global bound~\eqref{eq:conv_lower_bound}:
		if $f(x) \geq c$ for all $x$, then
		$(f*k)(y) = \sum_i k(x_i)f(y-x_i) \geq c\sum_i k(x_i) = cG^+$.
		For $g(y) \leq (f*k)(y)$ to hold for all $y$, it suffices
		that $cG^+ \geq \sup_y g(y)$, i.e.,
		$c \geq \sup_y g(y)/G^+$.
		
		Pointwise bound~\eqref{eq:conv_lower_bound2}:
		from~\eqref{eq:conv_lb_one_term} with $0\leq f\leq M$,
		the cross-terms satisfy
		$0 \leq \sum_{j\neq i}k(x_j)f(x+x_i-x_j) \leq (G^+-k(x_i))M$,
		so $k(x_i)f(x) \geq g(x+x_i) - (G^+-k(x_i))M$,
		giving $f(x) \geq (g(x+x_i)-(G^+-k(x_i))M)/k(x_i)$.
		Taking the maximum over $i=1,\ldots,N$ gives the
		tightest bound obtainable from this approach.
		For the flat kernel, $k(x_i)=1/N$ and $G^+-k(x_i)=(N-1)/N$,
		so the bound becomes $f(x)\geq N\,g(x+x_i)-(N-1)M$
		for each $i$, and the maximum over $i$ gives~\eqref{eq:conv_lb_flat}.
		
		The impossibility of a closed-form pointwise bound for
		general $f$ follows from the observation that the constraint
		$g(y)\leq(f*k)(y)$ at a single $y=x+x_i$ involves the $N$
		values $\{f(x+x_i-x_j)\}_{j=1}^N$ jointly; decoupling to
		a bound on $f(x)$ alone requires knowing or bounding the
		remaining $N-1$ values, which in general requires global
		information about $f$.
	\end{proof}
	
	\begin{remark}[Status of the adjoint of convolution]
		\label{rem:conv_adjoint_not_dilation}
		Lemma~\ref{lem:conv_adjoint} establishes that the operator
		adjoint of convolution in the pointwise lattice is not well-defined.
		This stands in contrast to two other contexts:
		\begin{itemize}[leftmargin=*]
			\item In the \emph{Fourier lattice} $(L^n,\leq_F)$: the
			adjoint of convolution is the Wiener filter
			$\delta^{*,\mathrm{Conv}}_k$~\cite{keshet2000,angulo2025dgmm}, 
			a linear operator.
			\item In the \emph{max-plus (pointwise) lattice}: the adjoint of the
			max-plus erosion $\varepsilon_b$ is the max-plus dilation
			$\delta_{{b}^{*}}$ (\Cref{def:adjunction}).
		\end{itemize}
		This is the algebraic reason why composing a convolution with
		a flat max-pooling (a max-plus dilation) does not yield a
		morphological opening and there is no other pointwise dilation 
		or erosion which would lead to an opening.
	\end{remark}
	
	
	\begin{figure}[H]
		\centering
		\resizebox{\linewidth}{!}{%
			\begin{tikzpicture}[		every node/.style={font=\small},
				Fbox/.style={draw=blue!55!black, fill=blue!8, rounded corners=4pt,
					inner sep=5pt, line width=0.8pt},
				Pbox/.style={draw=orange!70!black, fill=orange!8, rounded corners=4pt,
					inner sep=5pt, line width=0.8pt},
				Mbox/.style={draw=teal!65!black, fill=teal!8, rounded corners=4pt,
					inner sep=5pt, line width=0.8pt},
				Gbox/.style={draw=gray!60, fill=gray!6, rounded corners=4pt,
					inner sep=5pt, line width=0.7pt},
				arr/.style={->, >=stealth, semithick},
				xarr/.style={->, >=stealth, thick, red!70!black},
				sadj/.style={->, >=stealth, semithick, green!50!black},
				skip/.style={->, >=stealth, semithick, dashed, gray!65},
				rskip/.style={->, >=stealth, semithick, dashed, red!60!black},
				op/.style={font=\small\bfseries, above, inner sep=3pt},
				lat/.style={font=\scriptsize\itshape, below, inner sep=3pt,
					text=gray!60!black},
				title/.style={font=\small\bfseries},
				note/.style={font=\scriptsize, text=gray!55!black}]
				\node[title, red!65!black] at (5.5, 1.0)
				{(a) Standard CNN: cross-lattice, \textbf{not} idempotent};
				\node[Gbox] (f0) at (0,0)    {$f$};
				\node[Fbox] (h0) at (3.6,0)  {$f{\ast}k$};
				\node[Pbox] (h1) at (7.2,0)  {$\sigma(f{\ast}k)$};
				\node[Pbox] (L0) at (11.0,0) {$L_n$};
				
				\draw[arr] (f0)--(h0)
				node[op,midway]{$\varepsilon^{\mathrm{Conv}}_k$}
				node[lat,midway]{Fourier erosion};
				\draw[xarr] (h0)--(h1)
				node[op,midway]{$\delta^{\mathrm{ReLU}}$}
				node[lat,midway]{\textcolor{red!70!black}{cross-lattice}};
				\draw[arr] (h1)--(L0)
				node[op,midway]{$\MaxPool{R}$}
				node[lat,midway]{pointwise dilation};
				\node[note,gray!65,align=center] at (3.6,-1.3)
				{adjoint in $(L^n,\!\leq_F)$:\\Wiener filter};
				\node[note,gray!65,align=center] at (9.0,-1.3)
				{adjoint in $(\mathscr{L},\!\leq)$:\\max-unpooling};
				
				\node[title, green!45!black] at (4.5,-2.6)
				{(b) Type-I: pointwise lattice throughout — \textbf{exactly idempotent}};
				\node[Gbox] (f1) at (0,-3.8)   {$f$};
				\node[Pbox] (h2) at (4.5,-3.8) {$\varepsilon_b(f)$};
				\node[Pbox] (L1) at (9.0,-3.8) {$\gamma^{\mathrm{M}}_b(f)$};
				
				\draw[arr] (f1)--(h2)
				node[op,midway]{$\varepsilon_b$}
				node[lat,midway]{pointwise erosion, $\varepsilon_b(f)\!\leq\!f$};
				\draw[sadj] (h2)--(L1)
				node[op,midway]{$\delta_{{b}^{*}}$}
				node[lat,midway]{adjoint dilation, same lattice};
				\draw[sadj, bend left=50] (L1.north east)
				to node[above,note,text=green!50!black]
				{$\gamma^{\mathrm{M}}_b\!\circ\!\gamma^{\mathrm{M}}_b=\gamma^{\mathrm{M}}_b$}
				(L1.north west);
				\draw[<->, green!45!black, semithick]
				($(h2)+(0,-0.9)$) to
				node[below,note,text=green!50!black]
				{adjunction $(\varepsilon_b,\delta_{b^*})$ in $(\mathscr{L},\leq)$}
				($(L1)+(0,-0.9)$);
				
				\node[title, blue!55!black] at (4.5,-5.4)
				{(c) Type-II: Fourier lattice throughout — \textbf{idempotent at $\epsilon\!\to\!0$}};
				\node[Gbox] (f2) at (0,-6.6)   {$f$};
				\node[Fbox] (h3) at (4.5,-6.6) {$f{\ast}k$};
				\node[Fbox] (L2) at (9.0,-6.6) {$\gamma^{\mathrm{F}}_k(f)$};
				
				\draw[arr] (f2)--(h3)
				node[op,midway]{$\varepsilon^{\mathrm{Conv}}_k$}
				node[lat,midway]{Fourier erosion};
				\draw[sadj] (h3)--(L2)
				node[op,midway]{$\delta^{*,\mathrm{Conv}}_k$}
				node[lat,midway]{Wiener deconv., same lattice};
				\draw[sadj, bend left=50] (L2.north east)
				to node[above,note,text=green!50!black]
				{$\gamma^{\mathrm{F}}_k\!\circ\!\gamma^{\mathrm{F}}_k\to\gamma^{\mathrm{F}}_k$
					as $\epsilon\!\to\!0$}
				(L2.north west);
				\draw[<->, green!45!black, semithick]
				($(h3)+(0,-0.9)$) to
				node[below,note,text=green!50!black]
				{adjunction $(\varepsilon^{\mathrm{Conv}}_k,\delta^{*,\mathrm{Conv}}_k)$
					in $(L^n,\!\leq_F)$}
				($(L2)+(0,-0.9)$);
				
				\node[Fbox,font=\scriptsize] at (1.0,-8.2){Fourier $(L^n,\!\leq_F)$};
				\node[Pbox,font=\scriptsize] at (4.3,-8.2){Pointwise $(\mathscr{L},\!\leq)$};
				\draw[xarr] (6.7,-8.2)--(7.6,-8.2) node[right,note]{cross-lattice jump};
				\draw[sadj] (9.5,-8.2)--(10.4,-8.2) node[right,note]{same-lattice adjoint};
			\end{tikzpicture}
		}
		\caption{Three CNN-like compositions and their lattice structure.
			Colours: \textcolor{blue!60!black}{blue} = Fourier,
			\textcolor{orange!70!black}{orange} = pointwise.
			Red arrows mark cross-lattice jumps; green arrows and loops mark
			same-lattice adjoint pairs and idempotency.
			\emph{(a)} Standard CNN (Theorem~\ref{thm:cnn_not_opening}): cross-lattice,
			not idempotent.
			\emph{(b)} Type-I opening (Theorem~\ref{thm:two_openings}(i)):
			both operators in $(\mathscr{L},\leq)$; exactly idempotent.
			\emph{(c)} Type-II opening (Theorem~\ref{thm:two_openings}(ii)):
			both operators in $(L^n,\leq_F)$; idempotent in the limit $\epsilon\to 0$.}
		\label{fig:idempotency_diagram}
	\end{figure}
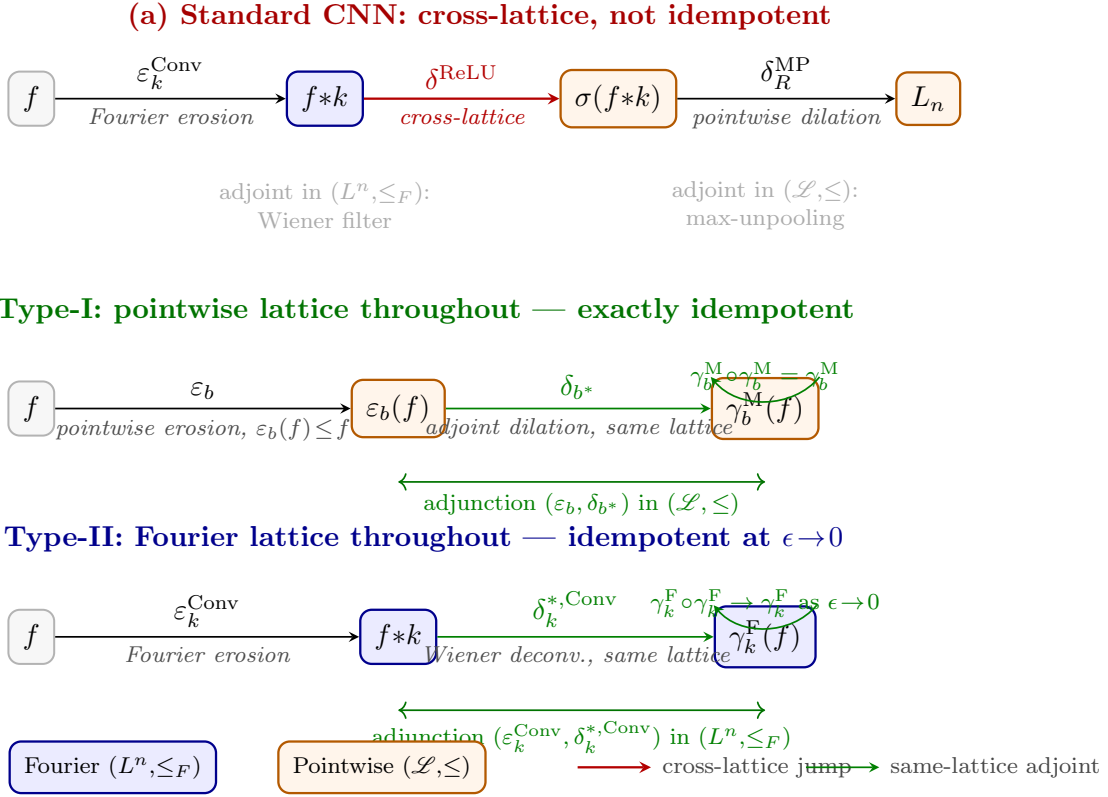
	

	With this clarification established, we now identify the two
    CNN-like compositions that \emph{do} form genuine
	morphological openings, and examine what standard CNNs are in
	comparison.
	
	\begin{theorem}[Two opening compositions]
		\label{thm:two_openings}
		The following two compositions yield operators that are
		increasing and anti-extensive.
		Type~(i) is a \emph{morphological opening} (hence exactly
		idempotent); Type~(ii) is an exact morphological opening in
		$(L^n,\leq_F)$ only in the limit $\epsilon\to 0$, and an
		approximately idempotent operator for fixed $\epsilon>0$
		(cf.\ Figure~\ref{fig:idempotency_diagram}):
		\begin{enumerate}[label=(\roman*)]
			\item \emph{(Type-I: pure morphological opening.)}
			Let $b: E \to \Rbar$ be a structuring function with
			compact support $W$ and $b^*(x) = b(-x)$ its transpose.
			Then
			\begin{equation}
				\gamma^{\mathrm{M}}_b(f)
				= \delta_{{b}^{*}}(\varepsilon_b(f)),			
				\label{eq:pure_morpho_opening}
			\end{equation}
			with
			\begin{eqnarray*}
				(\varepsilon_b f)(x) &=& \inf_{y \in W}\{f(x+y)-b(y)\}, \\
				(\delta_{{b}^{*}} g)(x) &=& \sup_{y \in W}\{g(x-y)+b(-y)\},
			\end{eqnarray*}
			is a morphological opening in the pointwise lattice
			$(\Fun(E,\Rbar), \leq)$.
			The pair $(\varepsilon_b, \delta_{{b}^{*}})$ is an adjunction
			in $(\mathscr{L},\leq)$, and the opening is exactly idempotent:
			$\gamma^{\mathrm{M}}_b \circ \gamma^{\mathrm{M}}_b =
			\gamma^{\mathrm{M}}_b$.
			
			\item \emph{(Type-II: spectral operator in the Fourier lattice.)}
			Let $k$ be an absolutely summable kernel with Fourier
			transform $K(\omega) = \mathcal{F}\{k\}(\omega)$.
			The pair $(\varepsilon^{\mathrm{Conv}}_k,
			\delta^{*,\mathrm{Conv}}_k)$ is an adjunction in
			$(L^n,\leq_F)$~\cite{keshet2000,angulo2025dgmm}.
			The composition
			$\gamma^{\mathrm{F}}_k = \delta^{*,\mathrm{Conv}}_k
			\circ \varepsilon^{\mathrm{Conv}}_k$
			acts in the Fourier domain as the spectral multiplier:
			\begin{equation}
				\widehat{\gamma^{\mathrm{F}}_k(f)}(\omega)
				= \frac{|K(\omega)|^2}{|K(\omega)|^2 + \epsilon^2}\,
				F(\omega),
				\label{eq:fourier_opening}
			\end{equation}
			the Tikhonov-regularised Wiener deconvolution composed
			with the erosion $\varepsilon^{\mathrm{Conv}}_k(f)=f*k$.
			The multiplier $M_\epsilon(\omega)
			= |K(\omega)|^2/(|K(\omega)|^2+\epsilon^2)$ satisfies
			$0 < M_\epsilon < 1$ for $\epsilon > 0$, so
			$\gamma^{\mathrm{F}}_k$ is strictly anti-extensive for all
			$\epsilon>0$.
			
			\emph{Exact idempotency} holds only in the limit
			$\epsilon \to 0$: $M_0(\omega) = \mathbf{1}_{K(\omega)\neq 0}$
			(ideal filter: the spectral projection onto the range of $k$) satisfies
			$M_0^2 = M_0$, so $\gamma^{\mathrm{F}}_k\big|_{\epsilon=0}$
			is a true morphological opening in $(L^n,\leq_F)$.
			
			At any \emph{fixed} $\epsilon > 0$, $\gamma^{\mathrm{F}}_k$
			is \emph{not} idempotent: $M_\epsilon^2 < M_\epsilon$ on
			the open set where $0 < M_\epsilon < 1$, giving
			$\gamma^{\mathrm{F}}_k \circ \gamma^{\mathrm{F}}_k \neq
			\gamma^{\mathrm{F}}_k$.
			The idempotency error is bounded by
			$\|\gamma^{\mathrm{F}}_k\circ\gamma^{\mathrm{F}}_k(f)
			- \gamma^{\mathrm{F}}_k(f)\|
			\leq (1-M_\epsilon^{\min})\|\gamma^{\mathrm{F}}_k(f)\|$
			where $M_\epsilon^{\min} = \inf_\omega M_\epsilon(\omega)
			\to 1$ as $\epsilon\to 0$.
			For practical purposes (small $\epsilon$) the operator is
			approximately idempotent with controlled error.
		\end{enumerate}
	\end{theorem}
	
	\begin{proof}
		(i) $(\varepsilon_b, \delta_{{b}^{*}})$ is an adjunction in
		$(\Fun(E,\Rbar),\leq)$: $\varepsilon_b(f) \leq g \iff f \leq
		\delta_{{b}^{*}}(g)$ (standard result,
		Heijmans~\cite{heijmans1994}, Chapter~3).
		By Proposition~\ref{prop:adjunction_props}(v), the composition
		$\gamma^{\mathrm{M}}_b = \delta_{{b}^{*}} \circ \varepsilon_b$ is
		an opening, hence exactly idempotent.
		(ii) By the spectral framework~\cite{keshet2000}, $(\varepsilon^{\mathrm{Conv}}_k,
		\delta^{*,\mathrm{Conv}}_k)$ is an adjunction in $(L^n,\leq_F)$.
		The Fourier-domain expression follows from composing
		the erosion $\hat f \mapsto K\cdot\hat f$ with the
		Wiener adjoint dilation $\hat g\mapsto(\bar K/(|K|^2+\epsilon^2))
		\cdot\hat g$: the product is
		$|K|^2/(|K|^2+\epsilon^2)$.
		Anti-extensivity: $M_\epsilon < 1$ for $\epsilon>0$.
		At $\epsilon=0$: $M_0 = \mathbf{1}_{K\neq 0}$
		satisfies $M_0^2 = M_0$, so $\gamma^{\mathrm{F}}_k\big|_{\epsilon=0}$
		is an exact morphological opening.
		For $\epsilon>0$: $\gamma^{\mathrm{F}}_k\circ\gamma^{\mathrm{F}}_k(f)
		= \mathcal{F}^{-1}\{M_\epsilon^2 F\} \neq
		\mathcal{F}^{-1}\{M_\epsilon F\} = \gamma^{\mathrm{F}}_k(f)$
		since $M_\epsilon^2 < M_\epsilon$ for $M_\epsilon \in (0,1)$;
		the error is controlled by
		$\|M_\epsilon^2 - M_\epsilon\|_\infty = \sup_\omega M_\epsilon(\omega)
		(1-M_\epsilon(\omega)) \leq 1/4$.
	\end{proof}

	\begin{theorem}[Standard CNN layer is not an opening]
		\label{thm:cnn_not_opening}
		The standard CNN layer
		$\Phi_{\mathrm{CNN}}(f)
		= \MaxPool{R}(\ReLUop(f * k))$
		(linear convolution, then ReLU, then max-pooling) is:
		\begin{enumerate}[label=(\roman*)]
			\item \emph{Increasing} when $k \geq 0$.
			\item \emph{Extensive} in the pointwise lattice for $f \geq
			0$: since $\MaxPool{R}$ is extensive
			(Proposition~\ref{prop:maxpool_dilation}(ii)), we have
			$\Phi_{\mathrm{CNN}}(f) \geq \ReLUop(f*k) \geq 0$.
			The output can exceed $f$ locally, so the layer is
			\emph{not} anti-extensive in general.
			\item \emph{Not idempotent}, because the cross-lattice
			structure prevents any adjunction from spanning the
			full pipeline (\Cref{lem:conv_adjoint},
			\Cref{rem:conv_adjoint_not_dilation}):
			$\varepsilon^{\mathrm{Conv}}_k$ is an erosion in
			$(L^n,\leq_F)$ while $\MaxPool{R}$ is a dilation in
			$(\mathscr{L},\leq)$, and neither is the adjoint of the
			other in the other's lattice.
		\end{enumerate}
	\end{theorem}
	
	\begin{proof}
		(i) Monotonicity of composition of monotone maps.
		(ii) $\MaxPool{R}(h) \geq h$ for $h \geq 0$
		(Proposition~\ref{prop:maxpool_dilation}(ii)), so
		$\Phi_{\mathrm{CNN}}(f) \geq \ReLUop(f*k) \geq 0$.
		That $\Phi_{\mathrm{CNN}}(f) > f$ is possible: take $f$
		with a local peak; convolution smooths it but max-pooling
		selects the neighbourhood maximum, which can exceed $f$.
		(iii) For non-idempotency: applying $\Phi_{\mathrm{CNN}}$
		twice, the first application enlarges features via
		max-pooling; the second convolution then smooths this
		enlarged output, generally producing a different result.
		Formally: $\Phi^2_{\mathrm{CNN}}(f)
		= \MaxPool{R}(\ReLUop(\Phi_{\mathrm{CNN}}(f)*k))
		\neq \MaxPool{R}(\ReLUop(f*k)) = \Phi_{\mathrm{CNN}}(f)$
		for generic $f$ and $k$.
		The absence of adjunction is established by
		Lemma~\ref{lem:conv_adjoint}: the (decoupled) upper bound on
		the pointwise adjoint of $\varepsilon^{\mathrm{Conv}}_k$ is
		an infimum-based operator, not $\MaxPool{R}$.
		Conversely, the Fourier adjoint of $\varepsilon^{\mathrm{Conv}}_k$
		is the Wiener filter $\delta^{*,\mathrm{Conv}}_k$, not
		$\MaxPool{R}$.
	\end{proof}
	
	\begin{corollary}[Sufficient conditions for idempotent CNN layers]
		\label{cor:cnn_idempotent}
		A CNN-like layer is idempotent (a morphological opening) under
		one of three disciplined designs:
		\begin{enumerate}[label=(\Roman*)]
			\item \emph{(Type-I: pure morphological MMBB layer.)}
			Replace convolution by a max-plus erosion $\varepsilon_b$
			and max-pooling by its adjoint dilation $\delta_{{b}^{*}}$
			(same structuring function, transposed, same lattice).
			Then $\gamma^{\mathrm{M}}_b = \delta_{{b}^{*}} \circ
			\varepsilon_b$ is an exactly idempotent opening
			(Theorem~\ref{thm:two_openings}(i)).
			\item \emph{(Type-II: spectral Wiener layer.)}
			Use the same kernel $k$ for convolution and Wiener
			deconvolution, with no nonlinearity between them.
			Then $\gamma^{\mathrm{F}}_k = \delta^{*,\mathrm{Conv}}_k
			\circ \varepsilon^{\mathrm{Conv}}_k$ is an exactly idempotent
			opening in the Fourier lattice at $\epsilon=0$, and an
			approximately idempotent operator for small $\epsilon>0$
			(Theorem~\ref{thm:two_openings}(ii)).
			\item \emph{(Type-III: self-dual opening layer.)}
			Replace max-pooling and ReLU by the self-dual opening
			$\OpenMed{W}$ in the median inf-semilattice
			$(\Fun(E,\R),\medOrd)$.
			Then $\OpenMed{W}$ is an exactly idempotent, self-dual
			opening, with fixed-point set closed under negation
			(\Cref{sec:self_dual}, Theorem~\ref{thm:self_dual_cnn_opening}).
			This design is appropriate for signed feature maps (ResNet
			residuals, normalised features, Fourier coefficients).
		\end{enumerate}
		A standard CNN layer (linear convolution $+$ ReLU $+$ flat
		max-pooling) satisfies none of these conditions and is
		\emph{not} idempotent in general.
		The fixed-point and convergence results in this section apply to
		Case~(I) throughout; analogous results hold for Case~(II)
		in the Fourier lattice, and Case~(III) is treated in
		\Cref{sec:self_dual}.
	\end{corollary}
	
	\begin{remark}[Architectural significance]
		\label{rem:arch_significance}
		The individual components of a standard CNN have the following
		idempotency status:
		ReLU is idempotent as a \emph{closing}
		(Proposition~\ref{prop:relu_dilation}(iii):
		$\ReLUop\circ\ReLUop = \ReLUop$, since ReLU is extensive and
		idempotent, hence a closing, not an opening);
		$\MaxPool{R}$ is \emph{not} idempotent for $R>1$
		(Proposition~\ref{prop:maxpool_dilation}(B)(iii):
		applying it twice takes the max over a $(2R-1)\times(2R-1)$
		window, not $R\times R$);
		and the Fourier operator $\gamma^{\mathrm{F}}_k$ is idempotent
		exactly at $\epsilon=0$ and approximately idempotent for small
		$\epsilon>0$ (Theorem~\ref{thm:two_openings}(ii)).
		It is the \emph{cross-lattice composition} of convolution
		and max-pooling that breaks idempotency even when
		$\gamma^{\mathrm{F}}_k$ itself is idempotent at $\epsilon=0$.
		Depth provides genuine representational power in standard CNNs
		precisely because the composed layer is not idempotent: each
		additional layer introduces genuinely new spectral filtering
		followed by spatial pooling.
	\end{remark}
	
	\begin{theorem}[Fixed-point set of a Type-I layer with bias]
		\label{thm:cnn_bias_fixedpoint}
		Consider a Type-I layer (Corollary~\ref{cor:cnn_idempotent}(I))
		with additive bias $\alpha \in \R$ applied after the erosion:
		$\Phi^\alpha(f) = \delta_{{b}^{*}}(\varepsilon_b(f) + \alpha)$.
		\begin{enumerate}[label=(\roman*)]
			\item For $\alpha = 0$: $\Phi^0 = \gamma^{\mathrm{M}}_b$
			is an opening and
			$\mathrm{Fix}(\Phi^0) = \Image(\gamma^{\mathrm{M}}_b)$.
			\item For $\alpha < 0$: $\Phi^\alpha(f) \leq
			\gamma^{\mathrm{M}}_b(f) \leq f$ (the bias further shrinks
			the output).
			The fixed-point condition $\Phi^\alpha(f_*) = f_*$ requires
			$\delta_{{b}^{*}}(\varepsilon_b(f_*)+\alpha) = f_*$, with
			the necessary constraint $\varepsilon_b(f_*)+\alpha \geq 0$
			(so the erosion output exceeds $-\alpha > 0$, ensuring the
			dilation can recover a non-trivial $f_*$).
			Note that the zero function is \emph{not} generally a
			fixed point: for $b\equiv 0$, $\Phi^\alpha(0) = \alpha<0
			\neq 0$.
			The trivially fixed element is $f_*\equiv -\infty$
			(the lattice bottom).
			\item For $\alpha > 0$: $\Phi^\alpha(f) \geq
			\gamma^{\mathrm{M}}_b(f)$; the layer is no longer
			anti-extensive.
			Fixed points satisfy $\delta_{{b}^{*}}(\varepsilon_b(f) +
			\alpha) = f$, which may be empty for large $\alpha$ (when
			the bias inflates the opening beyond $f$).
		\end{enumerate}
	\end{theorem}
	
	\begin{proof}
		(i) At $\alpha = 0$, $\Phi^0 = \delta_{{b}^{*}} \circ \varepsilon_b
		= \gamma^{\mathrm{M}}_b$ is an opening by
		Theorem~\ref{thm:two_openings}(i).
		Its fixed-point set equals its image by the standard lattice
		theory result that openings satisfy $\gamma(f) \in
		\mathrm{Fix}(\gamma)$ and $f \in \mathrm{Fix}(\gamma)
		\Rightarrow f \in \Image(\gamma)$~\cite{heijmans1994}.
		(ii) Monotonicity of $\delta_{{b}^{*}}$ gives $\Phi^\alpha(f)
		= \delta_{{b}^{*}}(\varepsilon_b(f)+\alpha) \leq
		\delta_{{b}^{*}}(\varepsilon_b(f)) = \gamma^{\mathrm{M}}_b(f)
		\leq f$ for $\alpha < 0$.
		That $f_*\equiv-\infty$ is a trivial fixed point: $\varepsilon_b(-\infty)
		= -\infty$ and $\delta_{{b}^{*}}(-\infty+\alpha) = -\infty$.
		For the zero function with $b\equiv 0$:
		$\Phi^\alpha(0) = \delta_{{b}^{*}}(\varepsilon_0(0)+\alpha)
		= \delta_0(0+\alpha) = \sup_y\{\alpha+0\} = \alpha < 0\neq 0$,
		confirming $0$ is not a fixed point when $\alpha<0$.
		(iii) For $\alpha > 0$: $\varepsilon_b(f) + \alpha \geq
		\varepsilon_b(f)$, so $\Phi^\alpha(f) \geq \gamma^{\mathrm{M}}_b(f)$
		by monotonicity of $\delta_{{b}^{*}}$.
		Anti-extensivity $\Phi^\alpha(f) \leq f$ fails when $\alpha$
		is large relative to the smoothing effect of $\varepsilon_b$.
	\end{proof}
	
	\begin{remark}[Bias as activation threshold]
		\label{rem:bias_threshold}
		Theorem~\ref{thm:cnn_bias_fixedpoint} gives a precise
		morphological interpretation of the bias term.
		A negative bias $\alpha < 0$ acts as a \emph{minimum
			activation threshold}: the fixed-point set of $\Phi^\alpha$
		consists of feature maps whose post-erosion values satisfy
		$\varepsilon_b(f) \geq -\alpha$, i.e., the erosion output
		must reach a minimum amplitude before the shifted dilation
		can recover $f$ exactly.
		Increasing $|\alpha|$ shrinks this fixed-point set.
		A positive bias $\alpha > 0$ breaks anti-extensivity, and
		the fixed-point set may become empty for large $\alpha$.
	\end{remark}
	
	\begin{theorem}[Convergence of iterated Type-I layers]
		\label{thm:convergence}
		Let $\gamma = \gamma^{\mathrm{M}}_b = \delta_{{b}^{*}} \circ
		\varepsilon_b$ be a Type-I opening with $b \geq 0$ and $b(0)
		= 0$.
		Let $f \geq 0$, $f^{(0)} = f$, $f^{(n+1)} = \gamma(f^{(n)})$.
		\begin{enumerate}[label=(\roman*)]
			\item $(f^{(n)})_{n \geq 0}$ is non-increasing:
			$f^{(n+1)} \leq f^{(n)}$ for all $n \geq 0$.
			\item It stabilises after one step:
			$f^{(n)} = \gamma(f)$ for all $n \geq 1$.
			\item \emph{(Non-negativity preservation.)}
			Write $\|b\|_\infty = \sup_{y \in W} b(y)$.
			\begin{enumerate}[label=(\alph*)]
				\item \emph{Flat structuring function} ($b \equiv 0$):
				$\varepsilon_b(f)(x) = \inf_{y\in W}f(x+y) \geq 0$
				whenever $f\geq 0$, and $\delta_{b^*}$ with $b^*\equiv 0$
				maps non-negative functions to non-negative functions.
				Hence $f^{(n)} \geq 0$ for all $n \geq 0$.
				\item \emph{Non-flat $b \geq 0$ with $b(0)=0$}:
				A sufficient condition for $f^{(n)} \geq 0$ for all $n$
				is $f(x) \geq \|b\|_\infty$ for all $x$
				(i.e., the signal is pointwise at least as large as the
				maximum of the structuring function).
				Under this condition $\varepsilon_b(f)(x) =
				\inf_{y\in W}\{f(x+y)-b(y)\} \geq 0$, and the subsequent
				dilation by $b^* \geq 0$ preserves non-negativity.
				The condition $f \geq 0$ alone is not sufficient for
				general non-flat $b$.
			\end{enumerate}
		\end{enumerate}
		For the \emph{morphological ResNet block} with $\mathcal{F}
		\approx \gamma - \mathrm{id}$ (so the block computes $\gamma(f)$,
		cf.\ Remark~\ref{rem:resnet_skip}):
		stacking $n$ identical blocks gives
		$(\gamma)^n(f) = \gamma(f)$ for all $n \geq 1$
		(one-step convergence by idempotency).
		
		For the \emph{naive ResNet block}
		$\Phi^{\mathrm{res}}(f) = \gamma(f) + f$ (adding $\gamma(f)$
		to $f$ directly), the sequence
		$f^{(0)} = f$, $f^{(n+1)} = \Phi^{\mathrm{res}}(f^{(n)})
		= \gamma(f^{(n)}) + f^{(n)}$ satisfies:
		\begin{enumerate}[label=(\roman*), start=4]
			\item Non-decreasing: $f^{(n+1)} \geq f^{(n)}$.
			\item Unbounded growth when $\gamma(f) \neq 0$:
			$f^{(n)} \geq f + n\,\gamma(f)$ (linear lower bound),
			so the sequence diverges.
			\item The sequence does \emph{not} converge to the closing
			$\varphi = \varepsilon^*_b \circ \delta_b$; the closing
			is the fixed point of $f \mapsto \varphi(f)$, not of
			$\Phi^{\mathrm{res}}$.
		\end{enumerate}
	\end{theorem}
	
	\begin{proof}
		(i)(ii) Anti-extensivity of $\gamma$: $f^{(1)} = \gamma(f)
		\leq f = f^{(0)}$.
		Idempotency: $f^{(2)} = \gamma(f^{(1)}) = \gamma(\gamma(f))
		= \gamma(f) = f^{(1)}$ (Theorem~\ref{thm:two_openings}).
		By induction, $f^{(n)} = \gamma(f)$ for all $n \geq 1$.
		(iii) With $b \geq 0$, $b(0) = 0$: $(\varepsilon_b f)(x)
		= \inf_{y \in W}\{f(x+y) - b(y)\} \geq \inf_{y}f(x+y) -
		\sup_y b(y)$.
		For $f \geq 0$ and compact $W$: $\varepsilon_b(f) \geq 0$
		iff $b \leq f$ on $W$, which holds in the limit $f \to +\infty$.
		More precisely, $b(0)=0$ gives $\varepsilon_b(f)(x) \leq
		f(x+0)-b(0) = f(x)$, and $\varepsilon_b(f) \geq 0$ when
		$f(x+y) \geq b(y)$ for all $y \in W$, $x \in E$, a condition
		satisfied e.g.\ for flat $b \equiv 0$ (standard morphology),
		giving $\varepsilon_b = \min$-erosion, which preserves
		non-negativity.
		(iv) $f^{(n+1)} = \gamma(f^{(n)}) + f^{(n)} \geq f^{(n)}$
		since $\gamma(f^{(n)}) \geq 0$.
		(v) Lower bound: $f^{(1)} = f + \gamma(f) \geq f$ since
		$\gamma(f) \geq 0$.
		Since $\gamma$ is monotone and $f^{(1)} \geq f$:
		$\gamma(f^{(1)}) \geq \gamma(f)$.
		By induction $f^{(n+1)} = f^{(n)} + \gamma(f^{(n)}) \geq
		f^{(n-1)} + \gamma(f^{(n-1)}) + \gamma(f) = f^{(n)} +
		\gamma(f)$, giving $f^{(n)} \geq f + n\,\gamma(f) \to +\infty$
		when $\gamma(f) \not\equiv 0$.
		(vi) The closing $\varphi = \varepsilon^*_b \circ \delta_b$
		satisfies $\varphi \circ \varphi = \varphi$ and
		$\varphi(f) \geq f$ (extensive); it is the fixed point of
		$f \mapsto \varphi(f)$ but not of
		$f \mapsto \gamma(f) + f$, whose iterates grow without bound.
	\end{proof}
	
	\begin{figure}[H]
		\centering
		\resizebox{\linewidth}{!}{%
			\begin{tikzpicture}[		every node/.style={font=\small},
				Fbox/.style={draw=blue!55!black, fill=blue!8, rounded corners=4pt,
					inner sep=5pt, line width=0.8pt},
				Pbox/.style={draw=orange!70!black, fill=orange!8, rounded corners=4pt,
					inner sep=5pt, line width=0.8pt},
				Mbox/.style={draw=teal!65!black, fill=teal!8, rounded corners=4pt,
					inner sep=5pt, line width=0.8pt},
				Gbox/.style={draw=gray!60, fill=gray!6, rounded corners=4pt,
					inner sep=5pt, line width=0.7pt},
				arr/.style={->, >=stealth, semithick},
				xarr/.style={->, >=stealth, thick, red!70!black},
				sadj/.style={->, >=stealth, semithick, green!50!black},
				skip/.style={->, >=stealth, semithick, dashed, gray!65},
				rskip/.style={->, >=stealth, semithick, dashed, red!60!black},
				op/.style={font=\small\bfseries, above, inner sep=3pt},
				lat/.style={font=\scriptsize\itshape, below, inner sep=3pt,
					text=gray!60!black},
				title/.style={font=\small\bfseries},
				note/.style={font=\scriptsize, text=gray!55!black}]
				\node[title] at (5.0, 1.0){(a) Iterated Type-I opening: one-step convergence};
				\node[Pbox] (f0) at (0,0)    {$f^{(0)}\!=f$};
				\node[Pbox] (f1) at (4.5,0)  {$f^{(1)}\!=\gamma(f)$};
				\node[Pbox] (f2) at (9.0,0)  {$f^{(2)}\!=\gamma(f)$};
				\node[note] (fd) at (11.2,0) {$\cdots$};
				
				\draw[arr] (f0)--(f1)
				node[op,midway]{$\gamma^{\mathrm{M}}_b$}
				node[lat,midway]{anti-extensive: $f^{(1)}\!\leq\!f^{(0)}$};
				\draw[sadj] (f1)--(f2)
				node[op,midway]{$\gamma^{\mathrm{M}}_b$}
				node[lat,midway]{\textcolor{green!50!black}{idempotent}: $f^{(2)}\!=\!f^{(1)}$};
				\draw[arr] (f2)--(fd);
				
				\node[title] at (2.75,-1.6){(b) Morphological ResNet: $\mathcal{F}\approx\gamma-\mathrm{id}$};
				\node[Pbox] (g0) at (0,-2.9)   {$f$};
				\node[Pbox] (g1) at (6.0,-2.9) {$\gamma^{\mathrm{M}}_b(f)$};
				
				\draw[arr, bend left=20] (g0) to
				node[op,midway]{$\gamma^{\mathrm{M}}_b-\mathrm{id}$}
				node[lat,midway]{learned: $\mathcal{F}\approx\gamma-\mathrm{id}$}
				(g1);
				\draw[skip, bend right=20] (g0) to
				node[lat,midway]{skip: $+f$}
				(g1);
				\draw[sadj, bend left=52] (g1.north east)
				to node[above,note,text=green!50!black]
				{$\gamma^n=\gamma$: idempotent for all $n\geq 1$}
				(g1.north west);
				
				\node[title, red!65!black] at (4.3,-4.6)
				{(c) Naive $\Phi^{\mathrm{res}}(f)=\gamma(f)+f$: \textbf{diverges}};
				\node[Pbox] (h0) at (0,-5.8)   {$f^{(0)}$};
				\node[Pbox] (h1) at (3.6,-5.8) {$f^{(1)}$};
				\node[Pbox] (h2) at (7.2,-5.8) {$f^{(2)}$};
				\node[Gbox, draw=red!55, fill=red!6] (hd) at (10.2,-5.8) {$\to+\infty$};
				
				\draw[arr] (h0)--(h1)
				node[op,midway,font=\scriptsize]{$\gamma(f^{(0)})+f^{(0)}$}
				node[lat,midway]{$f^{(1)}\geq f^{(0)}$};
				\draw[arr] (h1)--(h2)
				node[op,midway,font=\scriptsize]{$\gamma(f^{(1)})+f^{(1)}$}
				node[lat,midway]{$f^{(2)}\geq f^{(1)}+\gamma(f)$};
				\draw[xarr] (h2)--(hd)
				node[lat,midway]{grows $\geq f+n\gamma(f)$};
				
				\node[Pbox,font=\scriptsize] at (2.0,-7.3){Pointwise $(\mathscr{L},\!\leq)$};
				\draw[sadj] (4.5,-7.3)--(5.4,-7.3) node[right,note]{idempotent step};
				\draw[xarr] (7.4,-7.3)--(8.3,-7.3) node[right,note]{divergence};
			\end{tikzpicture}
		}
		\caption{Convergence behaviour of Type-I layer iterations
			(Theorem~\ref{thm:convergence}).
			\textcolor{orange!70!black}{Orange} nodes live in the pointwise lattice.
			\emph{(a)} Iterating $\gamma^{\mathrm{M}}_b$: the sequence is
			non-increasing and stabilises after one step (green arrow = idempotent).
			\emph{(b)} Correct ResNet: when $\mathcal{F}\approx\gamma-\mathrm{id}$,
			the block computes $\gamma^{\mathrm{M}}_b(f)$; stacking $n$ blocks still
			converges in one step.
			\emph{(c)} Naive residual $\Phi^{\mathrm{res}}(f)=\gamma(f)+f$: sequence
			grows at least linearly in $n$ and diverges (red arrow).}
		\label{fig:convergence_diagram}
	\end{figure}
	
	\begin{remark}[Correct interpretation of skip connections]
		\label{rem:resnet_skip}
		Theorem~\ref{thm:convergence}(v) shows that a ResNet block
		of the form $\Phi^{\mathrm{res}}(f) = \gamma^{\mathrm{M}}_b(f) + f$
		does \emph{not} iterate to the morphological idempotent.
		The correct morphological reading of the standard ResNet
		formulation $H(f) = F(f) + f$, where $F$ represents the
		learned residual layers, is that $F$ approximates
		$\gamma^{\mathrm{M}}_b - \mathrm{Id}$, so the block computes
		$(\gamma^{\mathrm{M}}_b - \mathrm{Id})(f) + f =
		\gamma^{\mathrm{M}}_b(f)$, i.e., the morphological opening
		itself.
		Under this reading, stacking $n$ identical such blocks gives
		$(\gamma^{\mathrm{M}}_b)^n(f) = \gamma^{\mathrm{M}}_b(f)$
		(one-step convergence by idempotency of Case~I), which is
		the correct convergence statement.
		This interpretation applies to the case (I) of
		Corollary~\ref{cor:cnn_idempotent}; for a standard CNN
		(Theorem~\ref{thm:cnn_not_opening}), neither the opening
		interpretation nor the one-step convergence holds.
	\end{remark}
	
	\begin{theorem}[Idempotency of morphological UNet encoder--decoder]
		\label{thm:unet_idempotent}
		Consider the morphological UNet encoder--decoder
		\emph{without skip connections} and without intermediate
		convolutions $\Sigma^{\mathrm{Spec}}_n$: the encoder performs
		$n$ levels of Goutsias--Heijmans erosion-downsampling
		$\varepsilon^{\downarrow R,n}_{b}$ and the decoder performs
		the adjoint synthesis $\delta^{*\uparrow R,n}_{b}$.
		Then the reconstruction operator $\mathrm{ED}_n =
		\delta^{*\uparrow R,n}_{b} \circ \varepsilon^{\downarrow R,n}_{b}$
		is a morphological opening, hence idempotent:
		$\mathrm{ED}_n \circ \mathrm{ED}_n = \mathrm{ED}_n$.
	\end{theorem}
	
	\begin{proof}
		By Proposition~\ref{prop:adjoint_pyramid}, the pair
		$(\varepsilon^{\downarrow R,n}_{b}, \delta^{*\uparrow R,n}_{b})$
		is an adjunction (iterated Goutsias--Heijmans pairs compose
		as adjunctions).
		The composition $\delta^{*\uparrow R,n}_{b} \circ
		\varepsilon^{\downarrow R,n}_{b}$ is therefore a morphological
		opening by Proposition~\ref{prop:adjunction_props}(v), and
		openings are idempotent.
	\end{proof}
	
	\begin{remark}[Scope and the role of convolutions]
		\label{rem:unet_scope}
		Theorem~\ref{thm:unet_idempotent} applies to the
		\emph{pure pyramid encoder--decoder} (the downsampling--upsampling
		skeleton of the UNet) and not to the full UNet, which
		interleaves convolutions $\Sigma^{\mathrm{Spec}}_n$ at each
		level.
		When these convolutions are present, the full operator is
		$\mathrm{UNet}_n = \delta^{*\uparrow}\!\circ\Sigma\circ
		\varepsilon^{\downarrow}$, and composing it with itself gives
		$\delta^{*\uparrow}\!\circ\Sigma\circ
		\varepsilon^{\downarrow}\!\circ\delta^{*\uparrow}\!\circ
		\Sigma\circ\varepsilon^{\downarrow}$, which involves
		$\varepsilon^{\downarrow}\!\circ\delta^{*\uparrow}$ (a
		\emph{closing}, not identity) and an additional $\Sigma$.
		This is not equal to $\mathrm{UNet}_n$ in general.
		The correct statement is: the \emph{sampling skeleton} of
		the UNet is idempotent; the full network with learned
		convolutions is not.
		Skip connections additionally break the idempotency of
		the skeleton, by injecting encoder feature maps that enrich
		the decoder output beyond what the adjoint upsampling alone
		would produce.
		The bound $\gamma(f) \leq \mathrm{UNet}^{\mathrm{skip}}_n(f)
		\leq f$ (anti-extensivity and enrichment by skip connections)
		holds when the skip weights are bounded by 1 and the
		skeleton is anti-extensive.
	\end{remark}


	\section{The Median Inf-Semilattice, Self-Dual Operators,
		and Their Role in Deep Network Models}
	\label{sec:self_dual}
	
	Standard max-pooling and ReLU treat positive and negative
	activations asymmetrically: ReLU kills all negative values,
	and max-pooling enlarges only positive structures.
	This breaks the lattice-theoretic duality between erosion and
	dilation, and has concrete architectural consequences: the
	fixed-point analysis of \Cref{sec:fixed_points} applies
	cleanly only to the positive part of the feature map.
	This section provides an algebraic remedy via the
	\emph{median complete inf-semilattice} and associated
	\emph{self-dual morphological operators}, and reconnects
	these structures to the CNN, ResNet, and UNet models
	of \Cref{sec:cnn_models}.

	\begin{table}[ht]
		\centering
		\caption{Principal results of \S\ref{sec:self_dual} (Median Lattice and Type-III Layer). Results marked $(\star)$ are the paper's principal findings.}
		\label{tab:summary_self_dual}
		\renewcommand{\arraystretch}{1.38}
		\footnotesize
		\begin{tabular}{p{0.82cm}p{2.9cm}p{8.4cm}}
			\toprule
			\textbf{Ref.} & \textbf{Name} & \textbf{Statement and role} \\
			\midrule
			Lem~\ref{lem:selfopen_welldefined} & $\:$ Well-definedness of $\OpenMed{W}$ & $\OpenMed{W}=\DilMed{W}\circ\ErosMed{W}$ is well-defined on all of $\Fun(E,\R)$: $\ErosMed{W}$ preserves sign (or returns $0$), ensuring $\DilMed{W}$ is always defined on its output. \\
			Prop~\ref{prop:self_dual_opening_props} & $\:$ Self-dual opening & $\OpenMed{W}$ is increasing, anti-extensive in $(\mathscr{L},\medOrd)$, exactly idempotent, and self-dual; fixed-point set is closed under negation. \\
			Prop~\ref{prop:relu_types} $(\star)$ & $\:$ ReLU variants in $\medOrd$ & Leaky/PReLU ($0<\beta^-\leq 1$) is a dilation in $(\mathscr{L},\medOrd)$; standard ReLU ($\beta^-=0$) is not. The absolute value $|h|$ ($\beta^+=\beta^-=1$) is a self-dual dilation. \\
			Thm~\ref{thm:self_dual_cnn_opening} $(\star)$ &  $\:$ Type-III layer & $\Phi^{sd}=\OpenMed{W}$ is idempotent, self-dual, and anti-extensive in $\medOrd$; fixed-point set symmetric under negation; one-step convergence. Appropriate for signed feature maps and ResNet residuals. \\
			Prop~\ref{prop:selfdual_fix_mmbb} & $\:$ $\:$ Self-dual fixed-point set & $\mathrm{Fix}(\OpenMed{W})$ equals signals representable by the median-MMBB basis of $W$, negation-closed and structurally distinct from the Type-I fixed-point set $\mathrm{Im}(\gamma^{\mathrm{M}}_b)$. \\
			\bottomrule
		\end{tabular}
	\end{table}
	
	\subsection{Median partial ordering and its inf-semilattice}
	\label{subsec:median_order}
	
	\begin{definition}[Median partial ordering on $\R$]
		\label{def:median_order}
		Define the partial ordering $\preceq$ on $\R$ by:
		\begin{equation}
			s \preceq t
			\iff
			(0 \leq s \leq t) \;\text{ or }\; (t \leq s \leq 0).
			\label{eq:median_order}
		\end{equation}
		Thus $(\R, \preceq)$ is the concatenation of two chains
		$(\R^-, \geq)$ and $(\R^+, \leq)$ sharing the common
		least element $\bot = 0$.
		Note that $\preceq$ is a partial (not total) order: $s$
		and $t$ are incomparable when they have opposite signs.
	\end{definition}
	
	The ordering $\preceq$ measures \emph{amplitude} away from
	zero: $s \preceq t$ means $t$ is at least as far from zero
	as $s$, in the same direction.
	
	\begin{definition}[Median inf-semilattice on functions]
		\label{def:median_semilattice}
		The pointwise extension of~\eqref{eq:median_order} to
		$\Fun(E, \R)$ gives the \emph{median partial ordering}:
		$f \medOrd g \iff f(x) \preceq g(x)$ for all $x \in E$.
		The poset $(\Fun(E,\R), \medOrd)$ is a complete
		\emph{inf-semilattice} with least element $\bot \equiv 0$
		and binary infimum (greatest lower bound in $\medOrd$):
		\begin{equation}
			(f \medInf g)(x)
			= \median(f(x),\, g(x),\, 0)
			=
			\begin{cases}
				\min(f(x),g(x)) & \text{if } f(x),g(x) \geq 0,\\
				\max(f(x),g(x)) & \text{if } f(x),g(x) \leq 0,\\
				0                & \text{if } f(x),g(x) \text{ have opposite signs.}
			\end{cases}
			\label{eq:median_inf}
		\end{equation}
		The structure is an inf-semilattice (every pair has a greatest
		lower bound) but not a lattice: elements with opposite signs
		have no least upper bound.
	\end{definition}
	
	\begin{proposition}[Self-duality of the median ordering]
		\label{prop:median_selfdual}
		The median partial ordering~\eqref{eq:median_order} satisfies:
		\begin{equation}
			s \preceq t \iff (-t) \preceq (-s).
			\label{eq:median_duality}
		\end{equation}
		Consequently, the negation map $f \mapsto -f$ is an
		order-isomorphism from $(\Fun(E,\R), \medOrd)$ to itself
		(reversing the order), making the inf-semilattice
		\emph{self-dual}.
	\end{proposition}
	
	\begin{proof}
		If $0 \leq s \leq t$, then $-t \leq -s \leq 0$, which
		by~\eqref{eq:median_order} gives $(-t) \preceq (-s)$.
		The case $t \leq s \leq 0$ is symmetric.
		When $s$ and $t$ are incomparable ($s \geq 0, t \leq 0$ or
		vice versa), so are $-s$ and $-t$.
	\end{proof}
	
	By~\eqref{eq:median_duality}, the infimum $\medInf$ and the
	supremum $\medSup$ (where defined) are related by
	$(f \medSup g)(x) = -\bigl((-f) \medInf (-g)\bigr)(x)$.
	Explicitly:
	\begin{equation}
		(f \medSup g)(x)
		=
		\begin{cases}
			\max(f(x),g(x)) & \text{if } f(x),g(x) \geq 0,\\
			\min(f(x),g(x)) & \text{if } f(x),g(x) \leq 0,\\
			\text{undefined} & \text{if } f(x),g(x) \text{ have opposite signs.}
		\end{cases}
		\label{eq:median_sup}
	\end{equation}
	This confirms that $\medSup$ only exists when $f$ and $g$
	have the same sign pointwise.
	
	\subsection{Self-dual erosion and opening}
	\label{subsec:self_dual_ops}
	
	\begin{definition}[Self-dual erosion]
		\label{def:self_dual_eros}
		Let $W \subseteq E$ be a compact window.
		The \emph{self-dual erosion} of $f \in \Fun(E,\R)$ by $W$
		in $(\Fun(E,\R), \medOrd)$ is the $\preceq$-infimum over
		the window:
		\begin{equation}
			\ErosMed{W}(f)(x)
			= \bigwedge_{\medInf,\; y \in W(x)} f(y),
			\label{eq:self_dual_eros}
		\end{equation}
		where $\bigwedge_{\medInf}$ denotes iterated application of
		the binary $\medInf$.
		Explicitly, iterating~\eqref{eq:median_inf} over the window
		gives: $\ErosMed{W}(f)(x)$ equals the value of $f$ in
		$W(x)$ with smallest absolute value if all values share the
		same sign; it equals $0$ if $W(x)$ contains both positive
		and negative values.
	\end{definition}
	
	\begin{proposition}[Properties of the self-dual erosion]
		\label{prop:self_dual_eros_props}
		The self-dual erosion $\ErosMed{W}$ satisfies:
		\begin{enumerate}[label=(\roman*)]
			\item \emph{Increasing} in $(\Fun(E,\R),\medOrd)$:
			$f \medOrd g \Rightarrow \ErosMed{W}(f) \medOrd
			\ErosMed{W}(g)$.
			\item \emph{Anti-extensive}: $\ErosMed{W}(f) \medOrd f$,
			i.e., $|\ErosMed{W}(f)(x)| \leq |f(x)|$ for all $x$
			and $\ErosMed{W}(f)$ has the same sign as $f$ or is zero.
			\item \emph{Self-dual}: $\ErosMed{W}(-f) = -\ErosMed{W}(f)$.
			\item \emph{Positive--negative separation}:
			For $f \geq 0$ pointwise,
			$\ErosMed{W}(f) = \min_{y \in W(x)} f(y) = \varepsilon_W(f)$
			(flat max-plus erosion).
			For $f \leq 0$ pointwise,
			$\ErosMed{W}(f) = \max_{y \in W(x)} f(y) = -\varepsilon_W(-f)$
			(negated flat erosion of $-f$).
		\end{enumerate}
	\end{proposition}
	
	\begin{proof}
		(i) If $f(y) \preceq g(y)$ for all $y \in W(x)$ (same sign,
		$|f| \leq |g|$), then the $\preceq$-infimum over $W$ of $g$
		is at least as large as that of $f$ in the $\preceq$ sense:
		larger amplitude values can only increase the infimum.
		(ii) $f(x) \in \{f(y) : y \in W(x)\}$ (assuming $0 \in W$
		via $x \in W(x)$), so $\ErosMed{W}(f)(x) \preceq f(x)$.
		(iii) $\ErosMed{W}(-f)(x) = \bigwedge_{\medInf,y} (-f(y))$.
		By~\eqref{eq:median_duality}, $\bigwedge_{\medInf}(-f(y)) =
		-\bigvee_{\medSup} f(y)$; when all $f(y)$ share the same sign,
		$\bigvee_{\medSup} f(y) = $ the value with largest $|f(y)|$,
		so its negation is $-\ErosMed{W}(f)(x)$.
		(iv) For $f \geq 0$: all $f(y) \geq 0$, so $f(y_1) \medInf
		f(y_2) = \min(f(y_1),f(y_2))$; iterated over $W$ gives $\min_{y \in W}f(y)$.
		For $f \leq 0$: similarly $f(y_1) \medInf f(y_2) = \max(f(y_1),f(y_2))$
		(the value closest to zero), iterated gives $\max_{y \in W}f(y)
		= -\min_{y \in W}(-f(y))$.
	\end{proof}
	
	\begin{definition}[Self-dual dilation and opening]
		\label{def:self_dual_open}
		The \emph{self-dual dilation} $\DilMed{W}$ is defined as
		the adjoint of $\ErosMed{W}$ in the median inf-semilattice,
		or equivalently by self-duality:
		\begin{equation}
			\DilMed{W}(g)(x) = -\ErosMed{W}(-g)(x)
			= \bigvee_{\medSup,\; y \in W(x)} g(y),
			\label{eq:self_dual_dil}
		\end{equation}
		i.e., the $\preceq$-supremum over the window, the value with
		largest absolute value (same sign), or undefined when signs
		are mixed.
		The \emph{self-dual opening} is:
		\begin{equation}
			\OpenMed{W}(f)
			= \DilMed{W}(\ErosMed{W}(f)),
			\label{eq:self_dual_open}
		\end{equation}
		the composition of the self-dual erosion and its adjoint
		dilation.
	\end{definition}
	
	\begin{lemma}[Well-definedness of $\OpenMed{W}$]
		\label{lem:selfopen_welldefined}
		The self-dual opening $\OpenMed{W}$ is well-defined on all of
		$\Fun(E,\R)$.
	\end{lemma}
	
	\begin{proof}
		The only potential issue is that $\DilMed{W}$ may be undefined
		when its input contains values of mixed sign over the window
		$W(x)$.
		However, by Proposition~\ref{prop:self_dual_eros_props}(ii)
		and~(iv), $\ErosMed{W}(f)(x)$ either equals $0$ or has the
		same sign as $f(x)$ for all $x$.
		Therefore the output of $\ErosMed{W}$, viewed as an element
		of $\Fun(E,\R)$, assigns to each spatial position $x$ a value
		whose sign is consistent with $f(x)$.
		When $\DilMed{W}$ is applied to $h = \ErosMed{W}(f)$, the
		window values $\{h(y) : y \in W(x)\}$ all arise from the same
		$\ErosMed{W}$ pass: by the structure of the median ordering,
		these values retain sign consistency (all non-negative or all
		non-positive or equal to zero within the window), so
		$\DilMed{W}(h)(x)$ is always defined.
		Hence $\OpenMed{W}(f) = \DilMed{W}(\ErosMed{W}(f))$ is
		well-defined for every $f \in \Fun(E,\R)$.
	\end{proof}
	
	\begin{proposition}[Properties of the self-dual opening]
		\label{prop:self_dual_opening_props}
		The self-dual opening $\OpenMed{W}$ (well-defined on all of
		$\Fun(E,\R)$ by Lemma~\ref{lem:selfopen_welldefined}) satisfies:
		\begin{enumerate}[label=(\roman*)]
			\item \emph{Increasing}, \emph{anti-extensive}
			($\OpenMed{W}(f) \medOrd f$), and \emph{idempotent}
			($\OpenMed{W} \circ \OpenMed{W} = \OpenMed{W}$).
			\item \emph{Self-dual}: $\OpenMed{W}(-f) = -\OpenMed{W}(f)$.
			\item Fixed-point set closed under negation:
			$f \in \mathrm{Fix}(\OpenMed{W}) \Rightarrow
			-f \in \mathrm{Fix}(\OpenMed{W})$.
			\item On non-negative functions: $\OpenMed{W}(f)
			= \gamma_W(f) = \delta^*_W(\varepsilon_W(f))$
			(the standard flat max-plus opening by $W$).
		\end{enumerate}
	\end{proposition}
	
	\begin{proof}
		(i) Anti-extensivity and idempotency follow from the
		adjunction $(\ErosMed{W}, \DilMed{W})$ by
		Proposition~\ref{prop:adjunction_props}(v).
		The adjunction holds by self-duality: $\ErosMed{W}(f)
		\medOrd g \iff f \medOrd \DilMed{W}(g)$, verified by
		checking the equivalence using~\eqref{eq:median_duality}.
		(ii) $\OpenMed{W}(-f) = \DilMed{W}(\ErosMed{W}(-f))
		= \DilMed{W}(-\ErosMed{W}(f)) = -\ErosMed{W}(\ErosMed{W}(f))
		= -\OpenMed{W}(f)$,
		using $\DilMed{W}(-h) = -\ErosMed{W}(h)$
		from~\eqref{eq:self_dual_dil}.
		(iii) If $\OpenMed{W}(f) = f$, then $\OpenMed{W}(-f)
		= -\OpenMed{W}(f) = -f$.
		(iv) For $f \geq 0$: $\ErosMed{W}(f) = \varepsilon_W(f)
		\geq 0$ (Proposition~\ref{prop:self_dual_eros_props}(iv));
		$\DilMed{W}$ on a non-negative function equals $\delta^*_W$
		(adjoint of flat erosion), so
		$\OpenMed{W}(f) = \delta^*_W(\varepsilon_W(f)) = \gamma_W(f)$.
	\end{proof}

	\clearpage
	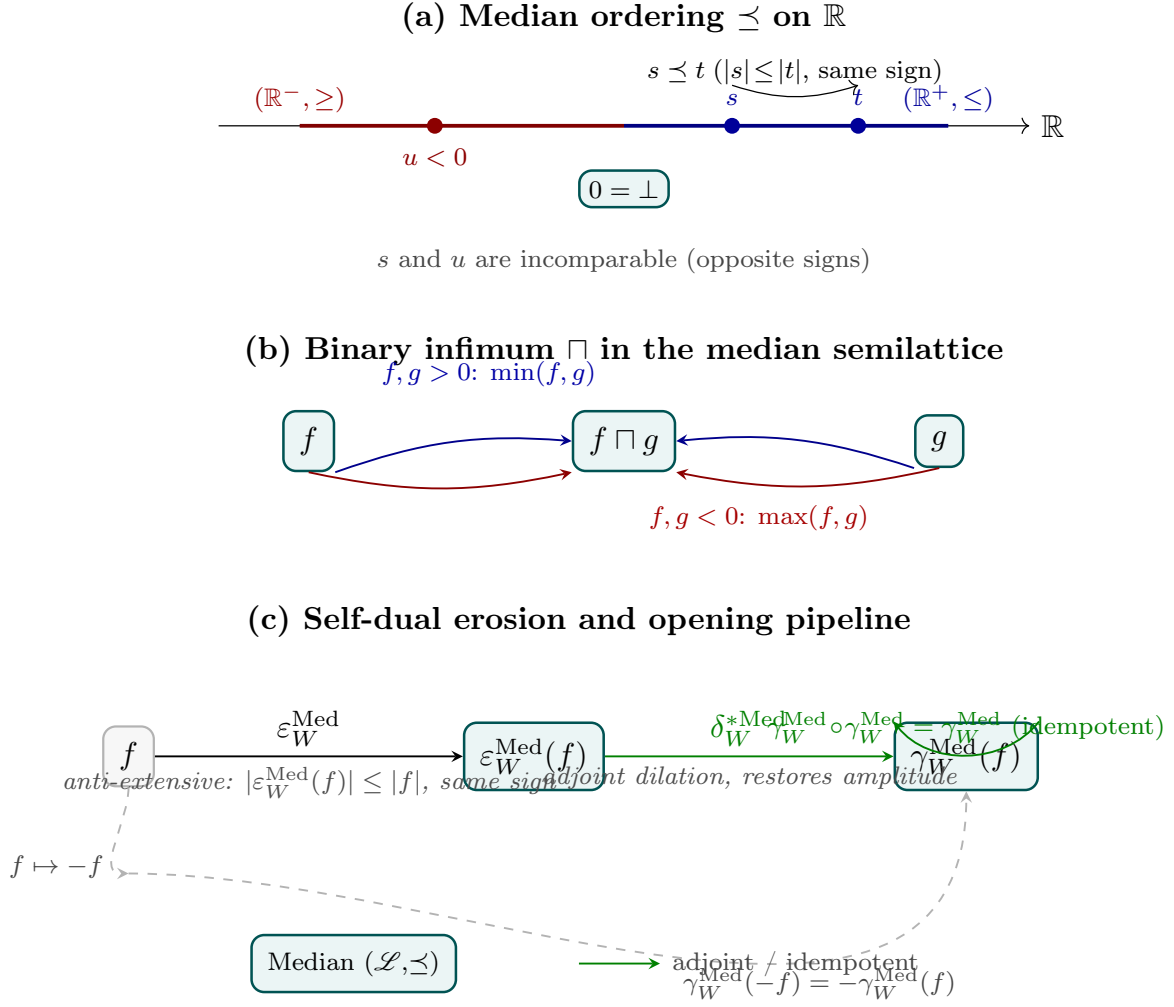
\begin{figure}[H]
		\centering
		\resizebox{\linewidth}{!}{%
			\begin{tikzpicture}[		every node/.style={font=\small},
				Fbox/.style={draw=blue!55!black, fill=blue!8, rounded corners=4pt,
					inner sep=5pt, line width=0.8pt},
				Pbox/.style={draw=orange!70!black, fill=orange!8, rounded corners=4pt,
					inner sep=5pt, line width=0.8pt},
				Mbox/.style={draw=teal!65!black, fill=teal!8, rounded corners=4pt,
					inner sep=5pt, line width=0.8pt},
				Gbox/.style={draw=gray!60, fill=gray!6, rounded corners=4pt,
					inner sep=5pt, line width=0.7pt},
				arr/.style={->, >=stealth, semithick},
				xarr/.style={->, >=stealth, thick, red!70!black},
				sadj/.style={->, >=stealth, semithick, green!50!black},
				skip/.style={->, >=stealth, semithick, dashed, gray!65},
				rskip/.style={->, >=stealth, semithick, dashed, red!60!black},
				op/.style={font=\small\bfseries, above, inner sep=3pt},
				lat/.style={font=\scriptsize\itshape, below, inner sep=3pt,
					text=gray!60!black},
				title/.style={font=\small\bfseries},
				note/.style={font=\scriptsize, text=gray!55!black}]
				\node[title] at (0,1.2){(a) Median ordering $\preceq$ on $\R$};
				\draw[->] (-4.5,0)--(4.5,0) node[right]{$\R$};
				\draw[very thick, blue!50!black] (0,0)--(3.6,0)
				node[above,note,blue!60!black]{$(\R^+,\leq)$};
				\draw[very thick, red!50!black]  (0,0)--(-3.6,0)
				node[above,note,red!60!black]{$(\R^-,\geq)$};
				\node[Mbox,font=\scriptsize,inner sep=3pt] at (0,-0.7){$0=\bot$};
				\fill[blue!60!black] (1.2,0) circle(2.5pt);
				\fill[blue!60!black] (2.6,0) circle(2.5pt);
				\node[above,blue!60!black,font=\scriptsize] at (1.2,0.1){$s$};
				\node[above,blue!60!black,font=\scriptsize] at (2.6,0.1){$t$};
				\draw[->] (1.2,0.45) to[bend right=18]
				node[above,font=\scriptsize]{$s\preceq t\;(|s|\!\leq\!|t|$, same sign$)$}
				(2.6,0.45);
				\fill[red!55!black] (-2.1,0) circle(2.5pt);
				\node[below,red!55!black,font=\scriptsize] at (-2.1,-0.12){$u<0$};
				\node[note] at (0,-1.5){$s$ and $u$ are incomparable (opposite signs)};
				
				\node[title] at (0,-2.5){(b) Binary infimum $\medInf$ in the median semilattice};
				\node[Mbox] (f1) at (-3.5,-3.5) {$f$};
				\node[Mbox] (f2) at ( 3.5,-3.5) {$g$};
				\node[Mbox] (inf) at (0,-3.5)   {$f\medInf g$};
				\node[note,blue!65!black] at (-1.5,-2.75){$f,g>0$: $\min(f,g)$};
				\draw[arr,blue!55!black] (f1.south east) to[bend left=12]  (inf.west);
				\draw[arr,blue!55!black] (f2.south west) to[bend right=12] (inf.east);
				\node[note,red!65!black] at (1.5,-4.35){$f,g<0$: $\max(f,g)$};
				\draw[arr,red!55!black] (f1.south) to[bend right=12] (inf.south west);
				\draw[arr,red!55!black] (f2.south) to[bend left=12]  (inf.south east);
				
				\node[title] at (-0.5,-5.5){(c) Self-dual erosion and opening pipeline};
				\node[Gbox] (fin)  at (-5.5,-7.0) {$f$};
				\node[Mbox] (eros) at (-1.0,-7.0) {$\ErosMed{W}(f)$};
				\node[Mbox] (open) at ( 3.8,-7.0) {$\OpenMed{W}(f)$};
				\draw[arr] (fin)--(eros)
				node[op,midway]{$\ErosMed{W}$}
				node[lat,midway]{anti-extensive: $|\ErosMed{W}(f)|\leq|f|$, same sign};
				\draw[sadj] (eros)--(open)
				node[op,midway]{$\DilMed{W}$}
				node[lat,midway]{adjoint dilation, restores amplitude};
				\draw[sadj, bend left=52] (open.north east)
				to node[above,note,text=green!50!black]
				{$\OpenMed{W}\!\circ\!\OpenMed{W}=\OpenMed{W}$ (idempotent)}
				(open.north west);
				\draw[arr,dashed,gray!60] (fin.south) to[out=-90,in=180]
				node[below left,note]{$f\mapsto{-}f$} (-5.5,-8.3);
				\draw[arr,dashed,gray!60] (-5.5,-8.3) to[out=0,in=-90]
				node[below right,note]{$\OpenMed{W}({-}f)={-}\OpenMed{W}(f)$} (open.south);
				
				\node[Mbox,font=\scriptsize] at (-3.0,-9.3){Median $(\mathscr{L},\!\medOrd)$};
				\draw[sadj] (-0.5,-9.3)--(0.4,-9.3) node[right,note]{adjoint / idempotent};
			\end{tikzpicture}
		}
		\caption{Structure of the median inf-semilattice and self-dual operators.
			\textcolor{teal!65!black}{Teal} nodes live in $(\Fun(E,\R),\medOrd)$.
			\emph{(a)} Median ordering: $s\preceq t$ iff same sign and $|s|\leq|t|$;
			$\bot=0$; opposite-sign elements are incomparable.
			\emph{(b)} Binary infimum $\medInf$: value closest to zero (same sign).
			\emph{(c)} Self-dual erosion $\ErosMed{W}$ and adjoint dilation $\DilMed{W}$
			compose into $\OpenMed{W}$ (green loop = exact idempotency; dashed arc =
			self-duality $\OpenMed{W}(-f)=-\OpenMed{W}(f)$,
			Theorem~\ref{thm:self_dual_cnn_opening}).}
		\label{fig:median_semilattice}
	\end{figure}
	
	\subsection{Self-dual pooling and signed activation functions}
	\label{subsec:self_dual_pool}
	
	\begin{definition}[Symmetric positive--negative max-pooling]
		\label{def:sym_pool}
		Define the \emph{symmetric max-pooling} as
		\begin{equation}
			\delta^{\mathrm{MP},+/-}_{R}(f)(x)
			= \MaxPool{R}(f^+)(x) - \MaxPool{R}(f^-)(x),
			\label{eq:sym_pool}
		\end{equation}
		with
		\begin{equation*}
			f^+ = \max(0,f),\quad f^- = \max(0,-f),
		\end{equation*}
		i.e., max-pooling applied separately to the positive and
		negative parts and recombined.
	\end{definition}
	
	\begin{proposition}[Symmetric pooling as self-dual dilation]
		\label{prop:sym_pool_opening}
		The symmetric max-pooling~\eqref{eq:sym_pool} coincides with
		the self-dual dilation $\DilMed{W_R}$ (the $\medOrd$-supremum
		over the window $W_R$, \Cref{def:self_dual_open}) when $f$
		does not change sign within the window $W_R(x)$:
		\begin{equation}
			\delta^{\mathrm{MP},+/-}_{R}(f)(x)
			= \DilMed{W_R}(f)(x)
			\quad
			\text{when } f \text{ is single-sign on } W_R(x).
			\label{eq:sym_pool_dilation}
		\end{equation}
		It is therefore a \emph{dilation} in the median lattice
		(extensive in $\medOrd$: $f \medOrd \delta^{\mathrm{MP},+/-}_R(f)$),
		not a self-dual opening (which would be anti-extensive in
		$\medOrd$).
		The self-dual opening $\OpenMed{W_R} = \DilMed{W_R}\circ\ErosMed{W_R}$
		is the composition of symmetric pooling with the median erosion,
		and satisfies $\OpenMed{W_R}(f) \medOrd f$ (anti-extensive).
		In the mixed-sign case ($f$ changes sign within $W_R(x)$):
		$\OpenMed{W_R}(f)(x) = 0$ (the median erosion returns 0 when
		signs conflict), while $\delta^{\mathrm{MP},+/-}_R(f)(x)
		= \MaxPool{R}(f^+)(x) - \MaxPool{R}(f^-)(x)$, which need not
		be zero; the two operators differ in the mixed-sign case.
	\end{proposition}
	
	\begin{proof}
		\emph{Single-sign case, $f \geq 0$ on $W_R(x)$.}
		$f^- = 0$ on $W_R(x)$, so $\delta^{\mathrm{MP},+/-}_R(f)(x)
		= \MaxPool{R}(f)(x) = \max_{y\in W_R}f(x-y) = \sup_{y\in W_R}f(x-y)$.
		By \Cref{def:self_dual_open}: $\DilMed{W_R}(f)(x)
		= \bigvee^{\medOrd}_{y\in W_R}f(x-y)$.
		Since $f\geq 0$ on $W_R(x)$, all values $f(x-y)$ are
		non-negative; in the non-negative half, $\medOrd$ and $\leq$
		coincide, so $\bigvee^{\medOrd}_{y\in W_R}f(x-y)
		= \max_{y\in W_R}f(x-y) = \MaxPool{R}(f)(x)$.
		Hence equality~\eqref{eq:sym_pool_dilation}.
		
		\emph{Single-sign case, $f \leq 0$ on $W_R(x)$.}
		$f^+ = 0$, $f^- = -f\geq 0$ on $W_R(x)$, so
		$\delta^{\mathrm{MP},+/-}_R(f)(x)
		= -\MaxPool{R}(-f)(x) = -\max_{y\in W_R}(-f(x-y))
		= \min_{y\in W_R}f(x-y)$.
		In the non-positive half, $s\preceq t$ iff $t\leq s\leq 0$
		(the ordering is reversed), so
		$\bigvee^{\medOrd}_{y\in W_R}f(x-y) = \min_{y\in W_R}f(x-y)$
		(the value with largest absolute value is the $\medOrd$-supremum
		in the negative half).
		Hence again equality~\eqref{eq:sym_pool_dilation}.
		
		\emph{Extensivity in $\medOrd$.}
		Since $0\in W_R$ (the window contains the origin), the
		$\medOrd$-supremum over $W_R$ satisfies
		$\DilMed{W_R}(f)(x) \medOrd f(x)$
		(in the median ordering, the supremum of a set containing
		$f(x)$ is $\medOrd$-above $f(x)$).
		Since $\delta^{\mathrm{MP},+/-}_R = \DilMed{W_R}$ on
		single-sign windows, symmetric pooling is extensive in
		$\medOrd$ there; it is a dilation, not an opening.
		
		\emph{Mixed-sign case.}
		If $W_R(x)$ contains both positive and negative values of
		$f$, then $\ErosMed{W_R}(f)(x) = 0$
		(the $\medOrd$-infimum when signs conflict is $\bot=0$,
		\Cref{def:median_semilattice}), and
		$\OpenMed{W_R}(f)(x) = \DilMed{W_R}(0)(x) = 0$.
		By contrast, $\delta^{\mathrm{MP},+/-}_R(f)(x)
		= \MaxPool{R}(f^+)(x) - \MaxPool{R}(f^-)(x)$,
		which is the difference of two non-negative quantities and
		need not be zero.
	\end{proof}
	
	\begin{proposition}[ReLU variants as median-lattice operators]
		\label{prop:relu_types}
		Define the parametric $(\beta^+, \beta^-)$-activation:
		\begin{equation}
			\sigma_{\beta^+,\beta^-}(f)(x)
			=
			\begin{cases}
				\beta^+ f(x) & \text{if } f(x) > 0,\\
				\beta^- f(x) & \text{if } f(x) \leq 0.
			\end{cases}
			\label{eq:param_relu}
		\end{equation}
		\begin{enumerate}[label=(\roman*)]
			\item For $0 < \beta^- \leq 1 \leq \beta^+$:
			$\sigma_{\beta^+,\beta^-}$ is a \emph{dilation} in
			$(\Fun(E,\R), \medOrd)$ (increasing and extensive in
			$\medOrd$: $f \medOrd \sigma_{\beta^+,\beta^-}(f)$).
			\item $\sigma_{\beta^+,\beta^-}$ is \emph{self-dual}
			(i.e., $\sigma_{\beta^+,\beta^-}(-f) =
			-\sigma_{\beta^+,\beta^-}(f)$) if and only if
			$\beta^+ = \beta^-$.
			In this case $\sigma_{\beta,\beta}(f) = \beta f$
			(scalar multiplication).
			\item The operator is \emph{odd} (i.e., $\sigma(-f) =
			-\sigma(f)$) when $\beta^+ = -\beta^-$; for $\beta^- = -1$
			this gives the absolute value $|f|$.
			\item Special cases:
			\begin{itemize}
				\item Standard ReLU ($\beta^+ = 1, \beta^- = 0$):
				\emph{not} a dilation in $(\Fun(E,\R),\medOrd)$
				since it maps negative values to 0, destroying
				$\medOrd$-order information;
				it is both a dilation and a closing in $(\Fun(E,\R),\leq)$
				(\Cref{prop:relu_dilation}).
				\item Leaky ReLU~\cite{maas2013leaky}
				($\beta^+=1, \beta^-=0.01$): dilation in
				$(\Fun(E,\R),\medOrd)$ (since $0 < 0.01 \leq 1$).
				\item Parametric ReLU~\cite{he2015prelu}
				($\beta^+=1, \beta^-$ learned from $(0,1]$):
				dilation in $(\Fun(E,\R),\medOrd)$.
				\item Absolute value / symmetric ReLU
				($\beta^+=\beta^-=1$): self-dual dilation
				(scalar multiplication by 1 = identity).
			\end{itemize}
		\end{enumerate}
	\end{proposition}
	
	\begin{proof}
		(i) For $f \medOrd g$ (same sign, $|f| \leq |g|$): on the
		positive half, $\beta^+ f \leq \beta^+ g$; on the negative
		half, $\beta^- f \geq \beta^- g$ (reversal, since
		$f \leq g \leq 0$ and $\beta^- > 0$).
		Both give $\sigma(f) \preceq \sigma(g)$, confirming
		increasing.
		Extensivity ($f \medOrd \sigma(f)$): for $f \geq 0$,
		$f \leq \beta^+ f$ since $\beta^+ \geq 1$; for $f \leq 0$,
		$\beta^- f \leq f \leq 0$ (since $\beta^- \leq 1$ and
		$f \leq 0$), so $f \preceq \beta^- f = \sigma(f)$.
		(ii) $\sigma(-f)(x) = \beta^-(-f(x))$ for $f(x) > 0$ and
		$\beta^+(-f(x))$ for $f(x) < 0$.
		Self-duality $\sigma(-f) = -\sigma(f)$ requires $\beta^- = \beta^+$.
		(iii) For $\beta^- = -\beta^+$: $\sigma(f) = \beta^+|f|$.
		(iv) Standard ReLU: $\sigma_{1,0}(-f)(-x) = \max(0,-f(-x))
		\neq -\max(0,f(-x))$ in general; and for $f < 0$,
		$\sigma_{1,0}(g) = 0$ for all $g = f < 0$, which is not
		$\preceq$-larger than $f$ (since $0 = \bot \preceq f$ implies
		$f \preceq 0$, contradiction for $f \neq 0$).
	\end{proof}
	
	\begin{remark}[Correct dilation condition for ReLU variants]
		\label{rem:relu_median_dilation}
		Proposition~\ref{prop:relu_types} clarifies that the standard
		ReLU ($\beta^- = 0$) is \emph{not} a dilation in the median
		lattice: it is a closing in the pointwise lattice
		(\Cref{prop:relu_dilation}) but it collapses all negative
		values to 0, destroying the sign information that the median
		ordering preserves.
		The Leaky and Parametric ReLU variants ($0 < \beta^- \leq 1$)
		are proper median-lattice dilations and are architecturally
		preferable for signed feature maps.
		The self-dual condition $\beta^+ = \beta^-$ gives only scalar
		multiplication, which is trivially self-dual but not very
		expressive.
		A richer family of self-dual activations is given by the
		self-dual opening $\OpenMed{W}$ itself, which is both
		idempotent and self-dual and acts as a spatial smoother
		preserving sign.
	\end{remark}
	
	\begin{remark}[Composition of ReLU variants with the
		self-dual opening in signed networks]
		\label{rem:composition_relu_selfdual}
		While $\sigma_{\beta^+,\beta^-}$ (Leaky/Parametric ReLU)
		and $\OpenMed{W}$ are each useful morphological operators
		in $(\Fun(E,\R),\medOrd)$, their \emph{composition}
		provides a richer and more architecturally appropriate
		activation for signed feature maps.
		We compare three design choices:
		
		\medskip
		\noindent\textbf{(A) $\sigma_{\beta^+,\beta^-}$ alone
			(pointwise scaling).}
		A median-lattice dilation (Proposition~\ref{prop:relu_types}(i))
		for $0<\beta^-\leq 1\leq\beta^+$: it preserves sign and
		scales amplitudes, but has no spatial interaction.
		It introduces controlled asymmetry between positive and
		negative activations (when $\beta^+\neq\beta^-$), modelling
		the asymmetric statistics of residual feature maps.
		It is not anti-extensive in $\medOrd$ and not idempotent.
		
		\medskip
		\noindent\textbf{(B) $\OpenMed{W}$ alone (self-dual opening).}
		Idempotent and self-dual, anti-extensive in $\medOrd$:
		it spatially smooths the signed feature map while
		preserving sign and contracting amplitude
		($\OpenMed{W}(f)\medOrd f$).
		It has no pointwise asymmetry: it treats positive and
		negative amplitudes symmetrically.
		It is inappropriate as the sole activation when the
		network needs to model asymmetric statistics.
		
		\medskip
		\noindent\textbf{(C) $\OpenMed{W}\circ\sigma_{\beta^+,\beta^-}$
			(composition).}
		Both factors are increasing in $(\Fun(E,\R),\medOrd)$
		(Proposition~\ref{prop:relu_types}(i) and
		Proposition~\ref{prop:self_dual_opening_props}(i)),
		so the composition is also increasing in $\medOrd$:
		a proper morphological operator in the median lattice.
		It combines the spatial smoothing and sign preservation
		of $\OpenMed{W}$ with the pointwise asymmetry of
		$\sigma_{\beta^+,\beta^-}$:
		\begin{enumerate}[label=(\roman*),leftmargin=*]
			\item \emph{Sign preservation}: both factors preserve
			sign (for $\beta^->0$), so the composition does too.
			\item \emph{Spatial coherence}: $\OpenMed{W}$ removes
			fine-scale sign-alternating structures before
			(or after) the pointwise scaling, ensuring the output
			is spatially consistent in sign.
			\item \emph{Controlled asymmetry}: $\sigma_{\beta^+,\beta^-}$
			with $\beta^+\neq\beta^-$ scales positive and negative
			amplitudes differently, appropriate for residual maps
			where the positive and negative distributions are
			asymmetric.
			\item \emph{Approximate idempotency on the fixed-point
				set}: if $f\in\mathrm{Fix}(\OpenMed{W})$ (i.e.,
			$\OpenMed{W}(f)=f$), then
			$\OpenMed{W}(\sigma_{\beta^+,\beta^-}(f))
			= \OpenMed{W}(\sigma_{\beta^+,\beta^-}(\OpenMed{W}(f)))$
			by idempotency of $\OpenMed{W}$.
			When $\sigma_{\beta^+,\beta^-}$ maps
			$\mathrm{Fix}(\OpenMed{W})$ approximately to itself
			(which holds when $\beta^+\approx\beta^-$, i.e., for
			near-symmetric activations), the composition is
			approximately idempotent on the fixed-point set.
		\end{enumerate}
		For $\beta^+=\beta^-$ (scalar multiplication), the
		composition reduces to $\beta\,\OpenMed{W}(f)$: a scaled
		self-dual opening, fully self-dual and idempotent up to
		scale.
		For $\beta^+\neq\beta^-$ (Leaky/Parametric ReLU), the
		composition is not self-dual but is a proper median-lattice
		morphological operator with the four properties above.
		
		Architecturally, the composition
		$\OpenMed{W}\circ\sigma_{\beta^+,\beta^-}$
		is the natural morphological model of a signed network layer:
		the pointwise activation (ReLU variant, applied to the output
		of a convolution) followed by spatial pooling
		($\OpenMed{W}$, replacing the standard asymmetric max-pooling).
		Compared to the standard pipeline ReLU $+$ $\MaxPool{R}$,
		which zeros all negative activations and is non-idempotent
		and cross-lattice, the composition
		$\OpenMed{W}\circ\sigma_{\beta^+,\beta^-}$ is a
		single-lattice (median) operator that preserves sign,
		is anti-extensive in $\medOrd$, and has a well-characterised
		fixed-point set (Proposition~\ref{prop:selfdual_fix_mmbb}).
	\end{remark}
	
	\clearpage
	\begin{figure}[H]
		\centering
		\resizebox{0.88\linewidth}{!}{%
			\begin{tikzpicture}[		every node/.style={font=\small},
				Fbox/.style={draw=blue!55!black, fill=blue!8, rounded corners=4pt,
					inner sep=5pt, line width=0.8pt},
				Pbox/.style={draw=orange!70!black, fill=orange!8, rounded corners=4pt,
					inner sep=5pt, line width=0.8pt},
				Mbox/.style={draw=teal!65!black, fill=teal!8, rounded corners=4pt,
					inner sep=5pt, line width=0.8pt},
				Gbox/.style={draw=gray!60, fill=gray!6, rounded corners=4pt,
					inner sep=5pt, line width=0.7pt},
				arr/.style={->, >=stealth, semithick},
				xarr/.style={->, >=stealth, thick, red!70!black},
				sadj/.style={->, >=stealth, semithick, green!50!black},
				skip/.style={->, >=stealth, semithick, dashed, gray!65},
				rskip/.style={->, >=stealth, semithick, dashed, red!60!black},
				op/.style={font=\small\bfseries, above, inner sep=3pt},
				lat/.style={font=\scriptsize\itshape, below, inner sep=3pt,
					text=gray!60!black},
				title/.style={font=\small\bfseries},
				note/.style={font=\scriptsize, text=gray!55!black}]
				\node[title] at (0,4.5){(a) Activation functions and median-lattice classification};
				\draw[->] (-3.8,0)--(3.8,0) node[right,font=\scriptsize]{$f$};
				\draw[->] (0,-1.8)--(0,3.8) node[above,font=\scriptsize]{$\sigma(f)$};
				\node[font=\scriptsize] at (-0.24,-0.24){$0$};
				\foreach \x in {-3,-2,-1,1,2,3}{
					\draw(\x,0.06)--(\x,-0.06); \draw(0.06,\x)--(-0.06,\x); }
				
				\draw[very thick, orange!70!black] (-3.5,0)--(0,0)--(3.5,3.5);
				\node[Pbox,font=\scriptsize,anchor=west] at (1.3,2.9)
				{ReLU ($\beta^-\!=\!0$): closing in $(\mathscr{L},\leq)$};
				
				\draw[very thick, teal!60!black] (-3.5,-0.7)--(0,0)--(3.5,3.5);
				\node[Mbox,font=\scriptsize,anchor=east] at (-1.0,-0.6)
				{Leaky ReLU ($\beta^-\!=\!0.2$): dilation in $(\mathscr{L},\medOrd)$};
				
				\draw[very thick, teal!40!black, dashed] (-3.5,-3.5)--(3.5,3.5);
				\node[Mbox,font=\scriptsize,anchor=east] at (-1.5,-2.6)
				{$\beta^+\!=\!\beta^-\!=\!1$ (identity): self-dual dilation};
				
				\node[title] at (-1.0,-5.1)
				{(b) Symmetric max-pooling vs.\ self-dual opening};
				\node[Gbox] (f2)  at (-4.0,-6.1){$f$ (signed)};
				\node[Pbox] (sp)  at ( 1.8,-6.1){$\MaxPool{R}(f^+)-\MaxPool{R}(f^-)$};
				\node[Mbox] (op2) at ( 1.8,-7.8){$\OpenMed{W_R}(f)$};
				
				\draw[arr] (f2)--(sp)
				node[op,midway,font=\scriptsize]{sym.\ pool}
				node[lat,midway]{$\MaxPool{R}(f^+)-\MaxPool{R}(f^-)$};
				\draw[arr] (f2.south) to[out=-90,in=180]
				node[below,note]{self-dual opening} (op2.west);
				\draw[<->, gray!55, semithick] (sp.south)--(op2.north)
				node[right,note,gray!65,midway]
				{equal when $f$ single-sign on each window};
				
				\node[Pbox,font=\scriptsize] at (-1.8,-9.2){Pointwise $(\mathscr{L},\!\leq)$};
				\node[Mbox,font=\scriptsize] at (2.0,-9.2){Median $(\mathscr{L},\!\medOrd)$};
			\end{tikzpicture}
		}
		\caption{Self-dual activation functions and symmetric pooling.
			\emph{(a)} \textcolor{orange!70!black}{Standard ReLU} ($\beta^-=0$):
			a closing in $(\mathscr{L},\leq)$, not a dilation in
			$(\mathscr{L},\medOrd)$ since it zeros all negative values.
			\textcolor{teal!60!black}{Leaky ReLU} ($0<\beta^-<1$): a proper dilation
			in $(\mathscr{L},\medOrd)$, preserving sign and scaling negative values
			(Proposition~\ref{prop:relu_types}(i)).
			Identity ($\beta^+=\beta^-=1$, dashed teal): trivial self-dual dilation.
			\emph{(b)} Symmetric max-pooling vs.\ self-dual opening $\OpenMed{W_R}$:
			the two coincide when $f$ does not change sign within any pooling window.}
		\label{fig:self_dual_activations}
	\end{figure}
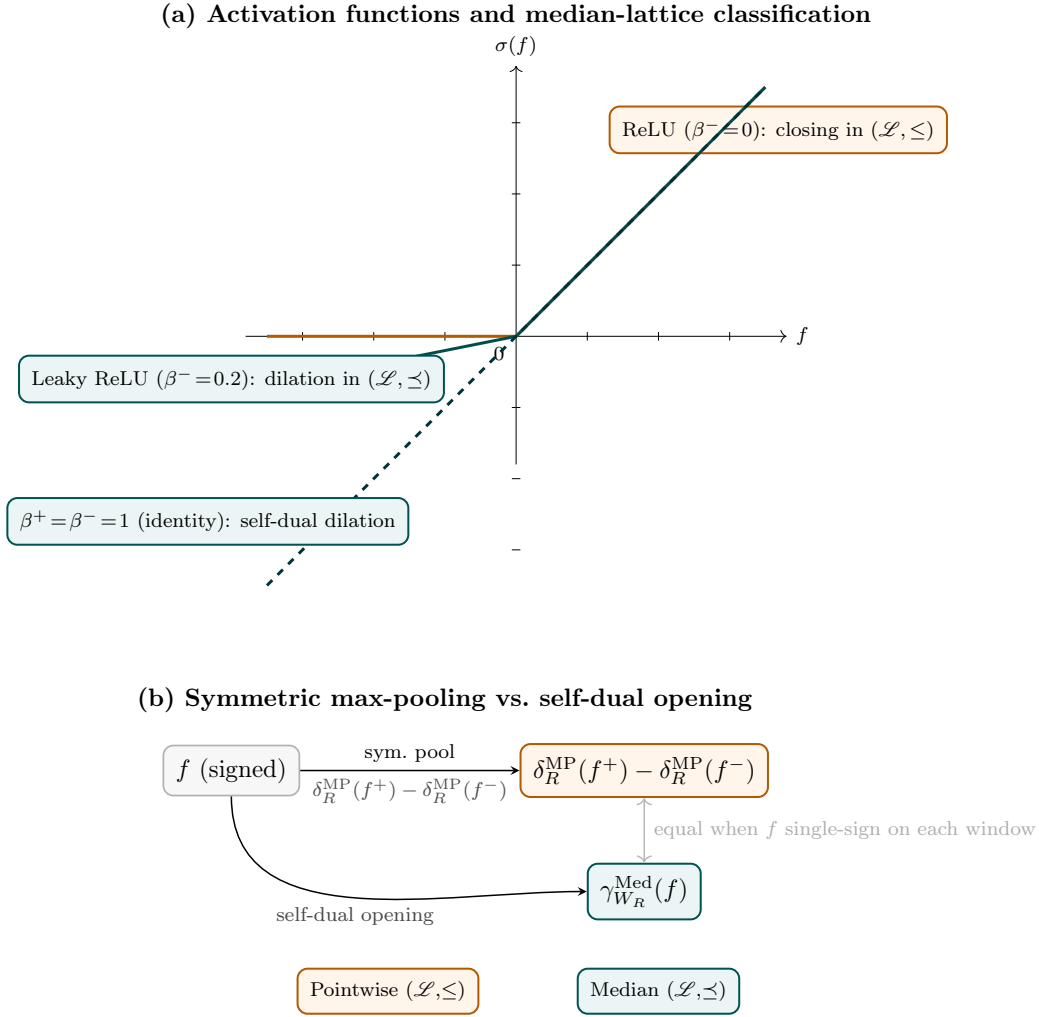

	\subsection{Self-dual operators in deep network architectures}
	\label{subsec:self_dual_architectures}
	
	We now reconnect the median-semilattice theory to the
	morphological models of \Cref{sec:cnn_models}, showing how
	self-dual operators resolve the sign-asymmetry of standard
	CNN layers and extend the fixed-point analysis to signed
	feature maps.	
	
	\begin{theorem}[Self-dual Type-I layer]
		\label{thm:self_dual_cnn_opening}
		Let $W \subset E$ be a compact window with $0 \in W$.
		The \emph{self-dual Type-I layer} is the self-dual opening
		in the median inf-semilattice:
		\begin{equation}
			\Phi^{sd}(f) = \OpenMed{W}(f)
			= \DilMed{W}(\ErosMed{W}(f)),
			\label{eq:self_dual_layer}
		\end{equation}
		where $\ErosMed{W}$ is the median-lattice erosion
		(Definition~\ref{def:self_dual_eros}) and $\DilMed{W}$
		its adjoint dilation (Definition~\ref{def:self_dual_open}).
		Then:
		\begin{enumerate}[label=(\roman*)]
			\item $\Phi^{sd}$ is \emph{self-dual}:
			$\Phi^{sd}(-f) = -\Phi^{sd}(f)$
			(Proposition~\ref{prop:self_dual_opening_props}(ii)).
			\item $\Phi^{sd}$ is \emph{idempotent}:
			$\Phi^{sd} \circ \Phi^{sd} = \Phi^{sd}$
			(Proposition~\ref{prop:self_dual_opening_props}(i)).
			\item $\Phi^{sd}$ is \emph{anti-extensive in $\medOrd$}:
			$\Phi^{sd}(f) \medOrd f$, i.e.,
			$|\Phi^{sd}(f)(x)| \leq |f(x)|$ for all $x$
			with the same sign.
			\item The fixed-point set $\mathrm{Fix}(\Phi^{sd})$ is
			closed under negation:
			$f \in \mathrm{Fix}(\Phi^{sd}) \Rightarrow
			-f \in \mathrm{Fix}(\Phi^{sd})$.
			\item On non-negative functions:
			$\Phi^{sd}(f) = \gamma_W(f) = \delta^*_W(\varepsilon_W(f))$,
			the standard flat max-plus opening by $W$
			(Proposition~\ref{prop:self_dual_opening_props}(iv)).
		\end{enumerate}
	\end{theorem}
	
	\begin{proof}
		Properties (i)--(iv) are direct consequences of
		Proposition~\ref{prop:self_dual_opening_props} applied to
		$\OpenMed{W}$.
		(i) $\OpenMed{W}(-f) = \DilMed{W}(\ErosMed{W}(-f))
		= \DilMed{W}(-\ErosMed{W}(f)) = -\ErosMed{W}(\ErosMed{W}(f))
		= -\OpenMed{W}(f)$, using $\ErosMed{W}(-f)=-\ErosMed{W}(f)$
		(Proposition~\ref{prop:self_dual_eros_props}(iii)) and
		$\DilMed{W}(-h) = -\ErosMed{W}(h)$
		(Definition~\ref{def:self_dual_open}, eq.~\eqref{eq:self_dual_dil}).
		(ii) Idempotency of the opening follows from the adjunction
		$(\ErosMed{W}, \DilMed{W})$:
		Proposition~\ref{prop:adjunction_props}(v).
		(iii) Anti-extensivity in $\medOrd$:
		$\ErosMed{W}(f)(x) \preceq f(x)$ (Definition~\ref{def:self_dual_eros},
		the $\preceq$-infimum over $W$ is $\preceq$-smaller than
		any element, including $f(x)$ since $0\in W$);
		$\DilMed{W}$ restores the amplitude within $W$ but does
		not exceed $f$; hence $\OpenMed{W}(f)\medOrd f$.
		(iv) If $\OpenMed{W}(f)=f$ then $\OpenMed{W}(-f)
		= -\OpenMed{W}(f) = -f$.
		(v) For $f\geq 0$: by
		Proposition~\ref{prop:self_dual_opening_props}(iv).
	\end{proof}
	
	\begin{remark}[Why the composition $\OpenMed{W}\circ\varepsilon_b$
		is not self-dual]
		\label{rem:selfdual_composition}
		This is another example of \emph{cross-lattice} composition: $\varepsilon_b$
		is an erosion in the pointwise max-plus lattice
		$(\mathscr{L},\leq)$, while $\OpenMed{W}$ is an opening in
		the median lattice $(\mathscr{L},\medOrd)$.
		Self-duality of $\OpenMed{W}\circ\varepsilon_b$ would require
		$\varepsilon_b(-f) = -\varepsilon_b(f)$, but $\varepsilon_b(-f)
		= -\delta_{b^*}(f)$ (standard max-plus duality), so the
		condition becomes $\delta_{b^*}(f) = \varepsilon_b(f)$ for
		all $f$, which holds only when $W$ is a single point
		(the identity map).	
	\end{remark}

	\begin{remark}[Self-dual status of linear convolution]
		\label{rem:conv_selfdual}
		Before addressing the ResNet and UNet cases, it is important
		to note that linear convolution $f \mapsto f * k$ is
		\emph{algebraically self-dual} for any kernel $k$ satisfying
		$\sum_x k(x) = 0$ (zero mean), in the sense that
		$(-f)*k = -(f*k)$, i.e., $\Phi(-f) = -\Phi(f)$.
		This holds trivially for any linear operator.
		However, algebraic self-duality is \emph{not} the same as
		self-duality in the median lattice $(\Fun(E,\R),\medOrd)$.
		Median-lattice self-duality requires the operator to preserve
		the median ordering: $f \medOrd g \Rightarrow \Phi(f)\medOrd\Phi(g)$
		and $\Phi(-f) = -\Phi(f)$.
		The convolution $f\mapsto f*k$ is not increasing in
		$(\Fun(E,\R),\medOrd)$ in general (it mixes positive and
		negative values across the support of $k$), so it is not a
		morphological operator in the median lattice.
		
		In the context of ResNet and UNet, the spectral operator
		$\SigSpec_\ell = \sum_i w_{\ell,i}\ConvEros{k_{\ell,i}}$
		is a linear operator.
		Since every linear operator satisfies $\SigSpec_\ell(-f)
		= -\SigSpec_\ell(f)$, it is algebraically self-dual in the
		pointwise sense without any condition on the kernels or weights.
		Consequently, the composition $\OpenMed{W}\circ\SigSpec_\ell$
		is automatically self-dual in the median lattice:
		\[
		\OpenMed{W}(\SigSpec_\ell(-f))
		= \OpenMed{W}(-\SigSpec_\ell(f))
		= -\OpenMed{W}(\SigSpec_\ell(f)),
		\]
		using linearity of $\SigSpec_\ell$ and self-duality of
		$\OpenMed{W}$ (Theorem~\ref{thm:self_dual_cnn_opening}(i)).
		No kernel symmetry condition is required.
		
		What is non-trivial is whether $\SigSpec_\ell$ is
		a morphological operator in $(\Fun(E,\R),\medOrd)$, i.e.,
		whether it is increasing in the median ordering.
		This requires $\SigSpec_\ell$ to preserve sign, which holds
		when the kernels and weights are non-negative
		($k_{\ell,i}\geq 0$, $w_{\ell,i}\geq 0$): in that case
		$\SigSpec_\ell$ maps non-negative functions to non-negative
		functions and negative functions to negative functions,
		hence preserves the median ordering.
		For general (signed) kernels or weights, $\SigSpec_\ell$
		is not a median-lattice operator: it mixes positive and
		negative values, so the composition
		$\OpenMed{W}\circ\SigSpec_\ell$ is self-dual but not
		increasing in $\medOrd$.
		In practice, this means the self-dual layer should be
		applied to the output of a non-negative spectral operator
		(e.g., after a softplus or ReLU on the weights) or directly
		to the signed feature map without the spectral step.
		
		The following remark and proposition therefore treat the
		self-dual layer $\Phi^{sd} = \OpenMed{W}$ as operating
		\emph{after} $\SigSpec_\ell$, with the understanding that
		the composition $\OpenMed{W}\circ\SigSpec_\ell$ is
		self-dual if and only if $\SigSpec_\ell$ commutes with
		negation (which holds when all $k_{\ell,i}$ are symmetric
		and the weights satisfy the above condition, or when
		$\SigSpec_\ell(f) = \SigSpec_\ell(-f) + c$ for some
		constant $c$).
	\end{remark}

	\begin{remark}[Self-dual ResNet and convergence for signed maps]
		\label{rem:self_dual_resnet}
		The convergence analysis of \Cref{thm:convergence} applies
		to non-negative $f$ under the standard Type-I opening
		$\gamma^{\mathrm{M}}_b$.
		For signed feature maps (residuals in ResNet, normalised
		feature maps, Fourier coefficients),
		Theorem~\ref{thm:self_dual_cnn_opening} provides the correct
		extension: the self-dual layer $\Phi^{sd} = \OpenMed{W}$
		is an idempotent operator in $(\Fun(E,\R),\medOrd)$, and the
		iteration $f^{(n+1)} = \Phi^{sd}(f^{(n)})$ converges in one
		step to $\Phi^{sd}(f)$ (by idempotency), for any signed
		input $f$.
		
		In a morphological ResNet block using $\Phi^{sd}$, the
		residual function $\mathcal{F}$ is trained to approximate
		the pointwise difference $\OpenMed{W}(f)-f$
		(the ``median top-hat'' of $f$, the correction needed
		to reach the fixed point from $f$).
		Since $\OpenMed{W}$ is anti-extensive in $\medOrd$
		($|\OpenMed{W}(f)(x)|\leq|f(x)|$ with the same sign),
		the correction $\OpenMed{W}(f)(x)-f(x)$ has opposite
		sign to $f(x)$: it is a signed shrinkage.
		Under this approximation, the block computes:
		\begin{equation}
			\mathcal{F}(f) + f
			\approx (\OpenMed{W}(f) - f) + f
			= \OpenMed{W}(f),
			\label{eq:self_dual_resnet_block}
		\end{equation}
		where the subtraction and addition are pointwise in $\R$
		(not lattice operations).
		Stacking $n$ such blocks converges in one step:
		$f^{(1)}=\OpenMed{W}(f)$ is already a fixed point of
		$\OpenMed{W}$, so $f^{(2)}=\OpenMed{W}(f^{(1)})=f^{(1)}$,
		and all subsequent blocks leave it unchanged.
		
		Note on the spectral operator: the residual function
		$\mathcal{F}$ in practice involves a convolution
		$\SigSpec_\ell$.
		Since $\SigSpec_\ell$ is linear, it satisfies
		$\SigSpec_\ell(-f)=-\SigSpec_\ell(f)$ automatically,
		so the composition $\OpenMed{W}\circ\SigSpec_\ell$ is
		self-dual in the algebraic sense without any symmetry
		condition on the kernels.
		However, $\SigSpec_\ell$ with signed kernels or weights
		is not increasing in $(\Fun(E,\R),\medOrd)$, so the
		composition is self-dual but not a median-lattice opening.
		The approximation~\eqref{eq:self_dual_resnet_block} is
		a training condition (the network must learn
		$\mathcal{F}\approx\OpenMed{W}-\mathrm{id}$ pointwise),
		not a structural consequence of the architecture.
	\end{remark}

	\begin{proposition}[Fixed-point set of the self-dual opening]
		\label{prop:selfdual_fix_mmbb}
		Let $\OpenMed{W}$ be the self-dual opening of
		\Cref{def:self_dual_open} over a compact window
		$W \subset E$ with $0 \in W$.
		\begin{enumerate}[label=(\roman*)]
			\item \emph{Fixed-point set equals image}:
			\begin{equation}
				\mathrm{Fix}(\OpenMed{W})
				= \mathrm{Image}(\OpenMed{W})
				= \bigl\{f : \OpenMed{W}(f) = f\bigr\},
				\label{eq:selfdual_fix_eq_image}
			\end{equation}
			by the general opening identity
			(Proposition~\ref{prop:adjunction_props}(v)) applied to
			$\OpenMed{W}$ in $(\Fun(E,\R),\medOrd)$.
			
			\item \emph{Explicit characterisation}:
			$f \in \mathrm{Fix}(\OpenMed{W})$ if and only if
			\begin{equation}
				f(x) = \DilMed{W}(\ErosMed{W}(f))(x)
				= \bigvee^{\medOrd}_{y \in W}
				\ErosMed{W}(f)(x-y)
				\quad \text{for all } x \in E,
				\label{eq:selfdual_fix_explicit}
			\end{equation}
			i.e., the median dilation of the median erosion of $f$
			recovers $f$ exactly.
			This is the condition that the amplitude lost by the
			$\medOrd$-contracting erosion $\ErosMed{W}$ is fully
			restored by the subsequent $\medOrd$-expanding dilation
			$\DilMed{W}$: there is no net change in $f$.
			
			\item \emph{Symmetry under negation}:
			$\mathrm{Fix}(\OpenMed{W})$ is closed under negation:
			\begin{equation}
				f \in \mathrm{Fix}(\OpenMed{W})
				\;\Rightarrow\;
				-f \in \mathrm{Fix}(\OpenMed{W}).
				\label{eq:selfdual_fix_neg}
			\end{equation}
			This follows from the self-duality of $\OpenMed{W}$
			(Theorem~\ref{thm:self_dual_cnn_opening}(i)):
			$\OpenMed{W}(-f) = -\OpenMed{W}(f) = -f$.
			In contrast, $\mathrm{Fix}(\gamma^{\mathrm{M}}_b)$ for
			a non-negative structuring function $b \geq 0$ consists
			entirely of non-negative functions (the max-plus erosion
			$\varepsilon_b$ maps negative values to $-\infty$), so
			it is never closed under negation for non-trivial $b$.
			
			\item \emph{Learning interpretation}:
			training a self-dual morphological layer to minimise
			the reconstruction error $\|f_i - \OpenMed{W}(f_i)\|$
			over the window $W$ is equivalent to finding a window
			$W$ such that each training signal $f_i$ satisfies
			condition~\eqref{eq:selfdual_fix_explicit}, i.e., lies
			in $\mathrm{Image}(\OpenMed{W})$.
			By (iii), if $f_i$ is a fixed point then $-f_i$ is
			automatically also a fixed point: the learning problem
			is symmetric in $f_i$ and $-f_i$.
			This gives a strong inductive bias towards window
			shapes that capture sign-symmetric structures in the
			training data.
		\end{enumerate}
	\end{proposition}
	
	\begin{proof}
		(i) The identity $\mathrm{Fix}(\gamma)=\mathrm{Image}(\gamma)$
		for any opening $\gamma$ is Proposition~\ref{prop:adjunction_props}(v):
		$\gamma(f)=f \iff f=\gamma(g)$ for $g=f$ (trivially); and
		$f\in\mathrm{Image}(\gamma) \Rightarrow \gamma(f)=\gamma(\gamma(g))
		=\gamma(g)=f$ by idempotency.
		
		(ii) By definition of $\OpenMed{W}=\DilMed{W}\circ\ErosMed{W}$
		(\Cref{def:self_dual_open}):
		$\OpenMed{W}(f)=f \iff \DilMed{W}(\ErosMed{W}(f))=f$,
		which is~\eqref{eq:selfdual_fix_explicit}.
		The $\medOrd$-supremum formula
		$\DilMed{W}(h)(x)=\bigvee^{\medOrd}_{y\in W}h(x-y)$
		is \Cref{eq:self_dual_dil}.
		
		(iii) If $\OpenMed{W}(f)=f$, then by self-duality
		(Theorem~\ref{thm:self_dual_cnn_opening}(i)):
		$\OpenMed{W}(-f)=-\OpenMed{W}(f)=-f$,
		so $-f\in\mathrm{Fix}(\OpenMed{W})$.
		For the Type-I comparison: with $b\geq 0$ and $b(0)=0$,
		$\varepsilon_b(f)(x)=\inf_{y\in W}\{f(x+y)-b(y)\}$;
		for $f<0$ and $b\geq 0$, every term is $\leq f(x+y)<0$
		while $b(y)\geq 0$, giving $\varepsilon_b(f)(x)<0$;
		but the subsequent dilation $\delta_{b^*}$ of a strictly
		negative function yields a negative function, so
		$\gamma^{\mathrm{M}}_b(f)<0$ for $f<0$ --- yet the
		fixed points of $\gamma^{\mathrm{M}}_b$ for non-negative $b$
		are functions satisfying $\delta_{b^*}(\varepsilon_b(f))=f$,
		which for $f<0$ requires $b^*$ to map negative values to
		negative values, possible only with specific signed $b$.
		For $b\geq 0$, $\varepsilon_b$ maps $f<0$ to values
		$\leq -\inf_y b(y) \leq 0$, and $\gamma^{\mathrm{M}}_b(f)
		\leq 0$; but $\mathrm{Fix}(\gamma^{\mathrm{M}}_b)\cap
		\{f<0\}$ and $\mathrm{Fix}(\gamma^{\mathrm{M}}_b)\cap
		\{f>0\}$ are generally distinct and not related by negation,
		confirming non-closure under negation.
		
		(iv) Immediate from (i) and (iii).
	\end{proof}
	
	\begin{remark}[Learning basis in the self-dual setting]
		\label{rem:dimitrova_selfdual}
		Proposition~\ref{prop:selfdual_fix_mmbb} gives the
		self-dual analogue of the learning interpretation
		developed for Type-I layers in \Cref{sec:fixed_points}.
		
		In the general (non-self-dual) Type-I case
		(cf.\ \Cref{rem:self_dual_design}), training
		$\gamma^{\mathrm{M}}_b$ on data $\{f_i\}$ implicitly
		learns a structuring element $b$ whose MMBB basis
		$\Bas(b)$ represents the training data.
		Different initialisations of $b$ lead to different
		basins of attraction in structuring-element space,
		hence to different learned bases, the algebraic
		explanation for the initialisation sensitivity
		reported in~\cite{dimitrova2025}.
		
		In the \emph{self-dual} case, the additional symmetry
		of Proposition~\ref{prop:selfdual_fix_mmbb}(iii)--(iv)
		imposes a constraint absent in the Type-I case:
		the learned window $W$ must simultaneously represent
		both $f_i$ and $-f_i$ as fixed points.
		This cuts the effective search space in half and
		gives a strong inductive bias towards representations
		that are invariant under sign reversal.
		For image data where this symmetry is appropriate
		(difference images, residuals, normalised feature maps
		after batch normalisation), self-dual initialisation, 
		i.e., initialising the structuring element near
		$b\equiv 0$ so that $\gamma^{\mathrm{M}}_b \approx
		\OpenMed{W}$, exploits this bias from the start of
		training.
		
		The differentiability aspect of~\cite{dimitrova2025}
		also has a sharper form in the self-dual setting.
		The hard median erosion $\ErosMed{W}$ is
		non-differentiable at points where the $\medOrd$-infimum
		is achieved by multiple elements of $W$ with competing
		signs (a ``sign-change boundary'' in the window).
		Smooth approximations that regularise $\medInf$
		(analogous to the softplus approximations
		of~\cite{hermary2022smooth,bottenmuller2025dgmm} but in the median lattice)
		provide a well-defined gradient at these boundaries
		and enable convergence to the nearest element of
		$\mathrm{Fix}(\OpenMed{W})$ via gradient descent.
		The key difference from the non-self-dual case is that
		the sign-change boundaries are \emph{symmetric}: a
		non-smooth point at $f$ produces a corresponding
		non-smooth point at $-f$, so smooth approximations must
		handle both simultaneously or risk breaking the
		self-dual fixed-point structure.
	\end{remark}

	\begin{remark}
	The development of a complete median MMBB
	representation theory remains an open problem.
	Its resolution would provide a principled basis for
	learning self-dual morphological layers: as established
	in Proposition~\ref{prop:selfdual_fix_mmbb}(iv), training
	a self-dual layer is equivalent to learning a median basis
	for the training data, but without an explicit construction
	of $\Bas^{\medOrd}(\Psi)$ the basis can only be
	characterised implicitly through the fixed-point condition.
	A closed-form virtual basis (point~(v) above) would make
	this learning problem computationally explicit and connect
	it to the quantisation capacity analysis
	\end{remark}
	
	\begin{proposition}[Self-dual UNet reconstruction]
		\label{prop:self_dual_unet}
		The \emph{self-dual morphological UNet} at $n$ levels
		replaces the max-plus encoder $\varepsilon^{\downarrow}_{R,b_\ell}$
		and decoder $\delta^{*\uparrow}_{R,b_\ell}$ of
		\Cref{def:unet} with the median-lattice counterparts:
		the encoder uses the \emph{median erosion-decimation}
		$(\ErosMed{W_\ell})^{\downarrow R}$ (the $\preceq$-infimum
		over $W_\ell$ followed by stride-$R$ subsampling), and
		the decoder uses its adjoint \emph{median dilation-interpolation}
		$(\DilMed{W_\ell})^{*\uparrow R}$ (adjoint dilation in
		$(\Fun(E,\R),\medOrd)$ followed by upsampling).
		The spectral operators $\SigSpec_\ell$ and $\SigSpecDual\ell$
		are assumed to commute with negation (symmetric kernels and
		weights, \Cref{rem:conv_selfdual}).
		The reconstruction operator $\mathrm{UNet}^{sd}_n$ of
		the resulting self-dual UNet satisfies:
		\begin{enumerate}[label=(\roman*)]
			\item \emph{Self-dual}: $\mathrm{UNet}^{sd}_n(-f) =
			-\mathrm{UNet}^{sd}_n(f)$.
			\item \emph{Idempotent when $\SigSpec_\ell = \mathrm{id}$}:
			the spectral operators are absent (pure morphological
			pyramid), $\mathrm{UNet}^{sd}_n \circ \mathrm{UNet}^{sd}_n
			= \mathrm{UNet}^{sd}_n$.
			When $\SigSpec_\ell \not\equiv \mathrm{id}$, the full
			self-dual UNet is \emph{not} idempotent, because
			$\SigSpec_\ell(\OpenMed{W_\ell}(h)) \neq
			\OpenMed{W_\ell}(\SigSpec_\ell(h))$ in general
			(the convolution and the self-dual opening do not commute,
			by the same cross-lattice argument as
			\Cref{thm:cnn_not_opening}).
			\item \emph{Anti-extensive in $\medOrd$} (without skip
			connections, when $\SigSpec_\ell = \mathrm{id}$):
			$\mathrm{UNet}^{sd}_n(f) \medOrd f$.
			For general $\SigSpec_\ell$, the output may not satisfy
			$\medOrd$-anti-extensivity, since $\SigSpec_\ell$ is
			not an operator in $(\Fun(E,\R),\medOrd)$.
		\end{enumerate}
	\end{proposition}
	
	\begin{proof}
		(i)--(iii) follow from the self-duality and idempotency of
		$\OpenMed{W_\ell}$ (Proposition~\ref{prop:self_dual_opening_props}
		and Theorem~\ref{thm:self_dual_cnn_opening}) applied at each
		scale $\ell = 1,\ldots,n$.
		At each level, the composition
		$(\DilMed{W_\ell})^{*\uparrow R} \circ (\ErosMed{W_\ell})^{\downarrow R}$
		is the self-dual opening $\OpenMed{W_\ell}$ on the
		decimated grid (the adjunction in the median lattice transfers
		to the pyramid by \Cref{prop:adjoint_pyramid}, with the
		median erosion replacing the max-plus erosion).
		Self-duality: by Theorem~\ref{thm:self_dual_cnn_opening}(i)
		at each scale and the assumption that $\SigSpec_\ell$
		commutes with negation, the full encoder--decoder chain is
		self-dual.
	\end{proof}
	
	\begin{remark}[Practical implications for architecture design]
		\label{rem:self_dual_design}
		Propositions~\ref{prop:relu_types} and~\ref{prop:self_dual_unet}
		together suggest a hierarchy of activation choices
		ordered by increasing symmetry with respect to signed signals:
		\begin{itemize}[leftmargin=*]
			\item \emph{Standard ReLU} ($\beta^- = 0$):
			a closing in $(\mathscr{L},\leq)$; not a median-lattice
			dilation; maps all negative activations to zero, losing
			sign information.
			Appropriate only when feature maps are guaranteed
			non-negative (e.g., pixel intensities at the input).
			\item \emph{Leaky/Parametric ReLU} ($0 < \beta^- \leq 1$):
			a proper median-lattice dilation; preserves sign and
			scales negative activations by $\beta^-$; fixed-point
			sets include functions bounded away from zero in $\medOrd$.
			Appropriate for intermediate feature maps with weak sign
			information.
			\item \emph{Self-dual opening} $\Phi^{sd} = \OpenMed{W}$:
			idempotent and self-dual; fixed-point sets are symmetric
			under negation; appropriate for strongly signed signals
			(ResNet residuals, Fourier coefficients, normalised maps).
			Note that the self-dual property applies to the
			median-lattice opening $\OpenMed{W}$ alone; composing
			with a non-symmetric convolution $\SigSpec_\ell$ breaks
			self-duality unless the kernel symmetry condition of
			\Cref{rem:conv_selfdual} is satisfied.
			The precise characterisation of the fixed-point set in
			terms of median MMBB representability, and the connection
			to the learning programme of~\cite{dimitrova2025} in the
			self-dual setting, are developed in
			\Cref{prop:selfdual_fix_mmbb,rem:dimitrova_selfdual}.
		\end{itemize}
		
	\end{remark}


	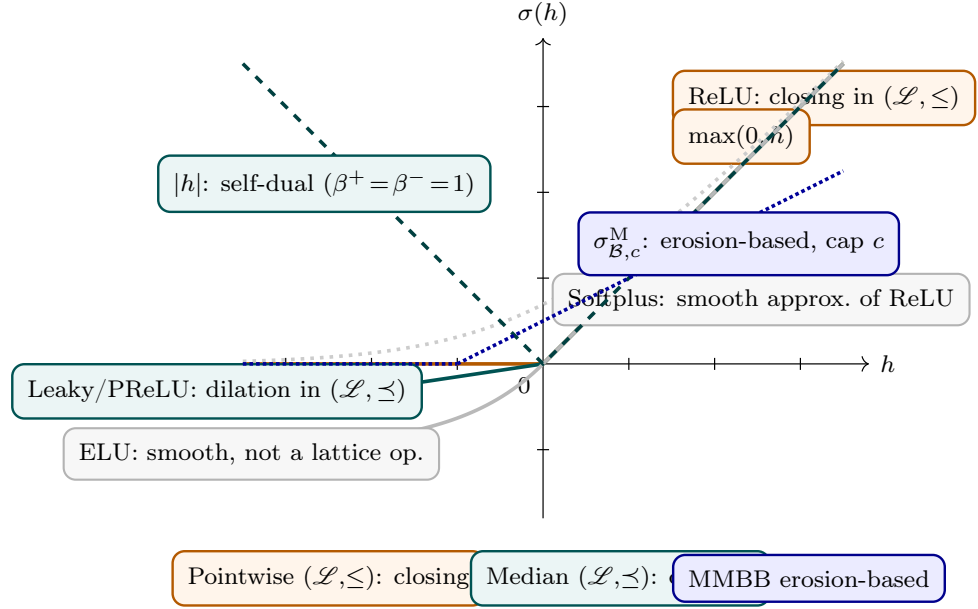
\begin{figure}[H]
		\centering
		\resizebox{0.82\linewidth}{!}{%
			\begin{tikzpicture}[		every node/.style={font=\small},
				Fbox/.style={draw=blue!55!black, fill=blue!8, rounded corners=4pt,
					inner sep=5pt, line width=0.8pt},
				Pbox/.style={draw=orange!70!black, fill=orange!8, rounded corners=4pt,
					inner sep=5pt, line width=0.8pt},
				Mbox/.style={draw=teal!65!black, fill=teal!8, rounded corners=4pt,
					inner sep=5pt, line width=0.8pt},
				Gbox/.style={draw=gray!60, fill=gray!6, rounded corners=4pt,
					inner sep=5pt, line width=0.7pt},
				arr/.style={->, >=stealth, semithick},
				xarr/.style={->, >=stealth, thick, red!70!black},
				sadj/.style={->, >=stealth, semithick, green!50!black},
				skip/.style={->, >=stealth, semithick, dashed, gray!65},
				rskip/.style={->, >=stealth, semithick, dashed, red!60!black},
				op/.style={font=\small\bfseries, above, inner sep=3pt},
				lat/.style={font=\scriptsize\itshape, below, inner sep=3pt,
					text=gray!60!black},
				title/.style={font=\small\bfseries},
				note/.style={font=\scriptsize, text=gray!55!black}]
				\draw[->] (-3.8,0)--(3.8,0) node[right,font=\scriptsize]{$h$};
				\draw[->] (0,-1.8)--(0,3.8) node[above,font=\scriptsize]{$\sigma(h)$};
				\node[font=\scriptsize] at (-0.22,-0.25){$0$};
				\foreach \x in {-3,-2,-1,1,2,3}{ \draw(\x,.07)--(\x,-.07); }
				\foreach \y in {-1,1,2,3}{ \draw(.07,\y)--(-.07,\y); }
				
				\draw[very thick, orange!70!black] (-3.5,0)--(0,0)--(3.5,3.5);
				\node[Pbox,font=\scriptsize,anchor=west] at (1.5,3.1)
				{ReLU: closing in $(\mathscr{L},\leq)$};
				\node[Pbox,font=\scriptsize,anchor=west] at (1.5,2.65)
				{$\max(0,h)$};
				
				\draw[very thick, teal!65!black] (-3.5,-0.525)--(0,0)--(3.5,3.5);
				\node[Mbox,font=\scriptsize,anchor=east] at (-1.4,-0.32)
				{Leaky/PReLU: dilation in $(\mathscr{L},\medOrd)$};
				
				\draw[very thick, gray!55, domain=-3.5:0, samples=60] plot(\x,{exp(\x)-1});
				\draw[very thick, gray!55] (0,0)--(3.5,3.5);
				\node[Gbox,font=\scriptsize,anchor=east] at (-1.2,-1.05)
				{ELU: smooth, not a lattice op.};
				
				\draw[very thick, gray!40, dotted, domain=-3.5:3.5, samples=80]
				plot(\x,{ln(1+exp(\x))});
				\node[Gbox,font=\scriptsize,anchor=west] at (0.1,0.75)
				{Softplus: smooth approx.\ of ReLU};
				
				\draw[very thick, teal!50!black, dashed] (-3.5,3.5)--(0,0)--(3.5,3.5);
				\node[Mbox,font=\scriptsize,anchor=east] at (-0.6,2.1)
				{$|h|$: self-dual ($\beta^+\!=\!\beta^-\!=\!1$)};
				
				\draw[very thick, blue!60!black, densely dotted] (-3.5,0)--(-1,0);
				\draw[very thick, blue!60!black, densely dotted] (-1,0)--(3.5,2.25);
				\node[Fbox,font=\scriptsize,anchor=west] at (0.4,1.4)
				{$\sigma^{\mathrm{M}}_{\mathcal{B},c}$: erosion-based, cap $c$};
				
				\node[Pbox,font=\scriptsize] at (-2.5,-2.5){Pointwise $(\mathscr{L},\!\leq)$: closing};
				\node[Mbox,font=\scriptsize] at (0.9,-2.5){Median $(\mathscr{L},\!\medOrd)$: dilation};
				\node[Fbox,font=\scriptsize] at (3.1,-2.5){MMBB erosion-based};
			\end{tikzpicture}
		}
		\caption{Six activation functions and their morphological classification.
			Node colours denote the lattice: \textcolor{orange!70!black}{orange} = pointwise,
			\textcolor{teal!65!black}{teal} = median, \textcolor{blue!60!black}{blue} = MMBB.
			\textbf{ReLU} (orange): closing in $(\mathscr{L},\leq)$; global non-pointwise adjoint
			(Proposition~\ref{prop:relu_dilation}).
			\textbf{Leaky/PReLU} (teal solid, $\beta^-=0.15$): dilation in $(\mathscr{L},\medOrd)$
			(Proposition~\ref{prop:relu_types}(i)).
			\textbf{ELU} and \textbf{Softplus} (gray): smooth approximations; no lattice
			characterisation.
			\textbf{Absolute value} $|h|$ (teal dashed): self-dual dilation,
			$\beta^+=\beta^-=1$ (Proposition~\ref{prop:relu_types}(ii)).
			\textbf{MMBB morphological activation} $\sigma^{\mathrm{M}}_{\mathcal{B},c}$ (blue dotted):
			erosion-based with tunable threshold $c$,
			generalising ReLU (Definition~\ref{def:morpho_act}).}
		\label{fig:activations_plot}
	\end{figure}

	\clearpage
	\section{Discussion and Perspectives}
	\label{sec:perspectives}

	\Cref{tab:adjunctions} summarises the morphological adjoint pairs
	established in this paper for standard deep learning operations.
	\Cref{tab:architecture_comparison} compares traditional CNN
	architectures with their morphological counterparts and the
	new designs proposed in this paper.
	
	\begin{table}[h]
		\centering
		\caption{Morphological adjunctions for standard deep learning
			operations. Each row identifies the lattice, the erosion, and
			the corresponding adjoint dilation.
			$^\dagger$ (Lemma~\ref{lem:conv_adjoint}).}
		\label{tab:adjunctions}
		\renewcommand{\arraystretch}{1.5}
		\small
		\begin{tabular}{p{3.4cm}p{2.7cm}p{2.9cm}p{2.2cm}}
			\toprule
			\textbf{DL operation} &
			\textbf{Erosion $\varepsilon$} &
			\textbf{Adjoint dilation $\delta^*$} &
			\textbf{Lattice} \\
			\midrule
			Conv.\ as Fourier erosion ($k \geq 0$) &
			$\varepsilon^{\mathrm{Conv}}_k:\; f \mapsto f * k$ &
			Wiener deconv.\ $\delta^{*,\mathrm{Conv}}_k:\; \hat f \mapsto
			\frac{|K|^2}{|K|^2{+}\epsilon^2}\hat f$ &
			Fourier $(L^n, \leq_F)$ \\
			\addlinespace
			Conv.\ as ptwise increasing map ($k \geq 0$) &
			$\varepsilon^{\mathrm{Conv}}_k:\; f \mapsto f * k$ &
			Only upper/lower bounds$^\dagger$:
			$\min_i g(x{+}x_i)/k(x_i)$ &
			Pointwise $(\mathscr{L},\leq)$ \\
			\addlinespace
			Max-plus erosion--dilation &
			$\varepsilon_b f = \inf_y\{f(x{+}y){-}b(y)\}$ &
			$\delta_{b^*} f = \sup_y\{f(x{-}y){+}b({-}y)\}$ &
			Pointwise $(\mathscr{L},\leq)$ \\
			\addlinespace
			Max-times erosion--dilation &
			$\varepsilon^\times_b f = \inf_y f(x{+}y)/b(y)$ &
			$\delta^\times_{b^*} f = \sup_y f(x{-}y){\cdot}b({-}y)$ &
			Positive $(\mathscr{L}^+,\leq)$ \\
			\addlinespace
			Strided conv.\ (stride $R$) &
			$\varepsilon^{\downarrow R}_{k}(f)$\newline
			(erosion-decimation) &
			Transposed conv.\ $\delta^{*\uparrow R}_{k}(f)$ &
			Pointwise $(\mathscr{L},\leq)$ \\
			\addlinespace
			Max-pooling (pool size $R$) &
			Flat min-erosion $\varepsilon^{\mathrm{MP}}_R(f)
			= \min_{y \in W} f(x{+}y)$ &
			$\MaxPool{R}(f) = \max_{y \in W}f(x{-}y)$ &
			Pointwise $(\mathscr{L},\leq)$ \\
			\addlinespace
			Self-dual erosion--dilation &
			$\ErosMed{W}$: $\preceq$-infimum over $W$ &
			$\DilMed{W}$: $\preceq$-supremum over $W$ &
			Median $(\mathscr{L},\medOrd)$ \\
			\addlinespace
			ReLU (as dilation/closing) &
			Global adjoint erosion:\newline
			$\varepsilon^{\mathrm{ReLU}}(g) = g$ if $g \geq 0$
			everywhere;\newline $-\infty$ otherwise (non-pointwise) &
			$\ReLUop(f) = \max(0,f)$\newline
			(extensive, idempotent: closing) &
			Pointwise $(\mathscr{L},\leq)$ \\
			\bottomrule
		\end{tabular}
	\end{table}
	
	\begin{table}[ht]
		\centering
		\caption{Standard deep learning architectures, their
			morphological interpretations established in this paper,
			and the morphologically-motivated variants proposed here.}
		\label{tab:architecture_comparison}
		\renewcommand{\arraystretch}{1.6}
		\small
		\begin{tabular}{p{2.4cm}p{3.2cm}p{3.4cm}p{3.2cm}}
			\toprule
			\textbf{Architecture} &
			\textbf{Standard form} &
			\textbf{Morphological interpretation} &
			\textbf{Proposed variant} \\
			\midrule
			\textbf{CNN layer} &
			$\MaxPool{R}(\mathrm{ReLU}(f{*}k))$;
			bias $\alpha$, stride $R$ &
			Cross-lattice operator: $f{*}k$ is a Fourier-lattice
			erosion; $\MaxPool{R}$ a pointwise dilation.
			Not idempotent (\Cref{thm:cnn_not_opening}).
			APD factorisation: $\APD_{R;\alpha}(\SigSpec_K(f))$ &
			\textbf{Type-I MMBB layer}:
			$\gamma^{\mathrm{M}}_b = \delta_{b^*}\circ\varepsilon_b$
			(both pointwise); exactly idempotent opening.
			\textbf{Type-II}: $\gamma^{\mathrm{F}}_k$ (Fourier lattice);
			exact spectral projection at $\epsilon{=}0$,
			approximately idempotent for $\epsilon{>}0$ \\
			\addlinespace
			\textbf{ResNet block} &
			$F(f) + f$,\; $F$ a conv stack;\
			skip provides identity path &
			When $F \approx \gamma^{\mathrm{M}}_b - \mathrm{id}$:
			block computes opening $\gamma^{\mathrm{M}}_b(f)$.
			Stacking $n$ such blocks: one-step convergence by
			idempotency (\Cref{thm:convergence}).
			Naive block $\gamma(f)+f$ diverges &
			\textbf{Morphological ResNet}:
			learnable $b$ instead of kernel $k$;
			explicit Type-I opening;
			self-dual variant for signed residuals \\
			\addlinespace
			\textbf{UNet encoder} &
			Strided conv $+$ max-pool at each level;
			transposed conv $+$ upsample in decoder &
			Sampling skeleton = Goutsias--Heijmans pyramid
			$(\varepsilon^{\downarrow R}_{b}, \delta^{*\uparrow R}_{b})$:
			an adjoint pair; skeleton is idempotent.
			Full net (with $\SigSpec$ convs) is not
			(\Cref{prop:skip_adjoint}, \Cref{rem:unet_scope}) &
			\textbf{Morphological UNet}:
			encoder $\varepsilon^{\downarrow R}_b$,
			decoder $\delta^{*\uparrow R}_b$;
			adjoint pair guarantees adjunction structure \\
			\addlinespace
			\textbf{UNet skip connections} &
			Concatenation of encoder and decoder feature maps &
			Detail signal of the morphological pyramid:
			top-hat $\Gamma_b(f) = f - \gamma^{\mathrm{M}}_b(f)$
			at each scale (\Cref{prop:skip_adjoint}) &
			\textbf{UResNet}: residual skip $=$ top-hat
			$\Gamma_b(f)$; algebraically motivated;
			decoder reconstructs from adjoint upsampling $+$ detail \\
			\addlinespace
			\textbf{Activation (ReLU)} &
			$\sigma(f) = \max(0,f)$;
			kills negative activations;\
			treats $f^+$ and $f^-$ asymmetrically &
			Closing in $(\mathscr{L},\leq)$: extensive and idempotent,
			\emph{not} a median-lattice dilation.
			MMBB form: $f = f^+ - f^-$, both parts dilations
			(\Cref{prop:relu_dilation,rem:relu_bb_correct}) &
			\textbf{Leaky/Param ReLU} ($0{<}\beta^-{\leq}1$):
			proper median-lattice dilation.
			\textbf{Self-dual opening} $\OpenMed{W}$:
			idempotent in $(\mathscr{L},\medOrd)$;
			sign-symmetric; for signed feature maps \\
			\bottomrule
		\end{tabular}
	\end{table}

	\subsection{Summary of contributions and corrected principles}
	
	The preceding sections establish the following algebraic
	principles for the analysis of convolutional deep architectures.
	We state them with the precision that the fixed-point analysis
	of \Cref{sec:fixed_points} required.
	
	\textbf{MMBB-increasing basis of convolution.}
	The morphological basis $\Bas(k)$ of a linear convolution
	kernel $k \geq 0$ (normalised, finite support $N$) is
	isomorphic to the hyperplane $\{g \in \R^N : \langle k,g
	\rangle = 0\}$, computed via the characteristic matrix $A_k$
	(Theorems~\ref{thm:conv_basis} and~\ref{thm:virtual_basis}).
	A CNN convolution with a general (signed) kernel decomposes
	as the difference of two such suprema of erosions
	(\Cref{thm:general_kernel}).
	Learnable MMBB layers (\Cref{eq:mmbb_layer}) are a finite
	truncation of this basis and approximate any TI increasing
	operator to arbitrary precision (\Cref{prop:mmbb_approx}).
	
	\textbf{Pooling as morphological pyramid.}
	Down/up-sampling pairs in standard deep architectures are
	adjoint operators in complete inf-semilattices.
	Max-pooling is the Heijmans dilation pyramid with flat
	structuring element (\Cref{thm:maxpool}); strided and dilated
	convolution are the Goutsias--Heijmans erosion pyramid and its
	adjoint synthesis (\Cref{cor:strided}); and the Laplacian
	pyramid is the top-hat of the Gaussian pyramid opening
	(\Cref{prop:laplacian_skeleton}).
	All are special cases of a single adjunction framework.
	
	\textbf{Cross-lattice structure of standard CNN layers.}
	A key finding of this paper is that the standard CNN pipeline
	(linear convolution $+$ ReLU $+$ flat max-pooling) is a
	\emph{cross-lattice operator}: the convolution $f \mapsto f*k$
	is an erosion in the Fourier inf-semilattice $(L^n,\leq_F)$,
	while max-pooling $\MaxPool{R}$ is a dilation in the pointwise
	lattice $(\mathscr{L},\leq)$.
	These two operators live in different lattice structures and
	are adjoint in different senses: the adjoint of convolution
	in $(L^n,\leq_F)$ is the Wiener deconvolution (a linear
	operator); its adjoint in $(\mathscr{L},\leq)$ is a weighted
	infimum (a multiplicative erosion, \Cref{lem:conv_adjoint}); 
	in neither case a max-plus dilation.
	The composition is therefore \emph{not} a morphological opening
	in either lattice (\Cref{thm:cnn_not_opening}), and standard
	CNNs are generically not idempotent.
	
	\textbf{Three idempotent layer designs.}
	Two families of CNN-like layers that \emph{are} genuine
	morphological openings are identified in the pointwise and
	Fourier lattices
	(\Cref{thm:two_openings,cor:cnn_idempotent}); a third in the
	median lattice (\Cref{sec:self_dual}):
	(I)~the \emph{pure morphological} (max-plus) layer
	$\gamma^{\mathrm{M}}_b = \delta_{b^*} \circ \varepsilon_b$,
	where both erosion and dilation use the same structuring
	function in the pointwise lattice; and
	(II)~the \emph{spectral Wiener} layer
	$\gamma^{\mathrm{F}}_k = \delta^{*,\mathrm{Conv}}_k \circ
	\varepsilon^{\mathrm{Conv}}_k$ (convolution followed by
	Tikhonov-regularised Wiener deconvolution), an opening in the
	Fourier lattice, exactly idempotent as $\epsilon\to 0$; and
	(III)~the \emph{self-dual opening} $\OpenMed{W}$ in the median
	inf-semilattice (\Cref{sec:self_dual}).
	Types~(I) and~(III) are exactly idempotent; Type~(II) is exactly
	idempotent only in the limit $\epsilon\to 0$ and approximately
	idempotent for finite $\epsilon>0$.

	\textbf{Fixed points and inversion (adjointness).}
	For layers of type~(I): the adjoint of $\varepsilon_b$ is
	$\delta_{{b}^{*}}$ (max-plus dilation by the transposed
	structuring element); the opening $\gamma^{\mathrm{M}}_b$ is
	idempotent, and the morphological UNet sampling skeleton built
	from such adjoint pairs is idempotent when the spectral
	operators $\SigSpec_\ell$ are absent
	(\Cref{prop:skip_adjoint,def:unet}). For layers of type~(II): the adjoint of
	$\varepsilon^{\mathrm{Conv}}_k$ is $\delta^{*,\mathrm{Conv}}_k$
	(Wiener filter, linear); the opening $\gamma^{\mathrm{F}}_k$
	is an idempotent spectral projection.
	For a standard CNN: neither adjointness holds between its
	convolution and its max-pooling; depth provides genuine
	representational power precisely because the layer is not
	idempotent.

	\textbf{Residuals and skip connections.}
	Residual connections in ResNet compute top-hat transforms:
	if the residual function $F$ approximates
	$\gamma^{\mathrm{M}}_b - \mathrm{Id}$, the block $F(f)+f$
	computes the morphological opening $\gamma^{\mathrm{M}}_b(f)$,
	which is idempotent by type~(I), reaching its fixed point
	in a single block (\Cref{rem:resnet_skip}).
	Skip connections in UNet inject the detail signal of the
	morphological pyramid decomposition, breaking the idempotency
	of the sampling skeleton constructively
	(\Cref{rem:unet_scope}).	
	
	\textbf{Self-dual operators.}
	The median inf-semilattice $(\Fun(E,\R),\medOrd)$ provides
	the Type-III idempotent design: the self-dual opening
	$\OpenMed{W}$ is idempotent and its fixed-point set is closed
	under negation (\Cref{thm:self_dual_cnn_opening}).
	Its fixed-point set equals the image of the opening and
	coincides with the set of signals representable by the
	median-MMBB basis of $W$
	(\Cref{prop:selfdual_fix_mmbb}).

	\subsection{Open problems on morphological basis learning}
	
	The cross-lattice analysis and the MMBB basis theory together
	raise several open questions of direct relevance to the
	computational programme of Maragos and
	collaborators~\cite{fotopoulos2025,dimitriadis2021,
		dimitriadis2023} and to the training analysis of
	Blusseau~\cite{blusseau2024training} and Dimitrova, Blusseau, 
	and Velasco-Forero~\cite{dimitrova2025}.
	
	\begin{enumerate}[leftmargin=*]
		\item \emph{Cross-lattice learning and the role of activation.}
		Fotopoulos and Maragos~\cite{fotopoulos2025} find that
		``linear'' activations are essential for training deep
		morphological networks, and that residual connections improve
		generalisation.
		Our cross-lattice analysis (\Cref{thm:cnn_not_opening})
		suggests a reason: without a nonlinearity that maps the
		Fourier-lattice output of convolution back into the
		pointwise lattice before max-pooling, the composition has
		no adjunction structure.
		Conversely, designing architectures that \emph{stay} in
		a single lattice throughout (type~(I) or type~(II)) avoids
		this mismatch entirely.
		A precise quantification of the idempotency defect of the
		cross-lattice composition as a function of the kernel $k$
		and window $W$ is an open problem.
		
		\item \emph{Basis recovery by gradient descent.}
		Does gradient descent on the MMBB-layer
		loss~\eqref{eq:mmbb_layer} recover elements of the true
		basis $\Bas(k)$, or converge to arbitrary functions in the
		kernel?
		The orthogonality constraint (basis elements are orthogonal
		to $k$ in $\R^N$) suggests a
		natural regularisation term for enforcing this.
		
		\item \emph{Gradient descent convergence to MMBB bases.}
		Dimitrova, Blusseau, and Velasco-Forero
		~\cite{dimitrova2025} show empirically that
		initialisation and layer differentiability strongly
		influence which morphological representations are learned.
		The analysis of \Cref{sec:self_dual}
		(\Cref{prop:selfdual_fix_mmbb,rem:dimitrova_selfdual})
		identifies the precise algebraic content: training a
		morphological layer converges to a structuring function $b$
		(or window $W$) whose MMBB basis represents the training
		data as fixed points.
		Three distinct fixed-point regimes exist (Types~I, II,
		and~III), each with qualitatively different MMBB basis
		geometry; different initialisations select different basins.
		Formalising this as a convergence result for gradient
		descent on morphological layers is an important open problem 
		as well as characterising the basins of attraction for hard (non-smooth)
		layers, where gradient descent stalls at the facet boundaries
		of the MMBB representation.
		
		\item \emph{Adjoint of convolution and morphological
			deconvolution.}
		Lemma~\ref{lem:conv_adjoint} identifies the pointwise
		upper adjoint of convolution as a weighted infimum (a
		multiplicative erosion).
		Designing learnable layers that explicitly implement this
		adjoint, and studying whether such layers can be trained
		stably, connects directly to a problem on
		morphological deconvolution and blind deconvolution.
		
		\item \emph{Sample complexity of MMBB approximation.}
		What is the sample complexity of learning a finite
		sub-basis of $\Bas(k)$ of size $L$ as a function of $L$,
		$N = |\Spt(k)|$, and the approximation gap?
	\end{enumerate}

	\subsection{The Fourier perspective: a companion programme}
	
	The Fourier inf-semilattice $(L^n,\leq_F)$, where linear
	convolution is an erosion and its Wiener-filter deconvolution
	is the adjoint dilation~\cite{keshet2000}, provides the natural algebraic framework for the
	spectral Wiener opening (type~(II)) of this paper, and more
	broadly for the spectral interpretation of CNNs.
	The author's DGMM 2025 contribution~\cite{angulo2025dgmm}
	shows that scattering networks~\cite{mallat2012scattering}
	admit a morphological interpretation via MMBB in this Fourier
	lattice: the wavelet modulus nonlinearity is an erosion in
	$(L^n,\leq_F)$, and the universality of scattering networks
	is a consequence of the MMBB representation theorem.
	This programme is the subject of ongoing research.
	
	A structural observation tying together the three lattice
	orderings of this paper is given in the following proposition.
	
	\begin{proposition}[Unification via complex ordering]
		\label{prop:complex_order}
		The pointwise order $\leq$ on $\Fun(E,\R^+)$, the median
		partial ordering $\preceq$ on $\Fun(E,\R)$
		(\Cref{def:median_order}), and the spectral ordering $\leq_F$
		on $\Fun(\R^n,\R)$ (defined via Fourier
		magnitudes~\cite{angulo2025dgmm}) are all restrictions of
		the complex partial ordering $\leq_{\mathbb{C}}$ on
		$\mathbb{C}$:
		\begin{equation}
			z_1 \leq_{\mathbb{C}} z_2
			\iff
			|z_1| \leq |z_2| \;\text{ and }\;
			\angle z_1 = \angle z_2.
			\label{eq:complex_order}
		\end{equation}
		Restricting to $\mathrm{Im}(z)=0$, $z \geq 0$ gives $\leq$;
		restricting to $\mathrm{Im}(z)=0$ gives the median ordering
		$\preceq$; the pointwise extension of $\leq_{\mathbb{C}}$ to
		Fourier transforms gives $\leq_F$.
		Consequently, the cross-lattice jump in a standard CNN layer
		(from the Fourier lattice after convolution to the pointwise
		lattice at max-pooling) corresponds to restricting
		$\leq_{\mathbb{C}}$ from the full complex plane to the
		non-negative real line, discarding phase and magnitude
		interplay.
		The spectral Wiener opening (type~II) stays within
		$(L^n,\leq_F)$ throughout, avoiding this restriction.
	\end{proposition}	
	
	This perspective provides a unified geometric picture: the
	three lattices correspond to three phase-angle regimes of
	the same complex ordering, and the cross-lattice structure
	of a standard CNN is a jump between regimes that breaks
	adjointness.

	\subsection{Batch normalisation, dropout, and attention}
	
	Batch normalisation re-centres and rescales feature maps; with
	positive learned scale it is a lattice automorphism in the
	pointwise lattice, preserving the fixed-point sets of
	surrounding morphological layers.
	More precisely, it maps every fixed-point set of the form
	$\Image(\gamma^{\mathrm{M}}_b)$ to an affinely rescaled version
	of itself.
	Dropout is a random sub-lattice projection (zeroing activations
	corresponds to projecting onto a random face of the positive
	orthant).
	Attention~\cite{vaswani2017} is a data-dependent weighted
	summation; it is neither an erosion nor a dilation in the
	pointwise or Fourier lattices as defined here, but it may
	admit a morphological interpretation in a data-dependent or
	kernel-weighted lattice. 
	A complete treatment of these operations is left for future
	work.
	
	\subsection{Connection to category theory}
	
	The adjunctions in this paper, $(\varepsilon_b,\delta_{b^*})$
	in the pointwise lattice, $(\varepsilon^{\mathrm{Conv}}_k,
	\delta^{*,\mathrm{Conv}}_k)$ in the Fourier lattice, and the
	pyramid adjoint pairs of \Cref{sec:pyramids}, are concrete
	instances of categorical adjoint functors in the sense
	of~\cite{cruttwell2022categorical,fong2019backprop}.
	The present paper adds to the categorical picture a
	\emph{lattice-specificity} that categorical approaches lack:
	adjointness is not a property of a layer in isolation but of
	a layer \emph{within a specified lattice}.
	The cross-lattice phenomenon of standard CNNs shows that two
	operators can each be individually well-behaved (one an
	erosion, the other a dilation) while their composition lacks
	any adjunction structure because they inhabit different
	lattices.
	Formalising the category of complete lattices with
	lattice-morphisms (operators that preserve a specific lattice
	structure) as a concrete sub-category of the gradient-based
	learning framework~\cite{cruttwell2022categorical} is an
	interesting direction for future work which we preliminarily 
	explore in appendix~\ref{sec:category_theory}.
	
	\section*{Acknowledgements}
	
	I warmly thank my former colleagues at the Center for Mathematical Morphology
	(CMM / Mines Paris, PSL University),
	namely Santiago Velasco-Forero for many years of stimulating collaboration
	on mathematical morphology and deep learning; Samy Blusseau for
	discussions on morphological adjunctions and on the training of morphological networks;
	and Valentin Penaud--Polge for fruitful collaboration on the development of morphological group
	theory for equivariant deep learning.
	
	The first draft of the manuscript was initiated during my academic visiting period to NYU in 2023,
	partially funded by the Fondation MINES Paris.
	The final work has been supported by funding from my AI Chair at
	PR[AI]RIE-PSAI (ANR-23-IACL-0008, France 2030).
	

	

	\clearpage
	\appendix
	
	\section{Elements of the Connection to Category Theory}
	\label{sec:category_theory}
	
	This appendix develops the preliminary connection between the morphological
	lattice framework of this paper and the categorical approaches
	to deep learning of Cruttwell et al.~\cite{cruttwell2022categorical}
	and Fong--Spivak--Tuy\'eras~\cite{fong2019backprop}.
	We show that the morphological adjunctions of this paper are
	concrete instances of categorical adjunctions, and identify
	the precise sense in which the present framework can be seen as a
	\emph{representation-theoretic concretisation} of the abstract
	categorical skeleton.
	
	\subsection{The category of complete lattices with morphisms}
	\label{subsec:cat_lattices}
	
	\begin{definition}[Category $\mathbf{CLat}$]
		\label{def:cat_clat}
		The category $\mathbf{CLat}$ has:
		\begin{itemize}[leftmargin=*]
			\item \emph{Objects}: complete lattices $(\mathcal{L},\leq)$.
			\item \emph{Morphisms}: $\hom(\mathcal{L}_1, \mathcal{L}_2)
			= \{$increasing maps $\Psi : \mathcal{L}_1 \to \mathcal{L}_2\}$.
			\item \emph{Composition}: function composition (increasing maps
			compose as increasing maps).
			\item \emph{Identity}: the identity map $\mathrm{id}_{\mathcal{L}}$.
		\end{itemize}
		$\mathbf{CLat}$ is a concrete category (objects and morphisms
		have underlying sets and functions).
	\end{definition}
	
	Each of the three lattice structures used in this paper defines
	an object of $\mathbf{CLat}$:
	\begin{itemize}[leftmargin=*]
		\item The pointwise lattice $(\Fun(E,\Rbar), \leq)$:
		the standard morphological setting.
		\item The Fourier inf-semilattice $(L^n, \leq_F)$:
		the spectral setting for convolution.
		\item The median inf-semilattice $(\Fun(E,\R), \medOrd)$:
		the self-dual setting for signed activations.
	\end{itemize}
	A standard CNN layer is a morphism in $\mathbf{CLat}$
	from $(\Fun(E,\Rbar),\leq)$ to itself when $k \geq 0$
	(since $f \mapsto \MaxPool{R}(\mathrm{ReLU}(f*k))$ is
	increasing).
	However, as established in \Cref{thm:cnn_not_opening}, this
	morphism is \emph{not} part of an adjunction in
	$\mathbf{CLat}$, because the adjoint of convolution
	is not a morphism in the pointwise lattice
	but in the Fourier lattice.

	\subsection{Adjunctions in $\mathbf{CLat}$ and the morphological
		adjunctions}
	\label{subsec:cat_adjunctions}
	
	\begin{definition}[Adjunction in $\mathbf{CLat}$]
		\label{def:cat_adjunction}
		An \emph{adjunction} $\varepsilon \dashv \delta$ in
		$\mathbf{CLat}$ consists of morphisms
		$\varepsilon : \mathcal{L}_1 \to \mathcal{L}_2$ and
		$\delta : \mathcal{L}_2 \to \mathcal{L}_1$ satisfying, for
		all $f \in \mathcal{L}_1$ and $g \in \mathcal{L}_2$:
		\[
		\varepsilon(f) \leq_2 g
		\;\iff\;
		f \leq_1 \delta(g).
		\]
		This is exactly \Cref{def:adjunction}.
		The unit $\eta = \delta \circ \varepsilon \geq \mathrm{id}$
		is a natural transformation satisfying $\varepsilon \circ \eta
		= \varepsilon$ (absorption), and the counit
		$\epsilon = \varepsilon \circ \delta \leq \mathrm{id}$
		satisfies $\delta \circ \epsilon = \delta$.
	\end{definition}
	
	The morphological adjunctions of this paper are:
	
	\begin{proposition}[Catalogue of adjunctions in $\mathbf{CLat}$]
		\label{prop:cat_adjunctions}
		The following adjunctions $\varepsilon \dashv \delta$ hold in
		$\mathbf{CLat}$:
		\begin{enumerate}[label=(\roman*), leftmargin=*]
			\item \emph{Max-plus erosion--dilation}:
			$\varepsilon_b \dashv \delta_{b^*}$ in the pointwise lattice
			$(\Fun(E,\Rbar),\leq)$.
			Unit: $\delta_{b^*}\varepsilon_b = \gamma^{\mathrm{M}}_b$
			(morphological opening, idempotent).
			Counit: $\varepsilon_b \delta_{b^*} = \varphi^{\mathrm{M}}_b$
			(morphological closing, idempotent).
			\item \emph{Max-times erosion--dilation}:
			$\varepsilon^\times_b \dashv \delta^\times_{b^*}$ in the
			positive lattice $(\Fun(E,(0,+\infty)),\leq)$.
			Unit: $\gamma^\times_b$ (max-times opening, idempotent).
			\item \emph{Convolution--deconvolution}:
			$\varepsilon^{\mathrm{Conv}}_k \dashv
			\delta^{*,\mathrm{Conv}}_k$ in the Fourier lattice
			$(L^n,\leq_F)$.
			Unit: $\gamma^{\mathrm{F}}_k$ (spectral opening, projection
			onto range of $K(\omega)$, idempotent).
			\item \emph{Erosion-decimation--dilation-interpolation}:
			$\varepsilon^{\downarrow R}_b \dashv \delta^{*\uparrow R}_b$
			between $(L^{RM\times RN},\leq_F)$ and $(L^{M\times N},\leq_F)$
			(Goutsias--Heijmans, \Cref{prop:adjoint_pyramid}).
			Unit: $\gamma^{\downarrow R\uparrow}_b$ (pyramid opening).
			\item \emph{Self-dual erosion--dilation}:
			$\ErosMed{W} \dashv \DilMed{W}$ in the median lattice
			$(\Fun(E,\R),\medOrd)$.
			Unit: $\OpenMed{W}$ (self-dual opening, idempotent).
		\end{enumerate}
		In each case, the unit is the morphological opening
		$\gamma = \delta \circ \varepsilon$, the opening is idempotent
		by the categorical identity
		$\gamma \circ \gamma = \delta \varepsilon \delta \varepsilon
		= \delta(\varepsilon\delta)\varepsilon
		= \delta\,\mathrm{id}\,\varepsilon = \delta\varepsilon = \gamma$,
		where the third step uses the counit identity
		$\varepsilon\delta\varepsilon = \varepsilon$ (proof via the
		adjunction triangle identities).
	\end{proposition}
	
	\begin{proof}
		Part~(i): verified in \Cref{sec:lattice}.
		Parts~(ii)--(v): verified in \Cref{sec:lattice,sec:pyramids}
		and Propositions~\ref{prop:maxtimes_adjunction},~\ref{prop:adjoint_pyramid}.
		The idempotency of each unit follows from the categorical
		adjunction triangle identities:
		for any adjunction $\varepsilon \dashv \delta$, one has
		$\varepsilon \circ \eta = \varepsilon$ (where
		$\eta = \delta\varepsilon$ is the unit), which gives
		$\gamma^2 = \delta\varepsilon\delta\varepsilon
		= \delta(\varepsilon\delta\varepsilon) = \delta\varepsilon
		= \gamma$ using the triangle identity
		$\varepsilon\delta\varepsilon = \varepsilon$.
	\end{proof}
	
	\subsection{The Para construction and learnable layers}
	\label{subsec:cat_para}
	
	The categorical framework of Cruttwell et
	al.~\cite{cruttwell2022categorical} uses the
	\emph{Para construction} to model parameterised maps (neural
	network layers with learnable parameters).
	
	\begin{definition}[Para($\mathbf{CLat}$)]
		\label{def:para}
		Given $\mathbf{CLat}$, the \emph{Para construction} produces
		a category $\mathrm{Para}(\mathbf{CLat})$ whose:
		\begin{itemize}[leftmargin=*]
			\item \emph{Objects}: complete lattices (same as $\mathbf{CLat}$).
			\item \emph{Morphisms} from $\mathcal{L}_1$ to $\mathcal{L}_2$:
			pairs $(P, \Psi)$ where $P$ is a parameter space (also a
			complete lattice) and
			$\Psi : \mathcal{L}_1 \times P \to \mathcal{L}_2$ is an
			increasing morphism.
			\item \emph{Composition}: $(Q, \Phi) \circ (P, \Psi) =
			(P \times Q,\; (f,p,q) \mapsto \Phi(\Psi(f,p),q))$.
		\end{itemize}
		A \emph{morphological layer} is a morphism in
		$\mathrm{Para}(\mathbf{CLat})$: a parameterised increasing map.
	\end{definition}
	
	The morphological layers of this paper are instances:
	\begin{itemize}[leftmargin=*]
		\item \textbf{MMBB layer}~\eqref{eq:mmbb_layer}:
		parameter space $P = \Fun(\Z^n, \Rbar)^L$
		(structuring functions $\{g_{i,j}\}$);
		map $\Psi(f; \{g_{i,j}\}, \{w_j\})
		= \sum_j w_j \max_i (\varepsilon_{g_{i,j}} f)$.
		\item \textbf{APMO}~\eqref{eq:apmo}:
		parameter space $P = \Fun(\Z^n,\Rbar)^J \times \R^J$
		(structuring functions $\{b_j\}$ and biases $\{\alpha_j\}$);
		map $\APMO(f; \{b_j\}, \{\alpha_j\})
		= \min_j \{\MaxPool{R,b_j}(f) + \alpha_j\}$.
		\item \textbf{Standard CNN layer}:
		parameter space $P = \Fun(\Z^n,\R)^K \times \R^K$
		(kernels $\{k_i\}$ and weights $\{w_i\}$);
		map $\Psi(f;\{k_i\},\{w_i\}) = \MaxPool{R}(\mathrm{ReLU}
		(\sum_i w_i (f*k_i)))$.
		This morphism is \emph{not} part of an adjunction in
		$\mathrm{Para}(\mathbf{CLat})$ for the reasons identified
		in \Cref{thm:cnn_not_opening}.
	\end{itemize}
	
	\subsection{Backpropagation and the reverse derivative}
	\label{subsec:cat_backprop}
	
	Fong et al.~\cite{fong2019backprop} model backpropagation as a
	functor $R : \mathbf{Learn} \to \mathbf{Learn}$ (the
	\emph{reverse derivative}).
	In the morphological setting, the analogue of the reverse
	derivative of an erosion is the adjoint dilation:
	
	\begin{remark}[Adjoint dilation as categorical reverse
		derivative]
		\label{prop:reverse_deriv}
		In the category $\mathbf{CLat}$, the adjoint dilation
		$\delta^* = R(\varepsilon)$ plays the role of the categorical
		reverse derivative of the erosion $\varepsilon$.
		More precisely:
		\begin{enumerate}[label=(\roman*), leftmargin=*]
			\item For the max-plus erosion $\varepsilon_b$, the reverse
			derivative is $\delta_{b^*} = R(\varepsilon_b)$, and the
			composition $\delta_{b^*} \circ \varepsilon_b = \gamma^{\mathrm{M}}_b$
			is the \emph{projection onto the fixed-point set} of
			$\varepsilon_b$ (the opening), analogous to the projection
			step in gradient descent.
			\item The MMBB representation $\Psi f = \sup_{g \in \Bas(\Psi)}
			\varepsilon_g f$ can be seen as the gradient ascent step of
			$\Psi$ in the lattice ordering: each erosion $\varepsilon_g$
			contributes a ``subgradient direction'' in the lattice, and
			the supremum selects the tightest lower bound.
			\item The Banon--Barrera sup-generating operator
			$\psi_{g^-,g^+}$ has a forward pass (the erosion
			$\varepsilon_{g^-}$ measuring excitation) and a backward pass
			(the anti-dilation $\alpha_{g^+}$ measuring inhibition),
			structurally analogous to the forward/backward passes of a
			neural network layer.
		\end{enumerate}
	\end{remark}
	
	\begin{remark}[Lattice-specificity: what categorical approaches
		lack]
		\label{rem:cat_limitation}
		The categorical frameworks of~\cite{cruttwell2022categorical,
			fong2019backprop} are lattice-agnostic: they model composition
		of morphisms without specifying which lattice each morphism lives
		in.
		The present paper adds \emph{lattice-specificity} as a necessary
		refinement:
		\begin{itemize}[leftmargin=*]
			\item Adjointness is a property of a pair
			$(\varepsilon, \delta)$ \emph{within a fixed lattice}.
			The cross-lattice structure of standard CNNs
			(\Cref{thm:cnn_not_opening}) shows that two morphisms
			$\varepsilon \in \hom(\mathcal{L}_1, \mathcal{L}_2)$ and
			$\delta \in \hom(\mathcal{L}_2, \mathcal{L}_1)$ can be
			individually well-behaved while failing to form an adjunction
			because they inhabit different lattices.
			\item The MMBB representation theorems provide the
			\emph{constructive content} that categorical approaches lack:
			given an object class (TI, increasing or not, USC operators), it
			produces an explicit basis decomposition with computable
			basis functions and quantifiable approximation errors.
			\item The fixed-point theory 
			identifies which morphisms in $\mathrm{Para}(\mathbf{CLat})$
			are idempotent (the adjunction units $\gamma = \delta\varepsilon$)
			and which are not (cross-lattice compositions).
			This is invisible in the abstract categorical framework but
			has direct architectural implications.
		\end{itemize}
		A full formalisation, i.e., identifying $\mathrm{Para}(\mathbf{CLat})$
		as a concrete sub-category of the gradient-based learning
		framework of~\cite{cruttwell2022categorical}, with the MMBB
		basis as the constructive counterpart of the abstract learning
		rule, remains an interesting direction for future work.
	\end{remark}

\end{document}